\documentclass[12pt, hidelinks]{article}

\usepackage{xr-hyper}
\makeatletter
\newcommand*{\addFileDependency}[1]{
  \typeout{(#1)}
  \@addtofilelist{#1}
  \IfFileExists{#1}{}{\typeout{No file #1.}}
}
\makeatother

\usepackage{amsfonts,amsmath,float,color,algorithm, amsthm,
  verbatim,subcaption,algorithmic,bm,booktabs,xcolor,enumerate,
  url,cases,indentfirst,graphicx,multicol,hyperref,multirow,amssymb}
\usepackage{arydshln}
\usepackage[T1]{fontenc}

\usepackage{natbib}
\usepackage{lineno}

\graphicspath{{pictures/}}
\allowdisplaybreaks

\newcommand{\rp}{\mathrm{pl}}
\newcommand{\rs}{\mathrm{sub}}
\newcommand{\rw}{\mathrm{w}}
\newcommand{\ruw}{\mathrm{mscl}}
\newcommand{\rf}{\mathrm{mle}}
\newcommand{\rmle}{\mathrm{scale}}
\newcommand{\radp}{\mathrm{adp}}
\newcommand{\rlas}{\mathrm{las}}
\newcommand{\rt}{\mathrm{t}}
\newcommand{\inn}{\mathrm{in}}
\newcommand{\out}{\mathrm{out}}
\newcommand{\oi}{\mathrm{out}\times\mathrm{in}}

\newcommand{\0}{\bm{0}}

\newcommand{\I}{\bm{I}}

\newcommand{\bOmega}{\bm{\Omega}}
\newcommand{\bLambda}{\bm{\Lambda}}
\newcommand{\ttheta}{{\tilde{\boldsymbol{\theta}}}}
\newcommand{\htheta}{{\hat{\boldsymbol{\theta}}}}

\newcommand{\btheta}{{\bm{\theta}}}

\newcommand{\bSigma}{\bm{\Sigma}}

\newcommand{\A}{\bm{A}}
\newcommand{\cA}{\mathcal{A}}
\newcommand{\cAc}{\mathcal{A}^c}
\newcommand{\hA}{(\hat{\mathcal{A}}_{\rp})}
\newcommand{\B}{\bm{B}}
\newcommand{\U}{\bm{U}}
\newcommand{\y}{\bm{y}}
\newcommand{\x}{\bm{x}}
\newcommand{\tx}{\tilde{\bm{x}}}
\newcommand{\X}{\bm{X}}
\newcommand{\bu}{\bm{u}}
\newcommand{\tu}{\hat{\bu}}
\newcommand{\bt}{\bm{t}}

\newcommand{\z}{\bm{z}}
\newcommand{\W}{\bm{Z}}
\newcommand{\G}{\bm{G}}

\renewcommand{\L}{\bm{L}}
\newcommand{\bd}{\bm{d}}
\newcommand{\V}{\bm{V}}
\newcommand{\M}{\bm{M}}

\newcommand{\batheta}{\acute{\bm{\theta}}}

\newcommand{\oo}{o(1)}
\newcommand{\op}{o_P(1)}
\newcommand{\Op}{O_P(1)}

\newcommand{\ud}{\mathrm{d}}
\newcommand{\g}{g}
\newcommand{\Q}{Q}
\newcommand{\dg}{\dot{\g}}
\newcommand{\ddg}{\ddot{\g}}

\newcommand{\dl}{\dot{\ell}}
\newcommand{\ddl}{\ddot{\ell}}

\newcommand{\df}{\dot{f}}
\newcommand{\Dn}{\mathcal{D}_n}
\newcommand{\Ds}{\mathcal{D}_{\delta}}
\newcommand{\sumn}{\sum_{i=1}^{N}}

\newcommand{\sumjp}{\sum_{j=1}^{p}}

\newcommand{\tp}{^{\mathrm{T}}}

\newcommand{\cvp}{\overset{P}{\longrightarrow}}
\newcommand{\cvas}{\overset{a.s.}{\longrightarrow}}

\newcommand{\cvd}{\rightsquigarrow}

\renewcommand{\Pr}{\mathbb{P}}
\newcommand{\Exp}{\mathbb{E}}
\newcommand{\Var}{\mathbb{V}}
\newcommand{\sgn}{\text{sgn}}
\newcommand{\mmse}{\mathrm{A-OS}}
\newcommand{\mvc}{\mathrm{L-OS}}
\newcommand{\mpr}{\mathrm{P-OS}}

\newcommand{\mle}{{\textnormal{\tiny MLE}}}%
\newcommand{\bbeta}{\bm{\beta}}
\newcommand{\hbeta}{\hat{\bbeta}}
\newcommand{\tbeta}{\tilde{\bbeta}}

\newcommand{\dd}{\mathrm{d}}

\newcommand{\eeta}{\bm{\eta}}
\newcommand{\bH}{\bm{H}}

\newtheorem{assumption}{Assumption}
\newtheorem{lemma}{Lemma}
\newtheorem{proposition}{Proposition}
\newtheorem{theorem}{Theorem}
\newtheorem{corollary}{Corollary}

\theoremstyle{remark}
\newtheorem{remark}{Remark}

\newenvironment{keywords}{%
  \par\medskip\noindent\textbf{Keywords:}\enspace\ignorespaces
}{%
  \par\medskip
}

\newcommand{\Nor}{\mathbb{N}}

\newcommand{\black}{\color{black}}

\newcounter{assumptionpart}[assumption]
\newcommand{\assumptionitem}[1]{%
    \refstepcounter{assumptionpart}\label{#1}%
    \textup{(\alph{assumptionpart})}%
}

\allowdisplaybreaks[4]
\usepackage[left=1.25in,right=1.25in]{geometry}

\begin{document}

\title{Scale-invariant Optimal Sampling for Rare-events Data with Sparse Models}

\author{Jing Wang$^{1}$, HaiYing Wang$^{1}$, Qiang Zhang$^{2}$, and Hao Helen Zhang$^{3}$ \\
  \\
  $^{1}$ Department of Statistics, University of Connecticut,\\ 
  Storrs, CT 06269, USA \\
  \\
  $^{2}$ Vickie and Jack Farber Vision Research Center, \\
  Wills Eye Hospital Thomas Jefferson University, \\
  840 Walnut St, Philadelphia, PA 19107, USA\\
  \\
  $^{3}$ Department of Mathematics, University of Arizona, \\
  Tucson, AZ 85721, USA}

\maketitle

\begin{abstract}%
  Subsampling is effective in tackling computational challenges for massive data
  with rare events.  Overly aggressive subsampling may adversely affect
  estimation efficiency, and optimal subsampling is essential to mitigate the
  information loss. However, existing optimal subsampling probabilities depend
  on data scales, and some scaling transformations may result in inefficient
  subsamples. This problem is more significant when there are inactive features,
  because their influence on the subsampling probabilities can be arbitrarily
  magnified by inappropriate scaling transformations. We tackle this challenge
  and introduce a scale-invariant optimal subsampling function in the context of sparse models, where inactive features are commonly assumed. Instead of
  focusing on estimating model parameters, we define an optimal subsampling
  function to minimize the prediction error, using adaptive lasso to outline the
estimation procedure and study its theoretical guarantee.  We first introduce
the adaptive lasso estimator for rare-events data and establish its oracle
properties, thereby validating the use of subsampling. Then we derive a
scale-invariant
  optimal subsampling function that minimizes the prediction error of the
  inverse probability weighted (IPW) adaptive lasso.  Finally, we present an
  estimator based on the maximum sampled conditional likelihood (MSCL) to
  further improve the estimation efficiency. We conduct numerical experiments
  using both simulated and real-world data sets to demonstrate the performance
  of the proposed methods.
\end{abstract}


\section{Introduction}\label{sec:intro}

Rare-events data refer to binary-response data that are highly imbalanced, i.e.,
the number of zeros (a.k.a ``controls'' or ``negative instances'') is possibly
hundreds or thousands of times as large as the number of ones (a.k.a. ``cases''
or ``positive instances''). This type of data is common in various fields, such
as medicine, natural science, political science, and social science, where
examples of rare events can be rare diseases, natural disasters, wars, and
financial crises, respectively. Modern technologies also
prompt
us to pay more attention to rare-events data. For example, in modern online recommendation
systems, clicks are usually rare events compared with nonclicks. Statistical
analyses, including parameter estimation and inference, pose unique challenges
for rare-events data because of high imbalance. In addition, rare-events
data often involve sparse models. For instance, rare diseases might be linked to
a limited number of key genes. Therefore, researchers frequently adopt sparse
models in genome-wide association studies for analyzing rare diseases. A
different yet related example is the use of deep neural networks to predict
click-through rates in modern online recommendation systems. These networks are
typically overparameterized, necessitating methods that balance rare-events data
with the sparsity of the underlying models.
Data balancing is a popular
approach to overcome challenges caused by imbalanced data and is usually
accomplished through subsampling the zeros \citep{drummond2003c4,
liu2008exploratory} or oversampling the ones \citep{chawla2002smote,
han2005borderline, mathew2017classification, douzas2017self}. In addition, rare-events data
are often massive in order to obtain an adequate number of ones, and computation
is demanding. Therefore, we focus on the subsampling approach since it addresses
the imbalance issue and reduces the computational burden simultaneously.

It is shown in \cite{wang2020logistic} that the efficiency of parameter
estimation is essentially determined by the number of ones for rare-events
logistic regression, and subsampling does not reduce the estimation efficiency
as long as sufficient zeros are kept.
In the case of excessive removal of zeros, \cite{wang2021nonuniform} developed an
optimal sampling approach to minimize information loss.  However, the optimal
sampling probabilities in \cite{wang2021nonuniform} are scale-dependent, which
may lead to inefficient results.
Figure~\ref{fig:intro} illustrates the issue using a simulated example, with details in
Section~\ref{sec:simuintro} of the appendix.
We generate the data from the same logistic regression model and re-scale one of the covariates with different scales $s=0.01, 0.1, 1, 10$, and $100$. Then we apply two optimal subsampling methods in
\citep{wang2021nonuniform}, labeled with ``A-OS'' and ``L-OS'' in
Figure~\ref{fig:intro}. It is observed that the prediction errors of A-OS and L-OS are significantly impacted by the data
scaling. The A-OS may perform similarly to the Uni (simple random sampling or
uniform sampling) in Figure~\ref{subfig:a} when $s=0.01$; so is the L-OS in
Figure~\ref{subfig:b} when $s=100$.
This scale-dependent issue is not specific to logistic regression and
rare-events data considered in \citep{wang2021nonuniform}; it is a wide concern in literature for various data types and models, including but not limited to
\cite{ai2019optimal, Zhang2020optimal, wang2021optimal, keret2023analyzing,
  yu2020optimal, yao2018optimal, wang2022sampling}. In
this paper, we propose a scale-invariant optimal subsampling method to overcome
the issue. It is labeled ``P-OS'' in Figure~\ref{fig:intro}.

\begin{figure}[ht]
  \centering
  \begin{subfigure}{0.435\textwidth}
    \includegraphics[width=\textwidth]{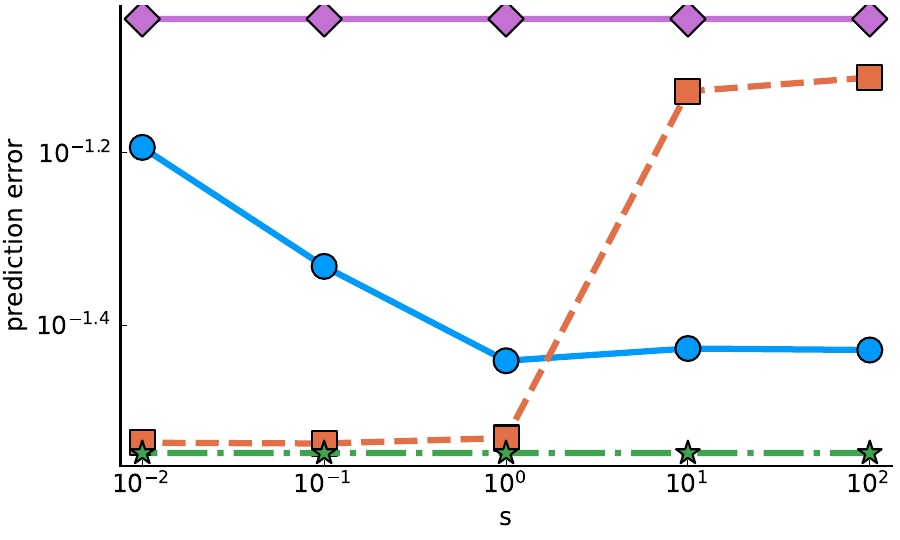}
    \caption{Non-sparse parameter}
    \label{subfig:a}
  \end{subfigure}
  \begin{subfigure}{0.435\textwidth}
    \includegraphics[width=\textwidth]{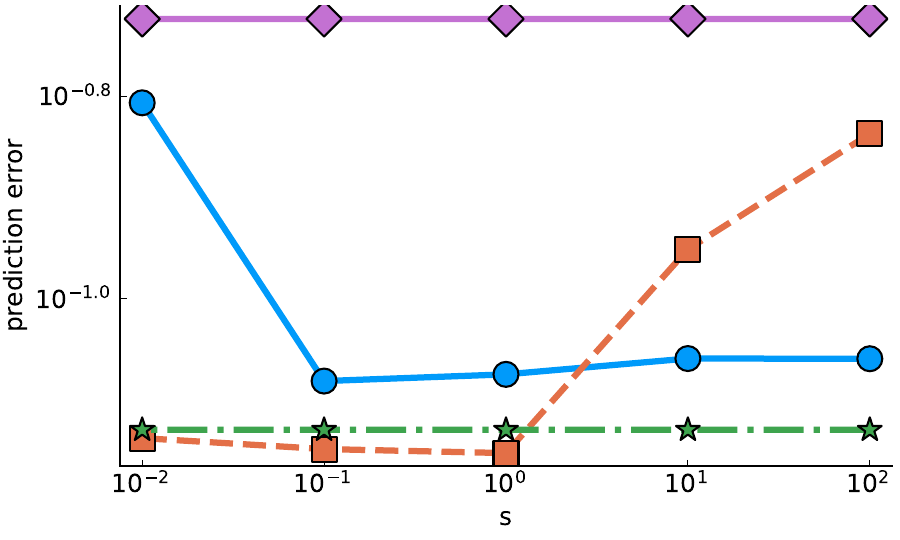}
    \caption{Sparse parameter}
    \label{subfig:b}
  \end{subfigure}
  \begin{subfigure}{0.1\textwidth}
    \includegraphics[width=\textwidth]{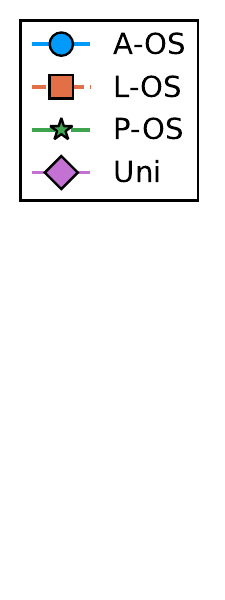}
  \end{subfigure}
  \caption{Prediction errors with different scale transformations of the same
    model with (a) non-sparse parameter $(-1,-1,-0.01,-0.01,-0.01,-0.01)\tp$ and
    (b) sparse parameter $(-1,0,0,0,0,0)\tp$.}
  \label{fig:intro}
\end{figure}

The scale-dependence issue can seriously impact variable selection results for sparse models, where true parameters are zero for inactive covariates. In this case, inactive variables may be arbitrarily transformed
without changing the underlying model, but the A-OS or L-OS would be highly influenced and may lead to misleading results. To resolve this issue, we investigate scale-invariant optimal subsampling in the context of variable
selection, for which one main goal is to distinguish active and inactive
features.

Penalty-based feature selection methods
are widely
used. Specifically, the adaptive lasso is a popular choice due to its oracle
properties, convexity, and practical ease of implementation
\citep[see][]{zou2006adaptive, zhang2007adaptive}. While penalization
methods have been used for bias reduction in rare-events analysis \citep{firth1993bias}, variable
selection for rare-events data has not been investigated.
Conducting effective variable selection is difficult in
the context of rare-events data analysis, mainly due to the scarcity of
information available for ones. An inaccurate variable selection result can
subsequently impact both the effectiveness of optimal subsampling and the
efficiency of parameter estimation. In this paper, we address the
challenge of variable selection in the context of rare-events data. First, we
propose the full data adaptive lasso and study its theoretical properties. Next,
we introduce a novel subsampling estimator that seamlessly combines
penalty-based variable selection and optimal sampling into one unified framework
for rare-events data. The implementation of the adaptive lasso requires a pilot
estimator to construct data-dependent weights for covariates. Given that
optimal sampling also relies on pilot estimates
\citep[see][]{WangZhuMa2017,ai2019optimal}, the adaptive lasso emerges as a
natural choice for conducting variable selection in the context of
subsampled rare-events data. We validate the new estimators by proving
their oracle properties and also develop an efficient algorithm to facilitate
their practical implementation when handling massive real-world data sets.
Our main contributions are listed as follows:
\begin{itemize}
\item We propose scale-invariant optimal subsampling to enhance parameter
  estimation and variable selection. Existing optimal subsampling methods are
  scale-dependent, which may lead to unreliable or misleading results.
\item We define adaptive lasso and establish its oracle properties for
  rare-events data, which show that the asymptotic variances are determined by
  the number of ones in the data and the active features in the model.
\item We present a practical subsampling algorithm based on optimal
  probabilities that significantly reduces the computational burden and
  accelerates the optimization for penalty-based feature selection methods.
\item We prove that the practical subsampling estimator has an asymptotic
  variance that reaches the lower bound of a large class of subsample estimators.
\end{itemize}
The rest of the paper is organized as follows. Section~\ref{sec:model}
introduces the model setup.
Section~\ref{sec:varsel} investigates nonuniform sampling and variable selection
tailored for rare-events data. 
We further discuss the limitation of scale-dependent optimal probabilities and
propose new methods to construct scale-invariant optimal
probabilities. Section~\ref{sec:theory} discusses theoretical properties of the
MSCL estimator and presents a two-step algorithm to implement the proposed
methods. Section~\ref{sec:numeric} conducts numerical experiments on simulated
and real data sets. Section~\ref{sec:conclusion} concludes the paper. Proofs of
theoretical results, and additional details and results of numerical experiments
are presented in the appendix.

\section{Background and model setup}\label{sec:model}

We use the subscript
$_{\rt}$ to indicate the true parameters. For a $p$-dimensional vector $\x$, we
use $x_{(i)}$ to represent its $i$-th element.
For an index subset $\cA \subset \{i:1,2,...,p\}$, we use $\x_{(\cA)}$ to denote
the subvector of $\x$, whose elements correspond to the indices in $\cA$.
Furthermore, we use $\x^{\otimes2}$ to denote
$\x\x\tp$, use ``$\cvd$'' to denote convergence in distribution, use
``$\cvp$'' to denote convergence in probability, and use ``$\cvas$'' to denote
convergence almost surely. We use $\I$ to denote an identity matrix of a
suitable dimension and use $\0$ to denote a vector of zeros of a suitable
dimension. 

Let $(\x_1,y_1),(\x_2,y_2),...,(\x_N,y_N)$ denote $N$ sample points from the joint distribution of $(\x,y)$,
where $\{\x_i\}_{i=1}^N$ denote the $p_{N}$-dimensional predictors and
$\{y_i\}_{i=1}^N$ the binary responses. Assume that the probability of
$y$ being a one ($y=1$) given $\x$ is  
\begin{equation*}
p(\x;\btheta_{\rt}):=\Pr(y=1|\x)
=\frac{e^{\alpha_{\rt}+f(\x;\bbeta_{\rt})}}{1+e^{\alpha_{\rt}+f(\x;\bbeta_{\rt})}}
=\frac{e^{g(\x;\btheta_{\rt})}}{1+e^{g(\x;\btheta_{\rt})}},
\end{equation*}
where $\btheta_{\rt}=(\alpha_{\rt},\bbeta_{\rt}\tp)\tp$ is the vector of true parameters and
$f(\x;\bbeta_t)$ is a smooth function of $\bbeta_t$.
For rare-events data, $N_1^*\ll N_0^{*}$, where $N_1^*=\sumn y_i$ is the
number of ones (i.e. $y_i=1$) and $N_0^*=N-N_1^*$ is the number of zeros
(i.e. $y_i=0$). Let $N_1=\Exp(N_1^*)$ and $N_0=\Exp(N_0^*)$. Following 
\cite{wang2021nonuniform}, we assume that $\alpha_{\rt}\to-\infty$ as
$N\to\infty$, which implies that, under appropriate moment conditions,
\begin{equation}\label{eq:2}
  \frac{N_1^*}{N_0^*} = \frac{N_1}{N_0} \{1+\oo\}
  =\frac{\Exp\{p(\x;\btheta_{\rt})\}}{1-\Exp\{p(\x;\btheta_{\rt})\}}\{1+\oo\}
  =\Exp\{e^{\alpha_{\rt}+f(\x;\bbeta_{\rt})}\}+\oo
  \rightarrow 0, 
\end{equation}
almost surely.
Under this assumption, the asymptotic variance of the full data maximum
likelihood estimator (MLE) is of order $1/N_{1}$ instead of $1/N$,
which is slower because $N_{1}/N\to 0$. 
The value $N_{1}$ can be considered as an intuitive surrogate of
$N_{1}^{*}$, which indicates that the estimation efficiency is
determined by the number of rare ones.\black  Therefore,
we can keep all the ones and sample the zeros to save computational costs. There
could be a variance inflation due to aggressive subsampling, and
\cite{wang2021nonuniform} developed optimal subsampling functions to reduce
the variance inflation. Specifically, the authors proposed non-uniform
  optimal sampling functions under the A- and L-optimality criteria,
  respectively, as follows: 
$\varphi_{\mmse}^{\mathrm{scale}}(\x)\propto p(\x;\btheta_{\rt})\|\M^{-1}\dg(\x;\btheta_{\rt})\|$
 and $\varphi_{\mvc}^{{\mathrm{scale}}}(\x)\propto
 p(\x;\btheta_{\rt})\|\dg(\x;\btheta_{\rt})\|$, where
 $\M=\Exp\{e^{f(\x;\bbeta_{\rt})}\dg^{\otimes2}(\x;\btheta_{\rt})\}$ and
 $\dg(\x;\btheta)$ denotes the derivative of $g(\x;\btheta)$
 with respect to $\btheta$. 
However, 
the sampling functions $\varphi_{\mmse}^{\mathrm{scale}}(\x)$ and $\varphi_{\mvc}^{\mathrm{scale}}(\x)$ proposed in
\cite{wang2021nonuniform} depend on the scale of $\x$, and may not
perform well for certain measurement scale of $\x$. For example, if
  $g(\x;\btheta_{\rt})=\alpha_{\rt}+\x\tp\btheta_{\rt}$, then
  $\varphi_{\mvc}^{\mathrm{scale}}(\x)$ is proportional to $1+\|\x\|$, which will be
  influenced by the scale of $\x$. Similarly, scale changes in $\x$ may also
  change $\varphi_{\mmse}^{\mathrm{scale}}(\x)$, although the impact may not be in the same
  direction, as demonstrated in Figure~\ref{fig:intro}. Besides parameter
estimation, variable selection is another important topic, which has not been
studied in the literature on rare-events data. This work aims to fill this gap.

\section{Nonuniform sampling with variable selection for rare-events data}\label{sec:varsel}
 
The adaptive lasso \citep{zou2006adaptive, zhang2007adaptive} is a popular
variable selection method because it has oracle properties and is easy to
implement. The full data adaptive lasso estimator is
\begin{equation}
\label{eq:f-las}
\htheta_{\rf}^{\radp}:=\mathop{\arg\max}_{\btheta}\left\{\sumn[y_ig(\x_i;\btheta) -
\log\{1+e^{g(\x_i;\btheta)}\}] -
\lambda_N\sumjp\frac{|\beta_{(j)}|}{|\hat{\beta}_{\rp(j)}|^{\gamma}}\right\},
\end{equation}
where $\lambda_N$ and $\gamma$ are tuning parameters, and $\hat{\bbeta}_{\rp}$ is a consistent
pilot estimator of $\bbeta_t$.
 In
practice, it is common to set $\gamma=1$. 
In the literature, iterative algorithms such as coordinate descent are commonly
used to solve the adaptive lasso \citep{friedman2007pathwise}.  However, their
computational demand can become prohibitive when dealing with massive data.
It is feasible to alleviate the computational burden by subsampling
zeros and create a smaller subset of data for adaptive lasso.
To be specific, consider Algorithm~\ref{alg:poi}.
\begin{algorithm}[H]%
  \caption{Poisson subsampling algorithm}
  \label{alg:poi}
  \begin{algorithmic}[1]
    \STATE For $i=1,...,N$:
    \IF {$y_i=1$}
    \STATE include $(\x_i,y_i)$ in the subsample;
    \ELSE
    \STATE compute $\varphi(\x_i)$ and generate $u_i\sim U[0,1]$;
    \IF {$u_i\leq\pi(\x_i,y_i)$}
    \STATE include $(\x_i,y_i)$ and record $\rho\varphi(\x_i)$ in the subsample;
    \ENDIF
    \ENDIF
  \end{algorithmic}
\end{algorithm}
The inclusion probability in Algorithm~\ref{alg:poi} for the $i$th observation
is $\pi(\x_i,y_i)=y_i+(1-y_i)\rho\varphi(\x_i)$, where $\rho$ is the baseline
sampling rate for the zeros and $\varphi(\x)>0$ satisfies
$\Exp\{\varphi(\x)\}=1$.
Let the subsample from Algorithm~\ref{alg:poi} be
$\{\x_i^{\rs},y_i^{\rs}\}_{i=1}^{N_{\rs}^{*}}$, which 
is biased since $\pi(\x_i,y_i)$'s depend on the responses. We introduce an
inverse probability weighting (IPW) adaptive lasso estimator to correct for the
bias, defined as
\begin{equation}
\label{eq:w-las}
\htheta_{\rw}^{\radp}:=\mathop{\arg\max}_{\btheta}\left\{\sum_{i=1}^{N_{\rs}^{*}} \frac{[y_i^{\rs}g(\x_i^{\rs};\btheta) -
\log\{1+e^{g(\x_i^{\rs};\btheta)}\}]}{\pi(\x_i^{\rs},y_i^{\rs})}-\lambda_N\sumjp\frac{|\beta_{(j)}|}{|\hat{\beta}_{\rp(j)}|^{\gamma}}\right\}.
\end{equation}

We use $\cA$ to denote the set of
indices of active variables, i.e., $\cA=\{j:\beta_{\rt(j)}\neq0\}$, and $\cAc$ to
denote the set of indices of inactive variables, i.e.,
$\cAc=\{j:\beta_{\rt(j)}=0\}$. Let $s_{N}=\#\cA$.
We list some general assumptions for theoretical analysis. 
\begin{assumption}
\label{asm:a1}
The first, second and third derivatives of $f(\x;\btheta)$ and
$e^{f(\x;\bbeta)}f(\x;\bbeta)$ with respect to $\bbeta$ are bounded by a square
integrable random variable $B(\x)$. The matrix $\Exp\left\{\dg^{\otimes2}(\x;\btheta)\right\}$
is finite and positive definite.
\end{assumption}
\begin{assumption}
\label{asm:a3}
We assume that $c_N=e^{\alpha_{\rt}}/\rho\to c$, where
$0\leq c<\infty$ is a constant. 
\end{assumption}
\begin{assumption}\label{asm:a4ab}
\assumptionitem{asm:a4} The integral
$\Exp\left[\left\{\varphi(\x)+\varphi^{-1}(\x)\right\}B^2(\x)\right]$ is
finite;
\assumptionitem{asm:a5} The integral $\Exp\left\{e^{f(\x;\bbeta)}\varphi^{-1}(\x)B(\x)\right\}$ is
finite.
\end{assumption}
\begin{assumption}\label{asm:a6}
  Letting $b_{N}=\min\{|\beta_{\rt(j)}|: j\in\cA\}$, we assume that there
  exist a constant $c_{0}>0$ and a sequence $r_{N}$ such that
  $\Pr(\min_{j\in\cA}|\hat{\beta}_{\rp(j)}|\le c_{0}b_{N})\to 0$ and
  $\max_{j\in\cAc}|\hat{\beta}_{\rp(j)}|=O_{P}(r_{N}^{-1})$ as $N\to\infty$.
  Assume that
  $\Pr[\max_{i=1,\ldots,N}\{B(\x_{i})+B(\x_{i})/\varphi(\x_{i})\}<C]\to 1$, for
  some constant $C$, 
  \begin{align}\label{eq:hd-cond}
    \frac{\sqrt{\log s_{N}}}{\sqrt{N_{1}}b_{N}}
    +\frac{\sqrt{N_{1}\log p_{N}}}{\lambda_{N}r_{N}^{\gamma}}
    +\frac{\sqrt{s_{N}}}{N_{1}}\max\left\{\log N_{1},\frac{\lambda_{N}}{b_{N}^{\gamma+1}}\right\}
    =\oo.
  \end{align}
\end{assumption}
Assumption~\ref{asm:a1} imposes moment conditions that are standard in the literature, which are satisfied by many widely used models, such as logistic regression.
Assumption~\ref{asm:a3} essentially requires that the sampling rate be not of a higher
order than the imbalance rate. Since $\rho$ is user-specified, this
assumption is easily satisfied in practice by setting $\rho$ to be
a constant multiple of the imbalance rate. 
Assumption~\ref{asm:a4ab}
imposes regularity conditions on $\varphi(\x)$. These are easily satisfied
since $\varphi(\x)$ is user-specified; in practice, we can restrict
$\varphi(\x)$ to be bounded away from zero and infinity.
Assumption~\ref{asm:a6} specifies
conditions on how fast the dimensions $p_{N}$ and $s_{N}$
can diverge as $N$ increases.
The sequence $b_{N}$ represents the signal strength and must be sufficiently
strong, e.g., satisfying $b_{N}/(\sqrt{\log s_{N}/N_{1}})\to\infty$. The sequence $r_{N}$
controls the convergence rate of the pilot estimator for inactive variables.
For diverging dimensions, we require $B(\x_{i})$ and
$B(\x_{i})/\varphi(\x_{i})$ to be bounded with probability approaching 1, rather
than only controlling their expected values as in Assumption~\ref{asm:a4ab}.
This requirement can be relaxed by assuming that $B(\x_{i}) /\min\{\varphi(\x_{i}),1\}$ is
sub-Gaussian, provided the left-hand side of~\eqref{eq:hd-cond} is multiplied by
$(\log N)^{3/2}$. 
Moreover, the dimension cannot grow too rapidly; for instance, we require $\log
p_{N}=o(\lambda_{N}^{2}r_{N}^{2\gamma}/N_{1})$. Finally, the third term
in~\eqref{eq:hd-cond} controls $\lambda_{N}$ to ensure that the bias of the
adaptive lasso estimator is asymptotically negligible. 
We first study the asymptotic properties of $\htheta_{\rw}^{\radp}$
in the following theorem.
\begin{theorem}
\label{thm:asym-ipw}
Under Assumptions~\ref{asm:a1}-\ref{asm:a3},~\ref{asm:a4}, and~\ref{asm:a6}, if
$\alpha_{\rt}\geq C\log\{\max\{s_{N},\log p_{N}\}/N\}$, where $C$ is a
constant, the
IPW adaptive lasso estimator in~\eqref{eq:w-las} has the
following properties:
\begin{enumerate}[1.]
\item Consistency in variable selection: The estimated active set
  $\hat{\cA}_{\rw}:=\{j:\hat{\beta}_{\rw(j)}^{\radp}\neq0\}$ satisfies that
  $\lim_{N\to\infty}\Pr(\hat{\cA}_{\rw}=\cA)=1$.
\item 
  Asymptotic normality: If $(\sqrt{s_{N}}\lambda_{N})/\sqrt{N_{1}}\to0$, the
  estimator of the active parameter vector satisfies that for any
  $s_{N}$-dimensional vector $\bm{e}_{N}$ with $\|\bm{e}_{N}\|=1$,
\begin{equation*}
\sqrt{N_1}\bm{e}_{N}\tp\V_{\rw(\cA)}^{-1/2}(\htheta_{\rw(\cA)}^{\radp}-\btheta_{\rt(\cA)})\cvd
\Nor(0,1),
\end{equation*}
where
$\V_{\rw(\cA)}=\Exp\left\{e^{f(\x;\bbeta_{\rt})}\right\}\M_{(\cA)}^{-1}\M_{\rw(\cA)}\M_{(\cA)}^{-1}
=\Exp\left\{e^{f(\x;\bbeta_{\rt})}\right\}\left\{\M_{(\cA)}^{-1}+c\V_{\mathrm{sub}(\cA)}\right\}$,
$\M_{(\cA)}=\Exp\left\{e^{f(\x;\bbeta_{\rt})}\dg_{(\cA)}^{\otimes2}(\x;\btheta_{\rt})\right\}$,
$\V_{\mathrm{sub}(\cA)}=\M_{(\cA)}^{-1}
\Exp\left\{\frac{e^{2f(\x;\bbeta_{\rt})}}{\varphi(\x)}
\dg_{(\cA)}^{\otimes2}(\x;\btheta_{\rt})\right\}\M_{(\cA)}^{-1}$, $c=\lim_{N\to\infty}e^{\alpha_{\rt}}/\rho<\infty$, and $\dg_{(\cA)}(\x;\btheta_{\rt})$ consists of the elements of gradient
vector $\dg(\x;\btheta_{\rt})$ with indexes in the active set $\cA$.
\end{enumerate}
\end{theorem}

\begin{corollary}\label{cor:ipw-fixp}
  If $s_{N}=s$ and $p_{N}=p$ are fixed, $\min\{|\beta_{\rt(j)}|:j\in\cA\}>0$, 
  Theorem~\ref{thm:asym-ipw} holds under
  Assumptions~\ref{asm:a1}-\ref{asm:a3},~\ref{asm:a4} and 
  $\lambda_N/\sqrt{N_1}\to0$, $\hat{\bbeta}_{\rp}$ is a
  consistent pilot estimator such that
  $\lambda_N/(\sqrt{N_1}|\hat{\beta}_{\rp(j)}|^{\gamma})\cvp\infty$ for $j\in\cAc$. 
  The asymptotic normality becomes
  $\sqrt{N_1}\V_{\rw(\cA)}^{-1/2}(\htheta_{\rw(\cA)}^{\radp}-\btheta_{\rt(\cA)})\cvd
  \Nor(\0,\I)$. 
\end{corollary}
\begin{remark}\label{rm:rho1}
  Theorem~\ref{thm:asym-ipw} shows that the estimation efficiency of $\htheta_{\rw(\cA)}^{\radp}$ is predominantly
  determined by the number of ones instead of the full data size.
  The term $c\V_{\mathrm{sub}(\cA)}$ is the variance inflation due to
  subsampling. The full data adaptive lasso in
\eqref{eq:f-las} corresponds to the scenario with $\rho=1$ and
$\varphi(\x)=1$, for which $c=\lim_{N\to\infty} e^{\alpha_{\rt}}/\rho=0$.
  Intuitively, $c$ can be interpreted as the imbalance rate in the subsample. 
    Theorem~\ref{thm:asym-ipw} requires $c<\infty$. If $c=\infty$, the
    asymptotic normality still holds but becomes 
  $\sqrt{\rho
  N}\bm{e}_{N}\tp\V_{\mathrm{sub}(\cA)}^{-1/2}(\htheta_{\rw(\cA)}^{\radp}-\btheta_{\rt(\cA)})
  \cvd\Nor(0,1)$. However, Algorithm~\ref{alg:poi} is not recommended in this
  case. It's recommended to subsample both zeros and ones for better estimation
  efficiency in this case.
\end{remark}

Theorem~\ref{thm:asym-ipw} also provides guidance on choosing the sampling rate
$\rho$. If we want to avoid any asymptotic information loss and have sufficient
computational resources, we set $\rho \gg N_1/N_0$, which corresponds to $c=0$
in Theorem~\ref{thm:asym-ipw}.  This includes sufficient zeros so that the
subsampling does not reduce the estimation efficiency compared to using the full
data while still reducing the computational cost. In this case, the function
$\varphi(\x)$ does not appear in the asymptotic variance and does not impact the
estimation efficiency, and we can simply use uniform sampling, i.e.,
$\varphi(\x)=1$.  If computational resources are so limited that we have to
choose a $\rho$ that is much smaller than $N_1/N_0$, then
Algorithm~\ref{alg:poi} is not recommended as pointed out in
Remark~\ref{rm:rho1}, because this corresponds to $c=\infty$.  A practical
scenario is when we want to balance estimation and computational efficiency, and
available computational resources allow us to include all ones and a portion of
zeros. We should set $\rho$ at the same order of $N_1/N_0$, which corresponds
to $c\in(0,\infty)$ in Theorem~\ref{thm:asym-ipw}.  In this scenario we set the
number of zeros in the subsample as a constant multiple of the number of ones
and obtain a relatively balanced subsample.  The function $\varphi(\x)$ controls
the information loss due to subsampling. We discuss how to minimize the
information loss by choosing optimal $\varphi(\x)$ in the following.

We derive optimal functions, where $\varphi_{\mmse}^{\radp}(\x)$ corresponds to the A-optimality
criterion \citep{pukelsheim2006optimal} and $\varphi_{\mvc}^{\radp}(\x)$ corresponds to the L-optimality
criterion \citep{pukelsheim2006optimal} in design of experiments.
  Here, the A-optimality minimizes the trace of the asymptotic variance of
  $\htheta_{\rw(\cA)}^{\radp}$; 
  the L-optimality focuses on the asymptotic variance of a linearly transformed estimator
 $\M_{(\cA)}\htheta_{\rw(\cA)}^{\radp}$, which is proportional to
 $\M_{\rw(\cA)}$. The A-optimality criterion has a more direct interpretation, while an
 advantage of the L-optimality criterion is that the resulting optimal
 function is often faster to calculate.

\begin{proposition}
\label{prop:optlas}
  Under the constraints that $\Exp\{\varphi(\x)\}=1$ and a given $\rho$ at
  the same order of $e^{\alpha_{\rt}}$,
the A-optimal function that minimizes $\text{tr}(\V_{\rw(\cA)})$ is 
\begin{equation}
\label{eq:optA}
\varphi_{\mmse}^{\radp}(\x)=\frac{p(\x;\btheta_{\rt})\|\M_{(\cA)}^{-1}\dg_{(\cA)}(\x;\btheta_{\rt})\|}{
\Exp\left\{p(\x;\btheta_{\rt})\|\M_{(\cA)}^{-1}\dg_{(\cA)}(\x;\btheta_{\rt})\right\}}.
\end{equation}
The L-optimal function that minimizes $\text{tr}(\M_{\rw(\cA)})$ is 
\begin{equation}
\label{eq:optL}
\varphi_{\mvc}^{\radp}(\x)=\frac{p(\x;\btheta_{\rt})\|\dg_{(\cA)}(\x;\btheta_{\rt})\|}{
\Exp\left\{p(\x;\btheta_{\rt})\|\dg_{(\cA)}(\x;\btheta_{\rt})\right\}}.
\end{equation}
\end{proposition}
The constraints on $\varphi(\x)$ and $\rho$ limit our discussion to
the scenario that $\varphi(\x)$ serves as a key factor controlling the
information loss. Beyond this
scenario $\varphi(\x)$ may be trivial, e.g.,
$\varphi(\x)\equiv\infty$ or $\varphi(\x)\equiv1$.

Unlike the optimal sampling function in \cite{wang2021nonuniform},
$\varphi^{\radp}_{\mmse}(\x)$ (or $\varphi^{\radp}_{\mvc}(\x)$)   
relies only on the active variables. This implies that a first-step pilot
estimator given by the adaptive lasso algorithm can benefit from sparse
estimation methods when calculating optimal probabilities. For example,
employing the standard lasso can effectively eliminate a large number of
inactive variables to facilitate the computation of optimal
$\varphi_{\mmse}^{\radp}(\x)$ and $\varphi^{\radp}_{\mvc}(\x)$.
However, in practice, pilot 
estimators are often obtained from a small subsample size, introducing additional
uncertainty. Therefore, it becomes crucial to exercise caution and be
conservative by over-selecting variables during the first step to prevent
the exclusion of important variables. As a consequence, although
theoretically
$\varphi_{\mmse}^{\radp}(\x)$ and $\varphi^{\radp}_{\mvc}(\x)$ do not depend on
inactive variables, they are affected by inactive variables
in practical implementations.

\subsection{Limitation of scale-dependent optimal functions}\label{sec:limit}

Both $\varphi_{\mmse}^{\radp}(\x)$ and $\varphi^{\radp}_{\mvc}(\x)$ presented in Proposition~\ref{prop:optlas} are scale
dependent, i.e., their values are influenced by the scale of the
covariates. 
In practical implementations, inactive variables may also
influence these two sampling functions, because the true active set is unknown
and has to be replaced by
pilot estimates. The influences of inactive variables on
$\varphi_{\mmse}^{\radp}(\x)$ and $\varphi^{\radp}_{\mvc}(\x)$ are similar to
those on $\varphi^{\rmle}_{\mmse}(\x)$ and $\varphi^{\rmle}_{\mvc}(\x)$, so we use
the latter to illustrate these influences with more details here.
Specifically,
$\varphi^{\rmle}_{\mmse}(\x)$ depends on two terms: $p(\x;\btheta_{\rt})$ and
$\|\M^{-1}\dg(\x;\btheta_{\rt})\|$. The first term $p(\x;\btheta_{\rt})$, called the probability term, represents the
probability of having a case. The
second term $\|\M^{-1}\dg(\x;\btheta_{\rt})\|$, called the leveraging term, and its value depends on the scale of covariates. 

Next, we use a simple logistic regression model to illustrate how scaling
affects $\varphi^{\rmle}_{\mmse}(\x)$ and $\varphi^{\rmle}_{\mvc}(\x)$. 
Without loss of generality, let
$\bbeta_{\rt}=(\bbeta^{\tp}_{\rt(\cA)},\bbeta_{\rt(\cAc)}^{\tp})\tp$, where
$\bbeta_{\rt(\cAc)}=\0$ holds for inactive variables. We further assume
that $\x_{(\cA)}$ and $\x_{(\cAc)}$ are independent and
$\Exp(\x_{(\cA)})=\Exp(\x_{(\cAc)})=\0$. In this scenario,
\begin{align*}
&\M=
\Exp\left\{e^{\x_{\cA}\tp\bbeta_{\rt(\cA)}+\x_{(\cAc)}\tp\bbeta_{\rt(\cAc)}}
\begin{pmatrix}
1 & \x_{(\cA)}\tp & \x_{(\cAc)}\tp\\
\x_{(\cA)} & \x_{(\cA)}\x_{(\cA)}\tp & \x_{(\cA)}\x_{(\cAc)}\tp\\
\x_{(\cAc)} & \x_{(\cAc)}\x_{(\cA)}\tp & \x_{(\cAc)}\x_{(\cAc)}\tp\\
\end{pmatrix}
\right\}\\
&=\begin{pmatrix}
\M_{(\cA)} & \Exp\{e^{\x_{(\cA)}\tp\bbeta_{\rt(\cA)}}(1,\x_{(\cA)}\tp)\tp\}\Exp\{\x_{(\cAc)}\tp\} \\
 \Exp\{\x_{(\cAc)}\}\Exp\{e^{\x_{(\cA)}\tp\bbeta_{\rt(\cA)}}(1,\x_{(\cA)}\tp)\}
 & \Exp\{e^{\x_{(\cA)}\tp\bbeta_{\rt(\cA)}}\}\Exp\{\x_{(\cAc)}\x_{(\cAc)}\tp\} \\
\end{pmatrix}\\
&=\begin{pmatrix}
\Exp\{e^{\x_{(\cA)}\tp\bbeta_{\rt(\cA)}}\} & \Exp\{e^{\x_{(\cA)}\tp\bbeta_{\rt(\cA)}}\x_{(\cA)}\tp\} & \0\\
\Exp\{e^{\x_{(\cA)}\tp\bbeta_{\rt(\cA)}}\x_{(\cA)}\} & \Exp\{e^{\x_{(\cA)}\tp\bbeta_{\rt(\cA)}}\x_{(\cA)}\x_{(\cA)}\tp\} & \0\\
\0 & \0 & \Exp\{e^{\x_{(\cA)}\tp\bbeta_{\rt(\cA)}}\}\Exp\{\x_{(\cAc)}\x_{(\cAc)}\tp\}
\end{pmatrix}.
\end{align*}
If the components of $\x_{(\cAc)}$ are independent,
denoting $\Var(x_{(j)})=\Exp(x_{(j)}^{2})$, $j\in\cAc$,
\begin{equation*}
\varphi^{\rmle}_{\mmse}(\x)\propto
p(\x;\btheta_{\rt})\|\M^{-1}\x\|
=p(\x;\btheta_{\rt})\sqrt{K_{\cA}
+\Exp\{e^{\x_{(\cA)}\tp\bbeta_{\rt}}\}^{-2}\sum_{j\in\cAc}\frac{x_{(j)}^2}{\Var(x_{(j)})^2}},
\end{equation*}
where
$K_{\cA}=\|\M_{(\cA)}^{-1}(1,\x_{(\cA)}\tp)\tp\|^{2}$
is a value depending only on active variables. Similarly,
$\varphi^{\rmle}_{\mvc}$ also has a decomposition without the independences:
\begin{equation*}
\varphi^{\rmle}_{\mvc}(\x)\propto
p(\x;\btheta_{\rt})\|(1,\x\tp)\tp\|
=p(\x;\btheta_{\rt})\sqrt{1+\sum_{j\in\cA}x_{(j)}^2+\sum_{j\in\cAc}x_{(j)}^2}.
\end{equation*}
It is important to note that the probability term $p(\x;\bbeta_{\rt})$ remains
unaffected by the scale of $\x$ because its value depends on
$\alpha_{\rt}+\x\tp\bbeta_{\rt}$ in the logistic regression model. Since
$\beta_{\rt(j)}=0$ for inactive variables, any rescaling of $x_{(j)}$
will not change
$p(\x;\bbeta_{\rt})$. However, the leveraging term is indeed scale-dependent.
For example, if we re-scale all inactive variables $\x_{\cAc}$
to $\tau\x_{\cAc}$, it would lead to changes in the values of the optimal
functions as
\begin{equation*}
\varphi^{\rmle}_{\mmse}(\x)\propto
p(\x;\btheta_{\rt})\sqrt{K_{\cA}
+\frac{1}{\tau^2}\Exp\{e^{\x_{(\cA)}\tp\bbeta_{\rt}}\}^{-2}\sum_{j\in\cAc}\frac{x_{(j)}^2}{\Var(x_{(j)})^2}},
\end{equation*}
\begin{equation*}
\varphi^{\rmle}_{\mvc}(\x)\propto
p(\x;\btheta_{\rt})\sqrt{1+\sum_{j\in\cA}x_{(j)}^2+\tau^2\sum_{j\in\cAc}x_{(j)}^2}.
\end{equation*}
If $\tau$ is extremely small, the contribution of inactive variables to
$\varphi^{\rmle}_{\mmse}(\x)$ would be highly inflated. On the other hand, if
$\tau$ is extremely large, the contribution of inactive variables to
$\varphi^{\rmle}_{\mvc}(\x)$ would be inflated. 
In fact, we have the following limits involving the denominators:
\begin{align*}
   &\varphi^{\rmle}_{\mmse}(\x)
    \to\frac{e^{\x_{(\cA)}\tp\bbeta_{\rt(\cA)}}
      \sqrt{\sum_{j\in\cAc}\frac{x_{(j)}^2}{\Var(x_{(j)})^2}}}
      {\Exp\left\{e^{\x_{(\cA)}\tp\bbeta_{\rt(\cA)}}\right\}
        \Exp\left\{\sqrt{\sum_{j\in\cAc}\frac{x_{(j)}^2}{\Var(x_{(j)})^2}}
      \right\}},\text{ as } \alpha_{\rt}\to-\infty, \tau\to0, \\
   &\varphi^{\rmle}_{\mvc}(\x)\to
      \frac{e^{\x_{(\cA)}\tp\bbeta_{\rt(\cA)}}
      \sqrt{\sum_{j\in\cAc}x_{(j)}^2}}
      {\Exp\left\{e^{\x_{(\cA)}\tp\bbeta_{\rt(\cA)}}\right\}
      \Exp\left\{\sqrt{\sum_{j\in\cAc}x_{(j)}^2}\right\}}, 
      \text{ as }\alpha_{\rt}\to-\infty, \tau\to\infty.
\end{align*}
Thus, the scale of $\x$ has an
effect on both $\varphi^{\rmle}_{\mmse}(\x)$ and $\varphi^{\rmle}_{\mvc}(\x)$,
although in opposing directions.

The scale-dependence of optimal probabilities could mislead the subsampling
procedure, potentially resulting in a poor performance. If a variable
$x_{(j)}$ is inactive, i.e., $\beta_{\rt(j)}=0$, then sampling probabilities
should not highly rely on $x_{(j)}$ as it does not contribute to predicting the
responses. However, for certain scaling of $x_{(j)}$, the optimal sampling
function would be highly affected by this inactive variable. For example, if the
scale of $x_{(j)}$ is very small, the leveraging term of the A-optimal function
$\varphi^{\rmle}_{\mmse}(\x)$ would be dominated by $x_{(j)}$. In this case, the
inclusion of a data point $(\x_i,y_i)$ would be highly determined by the
inactive $j$-th component of $\x_i$'s, though it is more reasonable to determine
sampling probabilities using active variables only. A similar issue exists for
$\varphi^{\rmle}_{\mvc}(\x)$ in the opposite direction; it is dominated by
components of $\x_i$'s with very large magnitudes.
Although $\varphi^{\radp}_{\mmse}(\x)$ and $\varphi^{\radp}_{\mvc}(\x)$ do not
depend on inactive variables, they are still affected by scales of inactive
variables because we need to estimate the active set $\cA$. It is desirable to not exclude important variables and thus we tend to include more variables %
in the pilot estimation. Even for active variables, %
scale-dependency may cause issues as well
because if $x_{(j)}\beta_{\rt(j)}$ is small, it contributes little to the responses. However, $x_{(j)}$ may dominate optimal probabilities if its scale is
inappropriate. %

In more general cases where components of $\x_{(\cAc)}$ are dependent,
theoretical analysis is less transparent and therefore we use a numerical
example to illustrate how scaling may affect optimal probabilities in
Figure~\ref{fig:prbexpr}. A training data set of size $N=250000$ is generated
from two normal distributions such that $\x_{(\cA)}=(x_{(1)},x_{(2)})\tp
\sim\Nor(\0,\I)$
and $\x_{(\cAc)}=(x_{(3)},...,x_{(6)})\tp\sim \Nor(\0,\bSigma)$ with the $(i,j)$th
element of $\bSigma$ being
$\Sigma_{ij}=0.5^{|i-j|}$ for $1\leq i,j\leq 4$. Here, $\x_{(\cA)}$ and
$\x_{(\cAc)}$
are independent. We set $\alpha_{\rt}=-6.5$,
$\bbeta_{\rt}=(1,1,0,...,0)$, and
$g(\x;\btheta)=\alpha+\x\tp\bbeta$. We use
$p(\x;\btheta_{\rt})^2\|\L\tilde{\x}_j\|^2/p(\x;\btheta_{\rt})^2\|\L\x\|^2$
as a measure of the $j$-th element's contribution to sampling probabilities
through the leveraging term, where
$\tilde{\x}_j=(0,0,...,x_{(j)},0,...,0)$, and $\L$ equals $\M^{-1}$ and
  $\I$, for the A- and L-optimal probabilities, respectively.

\begin{figure}[t] %
  \centering 
  \begin{subfigure}{\textwidth}
    \centering
    \includegraphics[width=0.355\textwidth]{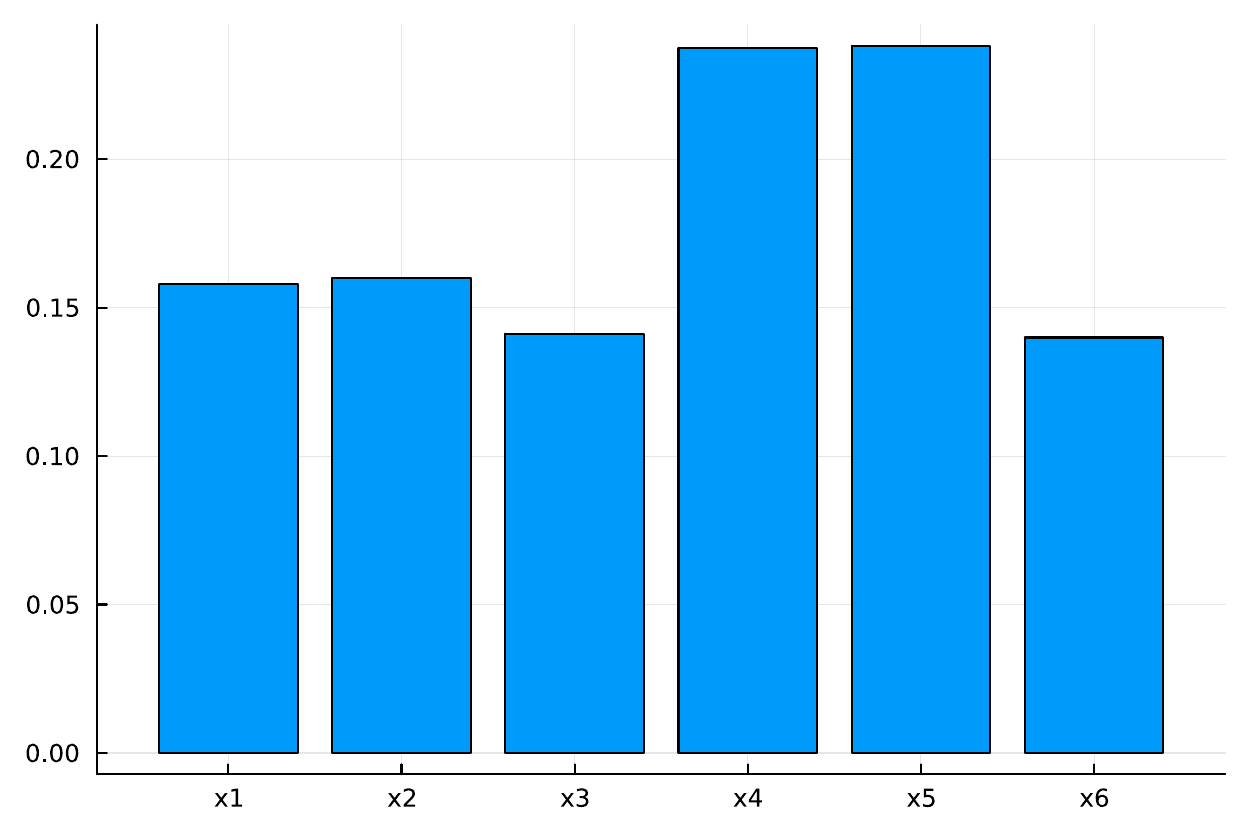}
    \includegraphics[width=0.355\textwidth]{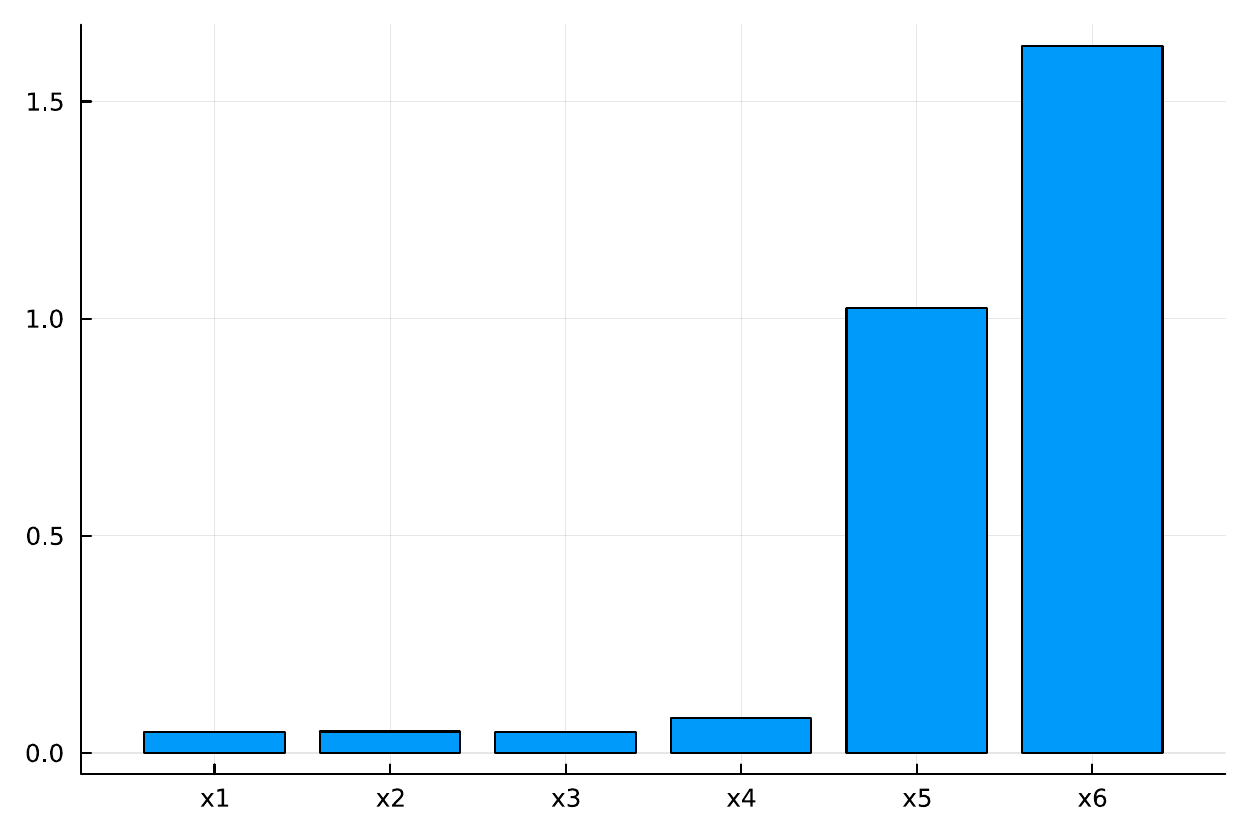}
    \caption{A-optimality probabilities (before and after rescaling)}
  \end{subfigure}
  \begin{subfigure}{\textwidth}
    \centering
    \includegraphics[width=0.355\textwidth]{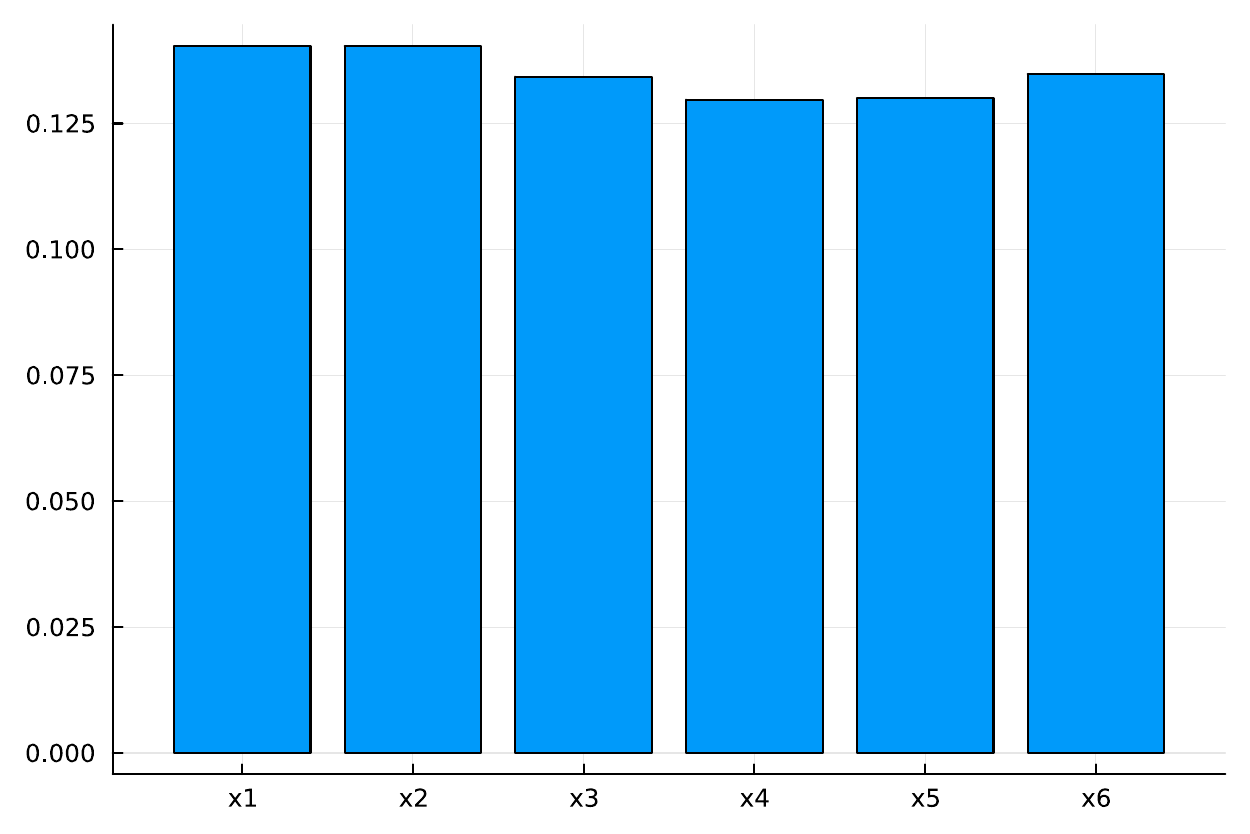}
    \includegraphics[width=0.355\textwidth]{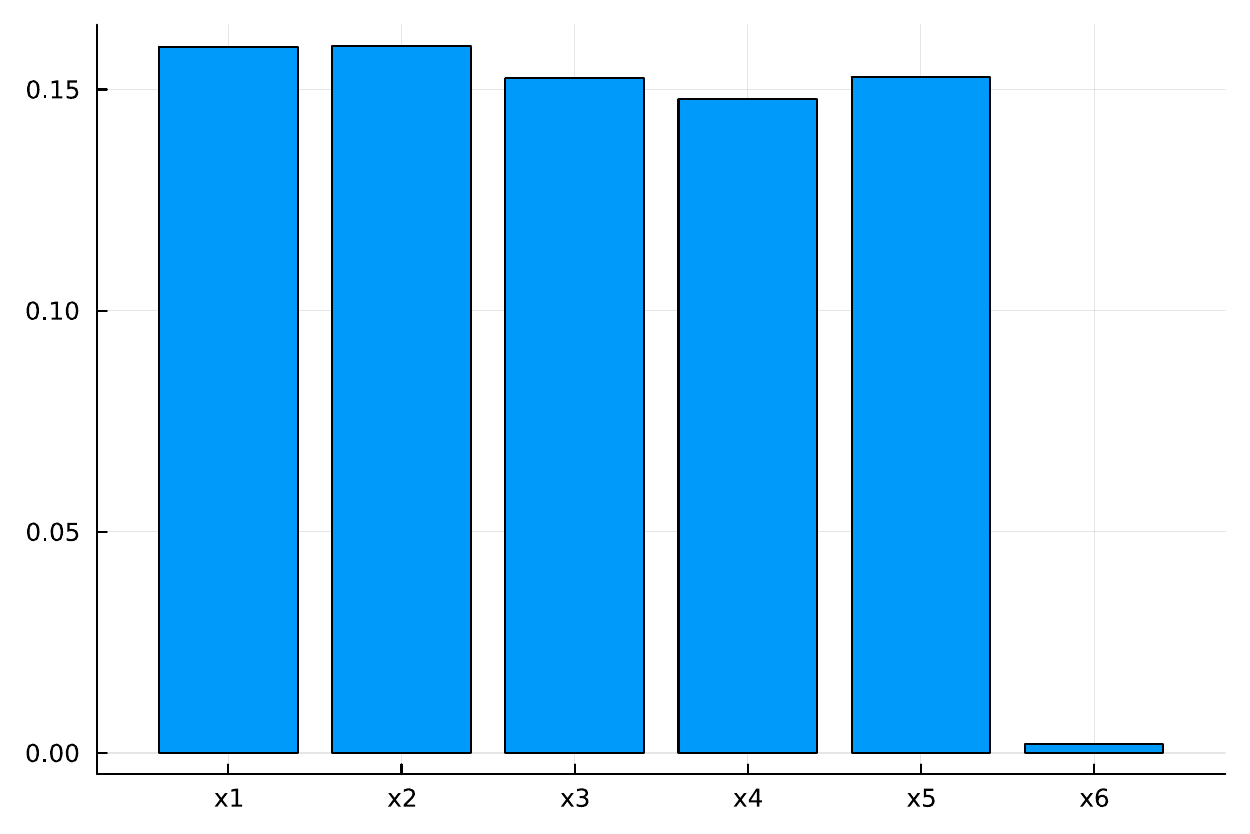}
    \caption{L-optimality probabilities (before and after rescaling)}
  \end{subfigure}
  \caption{Contributions of individual variables to sampling probabilities. The bar charts in the
    left panel present contributions with the
      original scale of $x_{(6)}$; those in the
    right panel present contributions after re-scaling $x_{(6)}$.}
  \label{fig:prbexpr}
\end{figure}

In Figure~\ref{fig:prbexpr}, the left panel illustrates the contributions of
individual variables to the sampling probabilities at the original data scale of
$x_{(6)}$, while the right panel shows the contributions after rescaling:
multiply 0.1 to $x_{(6)}$. It is observed that the contributions of two inactive
variables, $x_{(4)}$ and $x_{(5)}$, are relatively large for
$\varphi^{\rmle}_{\mmse}(\x)$ when $x_{(6)}$ is in the
original scale. Although this is not desirable, the contributions of
active variables $x_{(1)}$ and $x_{(2)}$ are not very small. However, after
re-scaling the inactive variable $x_{(6)}$, the contributions of $x_{(5)}$ and
$x_{(6)}$ dominate the optimal probability
  $\varphi^{\rmle}_{\mmse}(\x)$, which is unfavorable when compared with the
results obtained with the original scale of $x_{(6)}$. On the other hand,
the effects of each variable on $\varphi^{\rmle}_{\mvc}(\x)$ are
  distributed evenly with the original scale of $x_{(6)}$ because the
variances of all variables are the same. Balanced contributions among all
variables are often not suitable for high dimensional data when the dimension of
active variables is low. This is because the contributions of inactive variables
can overshadow those of active variables.
After re-scaling $x_{(6)}$, the
contribution of $x_{(6)}$ to $\varphi^{\rmle}_{\mvc}$ becomes negligible, while
the contributions of other inactive variables remain unchanged. Thus, the issue
of balanced contribution still exists. To overcome this challenge, we propose a
scale-invariant optimal probability in the next section, aiming to mitigate the
adverse effects of scaling dependence.

\subsection{Scale-invariant optimal function}\label{sec:newop}

As previously discussed, scaling-dependent optimal probabilities
may impact the performance of variable selection in practice. To address the
issue, we propose to construct a scale-invariant optimal function
by focusing on the prediction error of an estimator $\htheta$, defined as
\begin{align*}
  \mathrm{MSPE}(\htheta)
  =\Exp_{\x}\left[\left\{p(\x;\htheta)-p(\x;\btheta_{\rt})\right\}^2\right]
  =\int\left\{p(\x;\htheta)-p(\x;\btheta_{\rt})\right\}^2\ud\Pr_{\x},
\end{align*}
where $\Pr_{\x}$ is the distribution measure of $\x$. 
The probability term
$p(\x;\btheta_{\rt})$ involves both the covariates $\x$ and the parameter vector
$\btheta_{\rt}$, and it often does not depend on the scale of $\x$. For example,
in the logistic regression model, the value $p(\x;\btheta_{\rt})$ is only
related to $\x\tp\bbeta_{\rt}$. If we change the scale of $x_{(j)}$, the value
of $\btheta_{\rt}$ would change accordingly under the same data-generating
model and so $p(\x;\btheta_{\rt})$ remains the same. Thus, re-scaling covariates
would not
affect this criterion. In the following, we give an optimal function that minimizes the prediction error.

\begin{theorem}
\label{thm:optprb}
Under the assumptions of Theorem~\ref{thm:asym-ipw}, if the number of
active variables $s_{N}$ is fixed, the prediction error of the IPW
adaptive lasso estimator defined in~\eqref{eq:w-las} satisfies 
\begin{equation}\label{eq:1}
N_1e^{-2\alpha_{\rt}}\mathrm{MSPE}(\htheta^{\radp}_{\rw(\cA)})
\cvd\Exp^{-1}\left\{e^{f(\x;\bbeta_{\rt})}\right\}\W_{(\cA)}\tp\M_{\rw(\cA)}^{1/2}\M_{(\cA)}^{-1}
\bOmega_{(\cA)}\M_{(\cA)}^{-1}\M_{\rw(\cA)}^{1/2}\W_{(\cA)},
\end{equation}
where $\W_{(\cA)}\sim\Nor(\0,\I)$,
and
$\bOmega_{(\cA)}=\Exp\left[e^{2f(\x;\bbeta_{\rt})}\dg_{(\cA)}^{\otimes2}(\x,\btheta_{\rt})\right]$. The
optimal function that minimizes 
the asymptotic mean of the
prediction error in (\ref{eq:1}) is given as
\begin{equation}
\label{eq:optpr}
\varphi_{\mpr}^{\radp}(\x)
=\frac{p(\x;\btheta_{\rt})\|\bOmega_{(\cA)}^{\frac{1}{2}}\M_{(\cA)}^{-1}\dg_{(\cA)}(\x;\btheta_{\rt})\|}
{\Exp\left[p(\x;\btheta_{\rt})\|\bOmega_{(\cA)}^{\frac{1}{2}}\M_{(\cA)}^{-1}\dg_{(\cA)}(\x;\btheta_{\rt})\|\right]}.
\end{equation}
\end{theorem}

We refer to this prediction-oriented criterion as the P-optimality criterion. As we expect, the optimal
function in (\ref{eq:optpr}) is unaffected by the
scale of $\x$ for a class of functions $g$. We present this formally in the
following proposition. 

\begin{proposition}\label{prop:invar}
The $\varphi^{\radp}_{\mpr}(\x)$ is invariant to scale changes of $\x$,
if $g(\x;\btheta)$ satisfies that for every non-singular matrix
$\A$ there exists a non-singular matrix $\B$, such that 
\begin{equation}
\label{eq:invg}
g(\A\x;\B\tp\btheta)=g(\x;\btheta),
\end{equation}
\end{proposition}
\begin{remark}
The condition in \eqref{eq:invg} is not restrictive and it is easy to
satisfy. One example of $g(\x;\btheta)$ that satisfies the condition is
a linear function $g(\x;\btheta)=\alpha+\x\tp\bbeta$, which corresponds to the
logistic regression. The condition is also satisfied by
more complex models. For example, consider an $L$-layer neural network with
\begin{equation*}
g(\x;\bm{W}^1,\bm{W}^2,...,\bm{W}^L,\bm{b}^1,...,\bm{b}^L)
=f^L(f^{L-1}(...f^1(\x\tp\bm{W}^1+\bm{b}^1))\tp\bm{W}+\bm{b}^L),
\end{equation*}
where $\bm{W}^l$ are the weights and
$\bm{b}^l$ are the biases in each layer, $l=1,2,...,L$. If $\x$ is rescaled to
$\A\x$, we can change $\bm{W}^1$ to $(\A^T)^{-1}\bm{W}^1$ so that
\begin{align*}
&g(\bm{A}\x;(\A^T)^{-1}\bm{W}^1,\bm{W}^2,...,\bm{W}^L,\bm{b}^1,...,\bm{b}^L)
=f^L(f^{L-1}(...f^1(\x\tp\bm{W}^1+\bm{b}^1))\tp\bm{W}+\bm{b}^L)\\
&=g(\x;\bm{W}^1,\bm{W}^2,...,\bm{W}^L,\bm{b}^1,...,\bm{b}^L).
\end{align*}
\end{remark}
Next, we again use the simplified logistic model in Section~\ref{sec:limit} to
illustrate the invariance property of $\varphi_{\mpr}^{\radp}(\x)$ to scale
changes. The numerical results are shown in Figure~\ref{fig:prbexpr2}: the left
panel presents variable contributions with the original scale, and the right
panel shows their contributions after rescaling. Regardless of which variables
are rescaled or how they are rescaled, their contributions remain unchanged,
confirming the scale-invariant property of $\varphi_{\mpr}^{\radp}(\x)$.
Furthermore, the previously observed low contributions of active variables to
$\varphi_{\mmse}(\x)$ in Figure~\ref{fig:prbexpr}(a) are alleviated in
Figure~\ref{fig:prbexpr2}, and active variables now contribute more than the
inactive variables.
\begin{figure}[H]
  \centering 
  \begin{subfigure}{0.355\textwidth}
    \includegraphics[width=\textwidth]{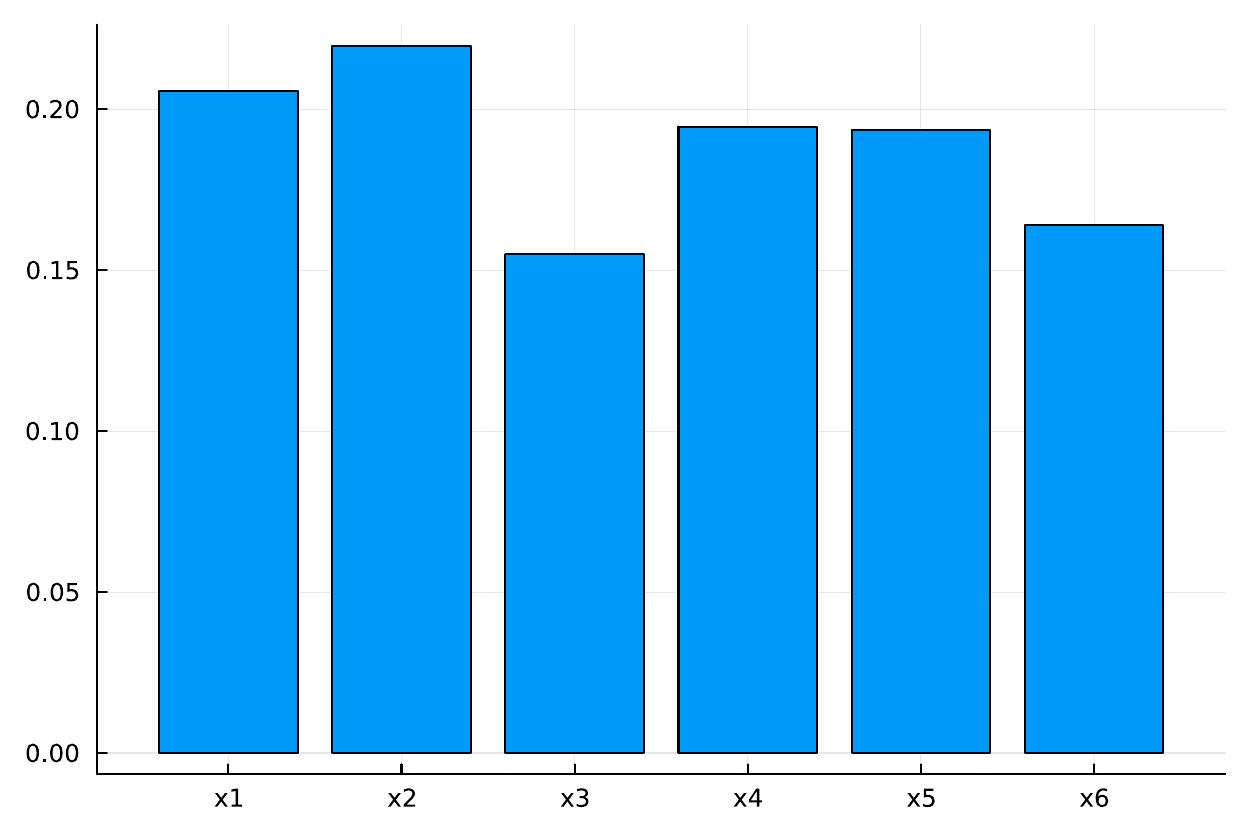}
  \end{subfigure}
  \begin{subfigure}{0.355\textwidth}
    \includegraphics[width=\textwidth]{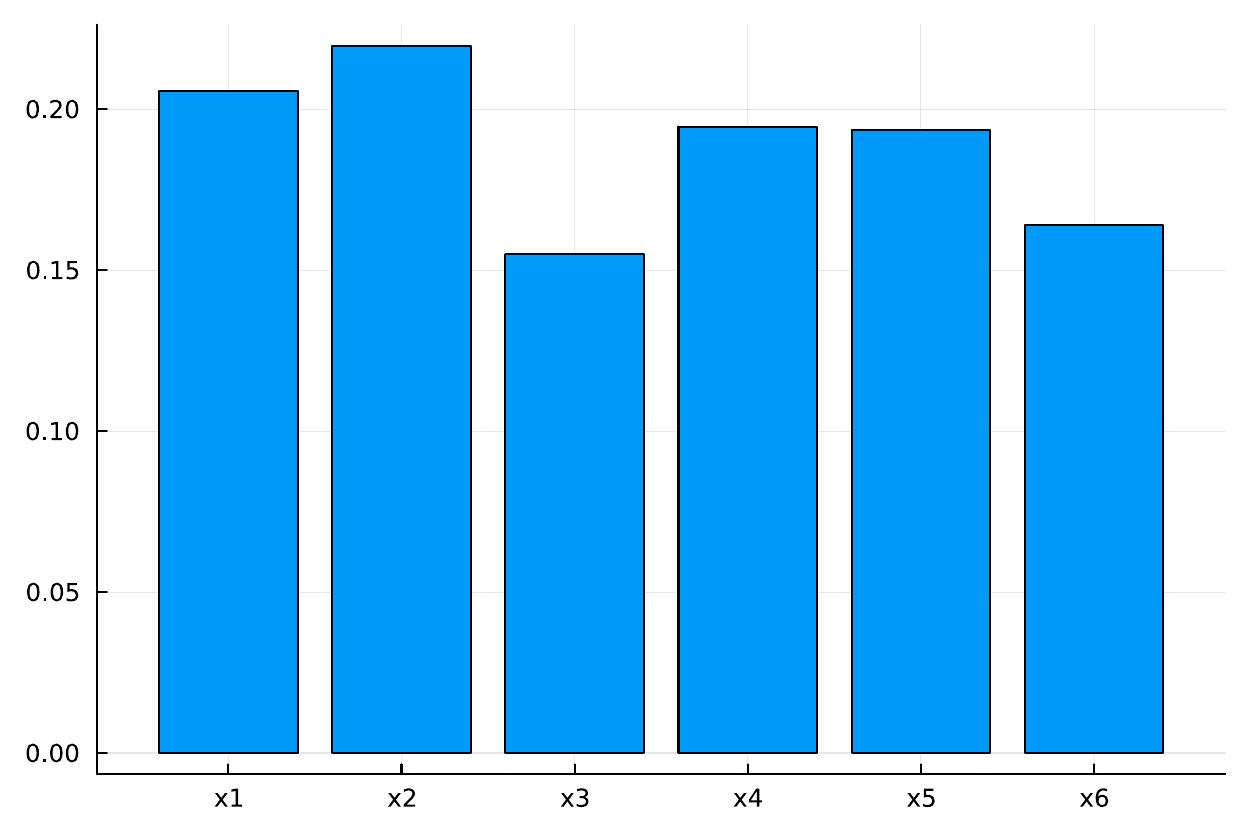}
  \end{subfigure}
  \caption{Bar charts for contributions of individual variables to the P-optimal
    probabilities. The left panel presents the contributions with the original
    scale and the right panel presents the contributions after re-scaling.}
  \label{fig:prbexpr2}
\end{figure}

\section{Penalized MSCL estimator}\label{sec:theory}

The IPW estimator in \eqref{eq:w-las} is not the most efficient estimator, because it assigns smaller
weights for more informative data points with larger sampling probabilities. To improve the estimation efficiency, we propose
 the
penalized MSCL estimator for variable selection: 
\begin{equation}
\label{eq:lik-las}
\htheta_{\ruw}^{\radp}:=\mathop{\arg\max}_{\btheta}\left\{\sum_{i=1}^{N_{\rs}^{*}}[y_i^{\rs}g(\x_i^{\rs};\btheta) -
\log\{1+e^{g(\x_i^{\rs};\btheta)+l_i^{\rs}}\}] - \lambda_N\sumjp\frac{|\beta_{(j)}|}{|\hat{\beta}_{\rp(j)}|^{\gamma}}\right\},
\end{equation}
where $l_i^{\rs}=-\log\left\{\rho\varphi(\x_i^{\rs})\right\}$. 
Here is how the off-set term comes in. Let  $\Delta_{i}$ denote the event that
observation $(y_{i}^{\rs},\x_{i}^{\rs})$ is included in the subsample. Then, we have
\begin{equation}\label{eq:lik-correct}
\Pr(y_{i}^{\rs}=1|\x_{i}^{\rs},\Delta_{i})
=\frac{\Pr(\Delta_{i}|y_{i}^{\rs}=1,\x_{i}^{\rs})\Pr(y_{i}^{\rs}=1|\x_{i}^{\rs})}{\Pr(\Delta_{i}|\x_{i}^{\rs})}
=\frac{e^{g(\x_{i}^{\rs};\btheta)+l_{i}^{\rs}}}{1+e^{g(\x_{i}^{\rs};\btheta)+l_{i}^{\rs}}}.
\end{equation}
Note that \eqref{eq:lik-correct} gives the conditional likelihood of the
selected subsample, and we use it to obtain~\eqref{eq:lik-las}.
More details on deriving \eqref{eq:lik-correct} are in Section~\ref{sec:lik-fixp}. Considering 
The MSCL estimator introduced in \cite{wang2021nonuniform} is defined as the minimizer of the objective function in~\eqref{eq:lik-las}, excluding the penalization term. In this paper, we extend this approach by proposing a penalized MSCL estimator to ensure model sparsity.
We present the oracle properties of the penalized MSCL estimator in the
following theorem.
\begin{theorem}
\label{thm:asym-las}
Under Assumptions~\ref{asm:a1}-\ref{asm:a3},
~\ref{asm:a5}, and~\ref{asm:a6}, 
if $\alpha_{\rt}\geq C\log\{\max\{s_{N},\log p_{N}\}/N\}$, where $C$ is a constant,
the estimator based on MSCL function with adaptive
lasso penalty in~\eqref{eq:lik-las}  
has the following properties:
\begin{enumerate}[1.]
\item Consistency in variable selection: The estimated active set
  $\hat{\cA}_{\ruw}:=\{j:\hat{\beta}_{\ruw(j)}^{\radp}\neq0\}$ satisfies that
  $\lim_{N\to\infty}\Pr(\hat{\cA}_{\ruw}=\cA)=1.$
\item Asymptotic normality: If $(\sqrt{s_{N}}\lambda_{N}/\sqrt{N_{1}})\to0$, the
estimator of the active parameter vector satisfies that for any
$s_{N}$-dimensional vector $\bm{e}_{N}$ with $\|\bm{e}_{N}\|=1$,
  \begin{equation}
  \sqrt{N_1}\bm{e}_{N}\tp\V_{\ruw(\cA)}^{-1/2}(\htheta_{\ruw(\cA)}^{\radp}-\btheta_{\rt(\cA)})\cvd
  \Nor(0,1),
  \end{equation}
  where $\V_{\ruw(\cA)}=\Exp\left\{ e^{f(\x;\bbeta_{\rt})}
  \right\}\bLambda_{\ruw(\cA)}^{-1}$
  and $\bLambda_{\ruw(\cA)}=\Exp\left\{
    \frac{e^{f(\x;\bbeta_{\rt})}\dg_{(\cA)}^{\otimes2}(\x;\btheta_{\rt})}{1+c\varphi^{-1}(\x)e^{f(\x;\bbeta_{\rt})}}
    \right\}$.
\end{enumerate}
\end{theorem}
\begin{corollary}\label{cor:lik-fixp}
  If $s_{N}=s$ and $p_{N}=p$ are fixed, and $\min\{|\beta_{\rt(j)}|:j\in\cA\}>0$, 
  then
  Theorem~\ref{thm:asym-las} holds under
  Assumptions~\ref{asm:a1}-\ref{asm:a3},~\ref{asm:a5} when
  $\lambda_N/\sqrt{N_1}\to0$, and $\hat{\bbeta}_{\rp}$ is a
  consistent pilot estimator such that
  $\lambda_N/(\sqrt{N_1}|\hat{\beta}_{\rp(j)}|^{\gamma})\cvp\infty$ for $j\in\cAc$. 
  The asymptotic normality becomes
  $\sqrt{N_1}\V_{\ruw(\cA)}^{-1/2}(\htheta_{\ruw(\cA)}^{\radp}-\btheta_{\rt(\cA)})\cvd
  \Nor(\0,\I)$. 
\end{corollary}

\begin{remark}\label{rm:rho2}
  Similarly to Remark~\ref{rm:rho1}, Theorem~\ref{thm:asym-las} requires $c<\infty$. When $c=\infty$, the asymptotic normality becomes 
  $
  \sqrt{\rho N}\bm{e}_{N}\tp\check{\bLambda}_{\ruw(\cA)}^{-1/2}
  (\htheta_{\ruw(\cA)}^{\radp}-\btheta_{\rt(\cA)})\cvd\Nor(0,1)
  $,
  with
  $\check{\bLambda}_{\ruw(\cA)}=\Exp\{\varphi(\x)\dg_{(\cA)}^{\otimes2}(\x;\btheta_{\rt})\}$. 
\end{remark}

The above theorem shows that the penalized MSCL estimator has the same
asymptotic variance as the MSCL estimator under the unknown true model.
This indicates that the MSCL estimator is more efficient than the penalized IPW estimator
\cite{wang2021nonuniform}.  We prove this and present the result in the
following theorem.

\begin{theorem}
\label{thm:effi-comp}
  For fixed $s_{N}$, let the asymptotic variances $\V_{\rw(\cA)}$ for $\htheta_{\rw(\cA)}^{\radp}$
in \eqref{eq:w-las} and $\V_{\ruw(\cA)}$ for $\htheta_{\ruw(\cA)}^{\radp}$ in
(\ref{eq:lik-las}) be finite and positive-definite,
i.e., $0<\V_{\rw(\cA)},\V_{\ruw(\cA)}<\infty$.
Then $\V_{\ruw(\cA)}\le\V_{\rw(\cA)}$ in the Loewner ordering, where the equality
holds when $c=0$.
\end{theorem}

Theorem~\ref{thm:effi-comp} shows that $\htheta_{\ruw}^{\radp}$ is statistically
more efficient than $\htheta_{\rw}^{\radp}$.
They can be equally efficient when $c=0$, which means that the number of zeros
in the subsample is sufficiently large. In this scenario, a subsample estimator
is as efficient as the full data estimator, and a better subsampling design is
not necessary. When the subsampling probability matters, the MSCL estimator is
recommended for better estimation efficiency.

The MSCL estimator $\htheta_{\ruw(\cA)}^{\radp}$ is not only more efficient than
the IPW estimator $\htheta_{\rw(\cA)}^{\radp}$, but also optimal among a class
of asymptotically unbiased estimators based on the selected subsample.
We prove that
$\htheta_{\ruw(\cA)}^{\radp}$ achieves the lower bound of asymptotic variances
for a class of subsample estimators in the following theorem.

\begin{theorem}
\label{thm:cr}
Denote $\X=(\x_1,...,\x_N)$ as the full-data design matrix and $\Ds$ as a
subsample. For fixed $s_{N}$,
consider a class of subsample estimators under the true model with the following
asymptotic representation:
\begin{equation}\label{eq:crtheta}
\htheta_{U(\cA)}=\U_{(\cA)}(\btheta_{\rt};\Ds)
+o_P\left(\frac{1}{\sqrt{N_1}}\right),
\end{equation}
where $\U_{(\cA)}(\btheta_{\rt};\Ds)$ satisfies that
$\Exp\{\U_{(\cA)}(\btheta_{\rt};\Ds)|\X\}=\btheta_{\rt(\cA)}$,
$N_1\Var\{\U_{(\cA)}(\btheta_{\rt};\Ds)\} \cvp\V_{U(\cA)}$, and
$\Exp\left\{\partial\U_{(\cA)}(\btheta_{\rt};\Ds)/\partial\btheta_{(\cA)}\tp\Big|\X\right\}=\0$. 
Then the asymptotic variance $\V_{U(\cA)}$ satisfies that
$\V_{U(\cA)}\geq\V_{\ruw(\cA)}$.
\end{theorem}

Theorem~\ref{thm:cr} establishes the optimality of $\htheta_{\ruw}^{\radp}$ among
 a class of asymptotically unbiased subsample estimators, which includes $\htheta_{\rw}^{\radp}$ as a special case.
 The variance $\V_{\ruw(\cA)}$ can be interpreted as the Cramér–Rao bound
for the class of asymptotically unbiased subsample estimators.
The condition on the partial derivative of $\U_{(\cA)}(\btheta_{\rt};\Ds)$ is
assuming that the derivative of the small term in~\eqref{eq:crtheta} is small as
well, which is satisfied for all unbiased estimators.
\black

\subsection{Practical algorithm and computational complexity}\label{sec:comp}

Since the MSCL estimator $\htheta_{\ruw}^{\radp}$ outperforms the IPW estimator
according to Theorems~\ref{thm:effi-comp} and \ref{thm:cr}, we recommend to use
the MSCL estimator $\htheta_{\ruw}^{\radp}$ and give a practical two-step
algorithm based on it. The optimal sampling functions contain
unknown values and the adaptive lasso
penalty also requires a consistent pilot estimator to build weights,
so it is natural to combine optimal sampling and the adaptive
lasso into one unified framework. We recommend to use the lasso 
for pilot estimation in most cases. One reason is that it does estimation and variable
selection simultaneously, 
 and excluding some inactive variables improves the estimation accuracy
of optimal probabilities. This also reduces the computational burden for
subsequent steps. Another reason is that the lasso estimator tends to include
more variables in practice and therefore has a low risk of excluding important
variables in the pilot step.
We present an outline of the practical implementation in
Algorithm~\ref{alg:adplas}. More details are given in
Section~\ref{sec:dtlalg} of the appendix.
\black

\begin{algorithm}[h]%
  \caption{Two-step subsampling adaptive lasso algorithm}
  \label{alg:adplas}
  \begin{algorithmic}[1]
    \STATE First stage screening:
    \begin{itemize}
    \item Take a pilot sample of
      expected sample size $N_{\rp}$ using
      $\{\pi(y_i)=\rho_0+y_i(\rho_1-\rho_0)\}_{i=1}^N$ and obtain a lasso
      penalized MSCL pilot
      estimator, an estimated active set $\hat{\cA}_{\rp}$.
  \item Calculate approximate optimal sampling probabilities
    $\{\hat{\pi}(\x_i,y_i)=y_i+(1-y_i)\rho\hat\varphi(\x_i)\}_{i=1}^N$
    based on \eqref{eq:optA}, \eqref{eq:optL}, or \eqref{eq:optpr}.
    \end{itemize}
    \STATE Second stage screening:
 Use Algorithm~\ref{alg:poi}  with the estimated optimal sampling
      probabilities to obtain a subsample of expected sample size $N_{\rs}$ and compute the adaptive
      lasso penalized MSCL estimator based on $\hat{\cA}_{\rp}$.
  \end{algorithmic}
\end{algorithm}

Our method requires a pilot step for sampling probabilities and
adaptive weights, along with a penalized MSCL optimization. As shown in~\eqref{eq:lik-las},
the MSCL objective function adds an off-set term to the regular likelihood
function. Therefore, it can be implemented using many existing packages such as
\texttt{Lasso.jl} \citep{juliastats2022lasso} in \texttt{Julia}. The overall
computational complexity is
significantly reduced compared with full data estimators despite two-steps
involve. We next analyze the
total computational complexity of the
practical algorithm. To facilitate the presentation, we consider a special case
when $g(\x;\btheta)=\alpha+f(\x\tp\bbeta)$, and assume that the number of
variables selected at the first-stage is $q$. The algorithm
proposed in~\cite{yuan2012improved} for lasso and adaptive lasso
requires inner iterations for optimal directions and outer iterations for
updating estimators. The computational complexity of full data lasso and
full data adaptive lasso is $O(\zeta_{\oi} Np)$ ($p$ is short for $p_{N}$ for
clearer presentation), where
$\zeta_{\oi}=\zeta_{\out}\zeta_{\inn}$, and $\zeta_{\inn}$, $\zeta_{\out}$ are
the numbers of inner and outer iterations, respectively. Using the coordinate
descent algorithm, the computational complexity
of the two-step algorithm is
$O\{\zeta_{\rp,\oi}^{\rlas}N_{\rp}p+Nq^2+\zeta_{\oi}N(e^{\alpha_{\rt}}+\rho)q\}$
using~\eqref{eq:optA} 
or~\eqref{eq:optpr}, and it is $O\{\zeta_{\rp,\oi}^{\rlas}N_{\rp}p+Nq+\zeta_{\oi}
N(e^{\alpha_{\rt}}+\rho)q\}$ using~\eqref{eq:optL}. 
Here, $\zeta_{\rp,\oi}^{\rlas}$ is the iteration number for pilot estimator.
For optimal probabilities in \eqref{eq:optA} or \eqref{eq:optpr}, when
$\zeta_{\oi}>q/(e^{\alpha_{\rt}}+\rho)$, the dominating term of the complexity is $\zeta_{\oi}
N(e^{\alpha_{\rt}}+\rho)q$, since the first step subsample size is usually
small compared with the second step subsample size. 
The condition $\zeta_{\oi}>q/(e^{\alpha_{\rt}}+\rho)$ is satisfied in
practice considering $\zeta_{\out}$ is usually large for coordinate descent. 
The dominating term for the time complexity of optimal probabilities in~\eqref{eq:optL} is also $\zeta_{\oi}
N(e^{\alpha_{\rt}}+\rho)q$. Compared with full data estimators, both
the sample size and the dimension are reduced. If we set the subsample size
to be the same order of $N_1$,
the time complexity of Algorithm~\ref{alg:adplas} is of order
$O(\zeta_{\oi}Ne^{\alpha_{\rt}}q)$, which is significantly faster than that of
the full data estimator. A more detailed illustration of computing the
complexities is presented in Section~\ref{sec:complexity}.

\section{Numerical experiments}\label{sec:numeric}
In this section, we use numerical experiments on both simulated and real data to
investigate the performances of proposed optimal subsampling and variable
selection procedures.

\subsection{Simulation}\label{sec:simu}

We consider a logistic regression with
$g(\x;\btheta)=\alpha+\x\tp\bbeta$ and the following three true parameters
$\bbeta_{\rt}$ of dimension 50.
We set different $\alpha_{\rt}$ so that the proportion of ones is $0.005$:  
\begin{enumerate}[(1)]
\item{\textbf{Case A:}}
  $\bbeta_{\rt}=(0.75,0.75, \0_7\tp,0.75,0,0.75,0.75,
  \0_{37}\tp)\tp $ and $\alpha_{\rt}=-5.8$.

\item{\textbf{Case B:}}
  $\bbeta_{\rt}=(3,-2,\0_7\tp,0.85,0,-0.75,
  \0_{38}\tp)\tp$ and $\alpha_{\rt}=-6.2$.
  
\item{\textbf{Case C:}}
  $\bbeta_{\rt}=(3,2,\0_7\tp,0.85,
  \0_{40}\tp)\tp$ and $\alpha_{\rt}=-7.5$.
\end{enumerate}
Here, $\0_d$ denotes the zero vector of dimension $d$.
 We use $p_{\cA}$ and $p_{\cAc}$ to denote the number of active and inactive
 variables, respectively, and assume that $\x$ is 
a normal random vector. 
   The active components $x_{(\cA,j)}$, $1\le j \le p_{\cA}$ of
 $\x$ have variances 0.25 and the inactive components $x_{(\cAc,j)}$, $1\le j
 \le p_{\cAc}$ of $\x$ have variances $100/p_{\cAc}^3, 100/(p_{\cAc}-1)^3, ...,
100/3^3,100/2^3,100/1^3$. The correlation between the $i$-th and
$j$-th elements of $\x$ is $0.5^{|i-j|}, 1\leq i,j\leq p$.
We repeat our experiments $S=500$ times generating $N=500000$ data points in
each run and use a pilot sample of size $N_{\rp}=500$ for
obtaining pilot estimates based on the lasso.
We consider
uniform sampling, the full data lasso, and the full data adaptive lasso for
comparison.
We use the 5-fold cross-validation and Bayesian information
criterion (BIC) to determine the tuning parameter $\lambda$ for the lasso and the
adaptive lasso, and choose $\gamma=1$ for the adaptive lasso.

\subsubsection{Estimation and prediction efficiency}

We present the empirical median squared error (eMSE) for parameter estimation in Figure~\ref{fig:msebadA}.
All optimal sampling estimators outperform the uniform sampling. 
As the sampling rate increases, sampling
estimators outperform the full data lasso estimator eventually in all
  of the three cases.
Among the three optimal subsampling methods, $\hbeta^{\radp}_{\mpr}$ performs better
than the other two subsampling methods.

\begin{figure}[htp]%
  \centering 
  \begin{subfigure}{0.27\textwidth}
    \includegraphics[width=\textwidth]{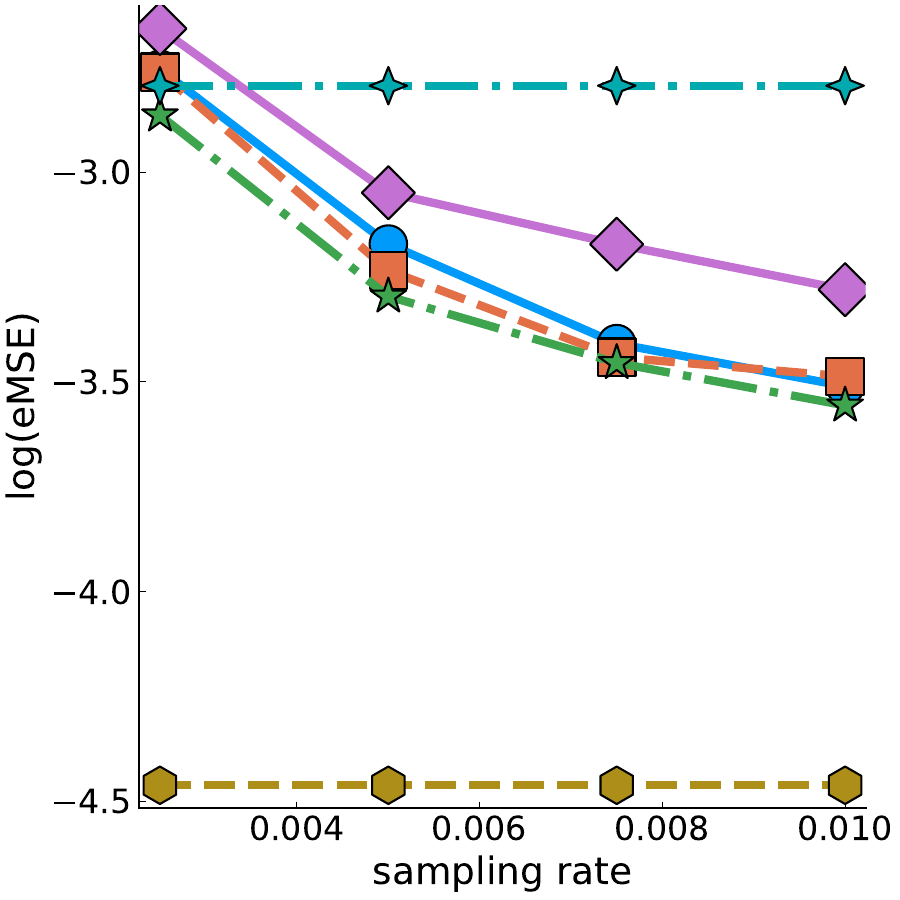}
    \caption{Case A}
  \end{subfigure}
  \begin{subfigure}{0.27\textwidth}
    \includegraphics[width=\textwidth]{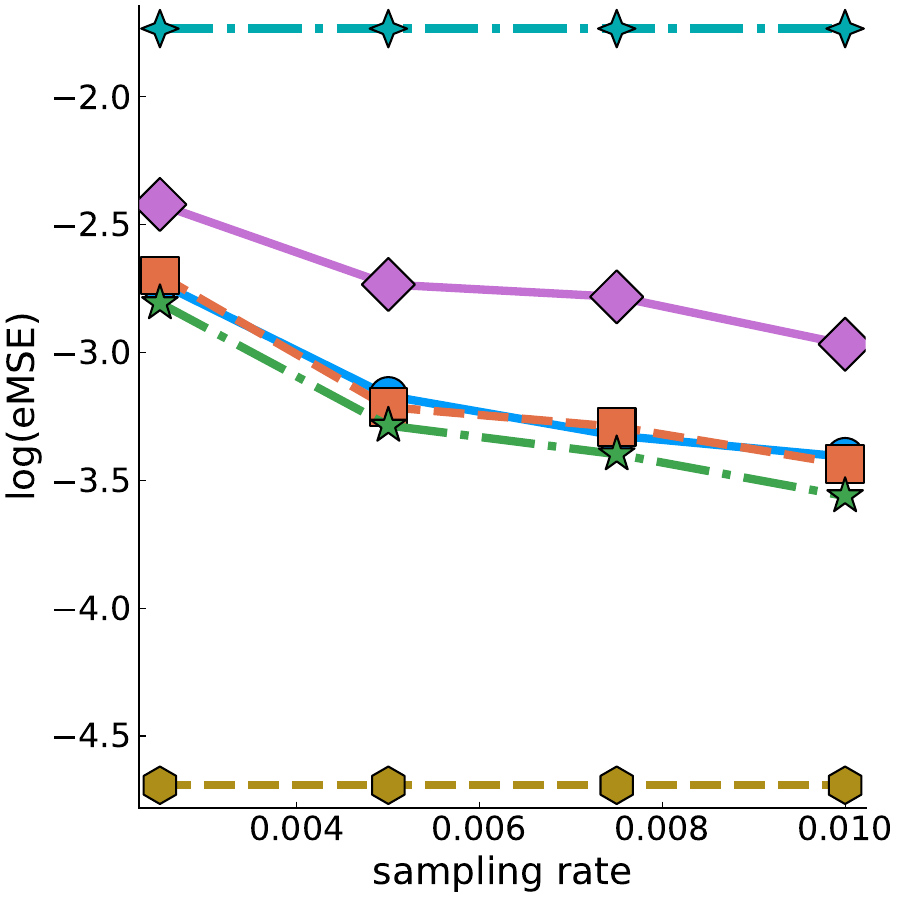}
    \caption{Case B}
  \end{subfigure}
  \begin{subfigure}{0.27\textwidth}
    \includegraphics[width=\textwidth]{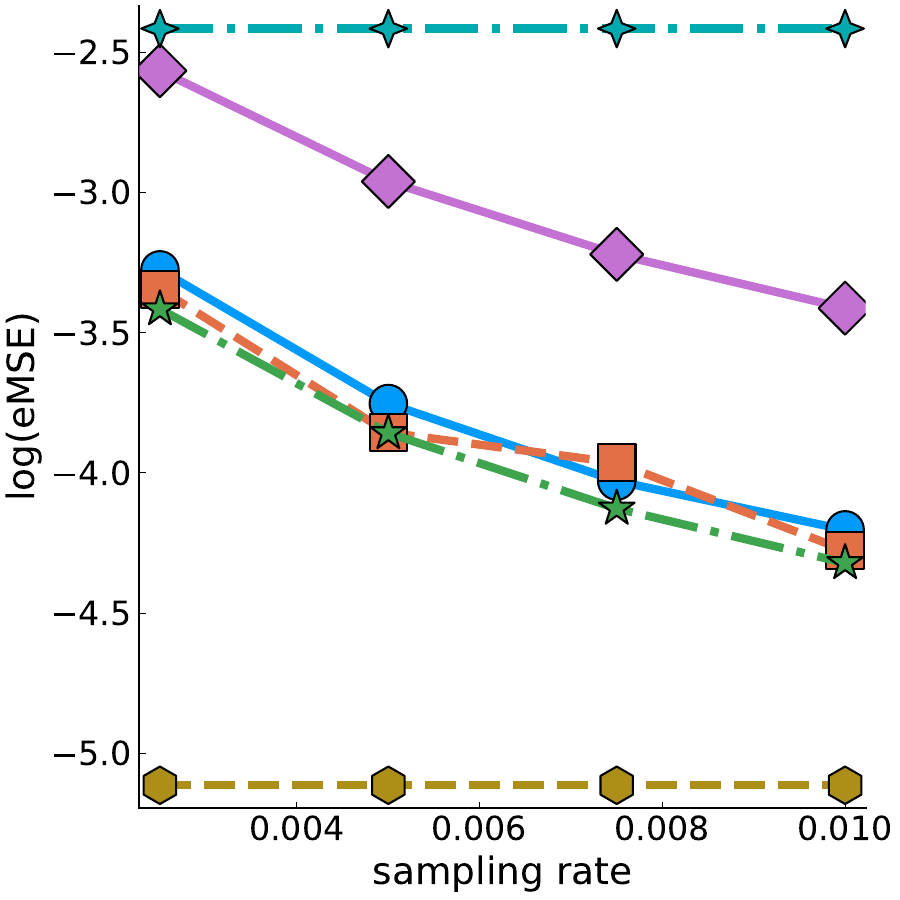}
    \caption{Case C}
  \end{subfigure}  
  \begin{subfigure}{0.1\textwidth}
    \includegraphics[width=\textwidth]{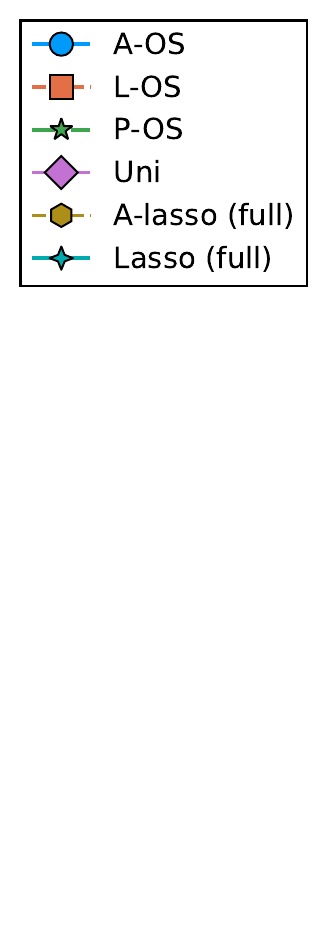}
  \end{subfigure}
  \caption{eMSE for different true parameters with different sampling rates.}
  \label{fig:msebadA}
\end{figure}

\begin{figure}[htp]%
  \centering 
  \begin{subfigure}{0.27\textwidth}
    \includegraphics[width=\textwidth]{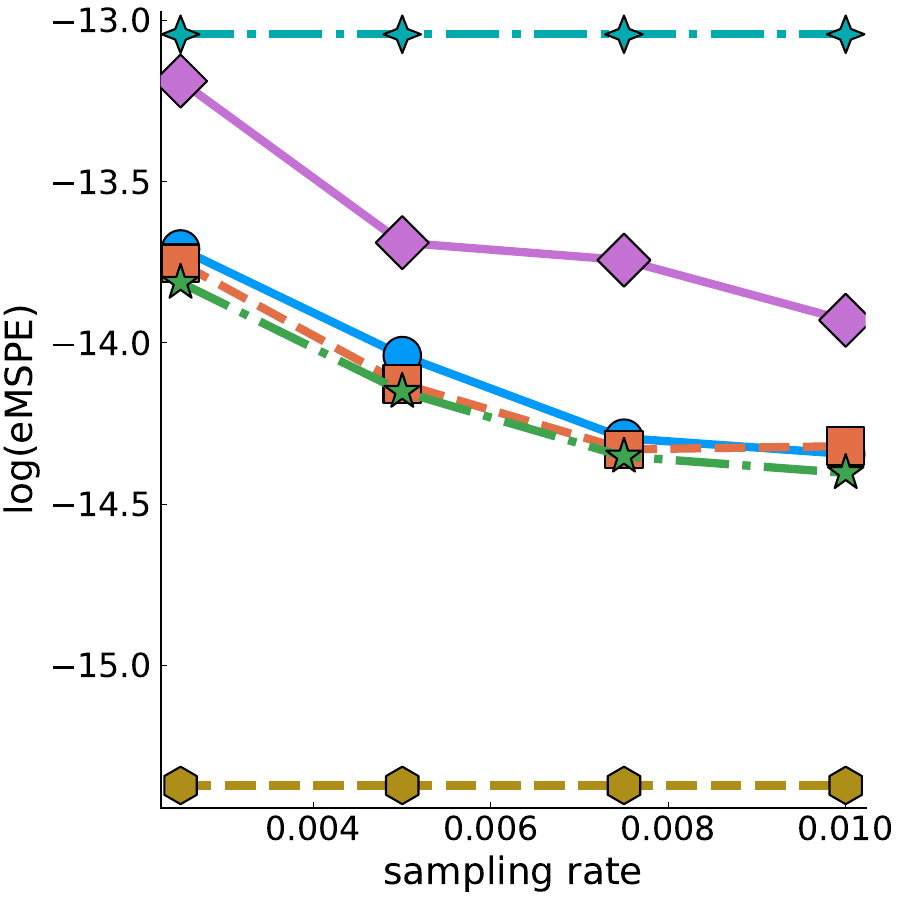}
    \caption{Case A}
  \end{subfigure}
  \begin{subfigure}{0.27\textwidth}
    \includegraphics[width=\textwidth]{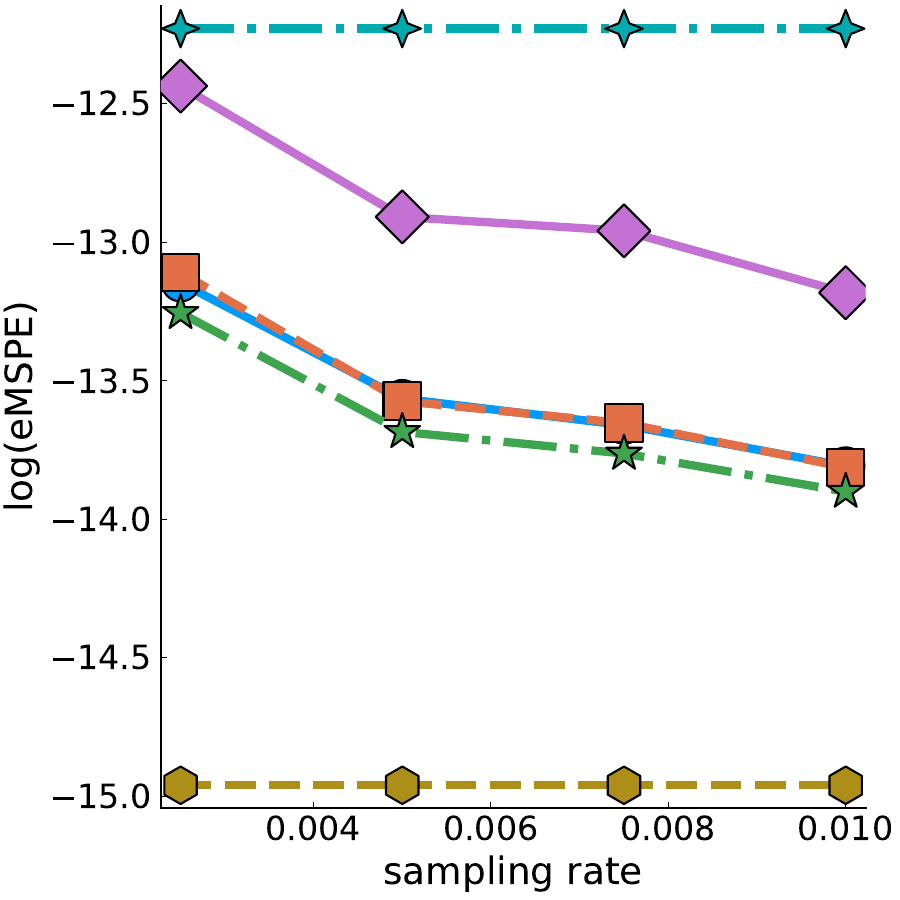}
    \caption{Case B}
  \end{subfigure}  
  \begin{subfigure}{0.27\textwidth}
    \includegraphics[width=\textwidth]{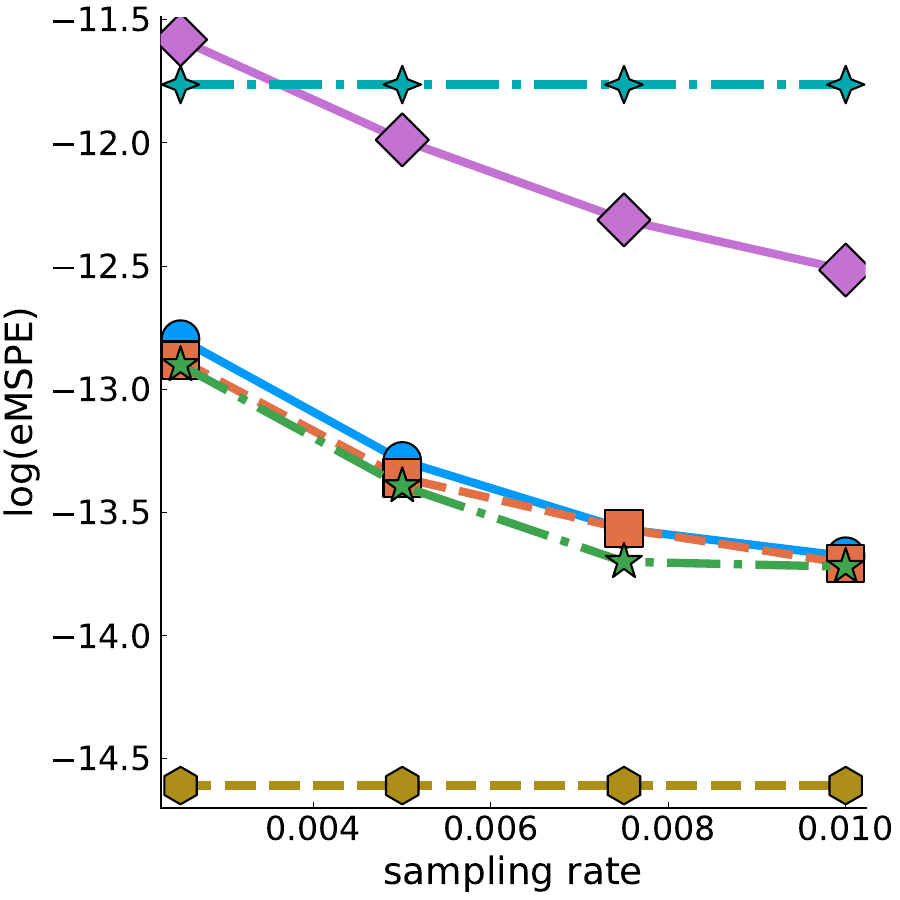}
    \caption{Case C}
  \end{subfigure}
  \begin{subfigure}{0.1\textwidth}
    \includegraphics[width=\textwidth]{leg10.pdf}
  \end{subfigure}
  
  \caption{eMPSE of estimated probability with different sampling rates.}
  \label{fig:mprbadA}
\end{figure}

Figure~\ref{fig:mprbadA} shows the results of the empirical median squared prediction error (eMSPE). Similarly to the
  results of eMSE, optimal sampling estimators
perform better than the uniform sampling, meaning optimal sampling results
  in less information loss. It is possible that
sampling estimators outperform the full data lasso estimator as the
sampling rate increases, despite that the latter uses all of the data. 
In general, $\hbeta^{\radp}_{\mpr}$ performs the best among the three optimal
  subsampling algorithms.

\subsubsection{Variable selection and computational complexity}
In this section, we discuss the results of variable selection in terms of the
first stage screening and the second stage screening. 
Table~\ref{tb:mnumv} presents the mean numbers of selected variables in Case C, where the
numbers in the parentheses are the corresponding standard errors.
Results for Cases A and B are similar so are put in Table~\ref{tb:mnumvS} of the
appendix.

\begin{table}[htp] %
\caption{Mean number of selected variables in Case C.}
\label{tb:mnumv}
\centering
\begin{tabular}{lccccc}
\hline
$\rho$  & first-stage & Uni        & A-OS       & L-OS       & P-OS       \\\hline
 0.0025 & 13.27(0.34) & 2.84(0.02) & 2.97(0.02) & 2.96(0.02) & 2.96(0.02) \\
  0.005 & 12.46(0.32) & 2.94(0.02) & 3.04(0.03) & 3.05(0.03) & 3.06(0.03) \\
 0.0075 & 12.76(0.33) & 2.97(0.01) & 3.04(0.02) & 3.03(0.02) & 3.03(0.02) \\
   0.01 & 12.81(0.34) & 2.96(0.01) & 3.03(0.02) & 3.02(0.01) & 3.02(0.01) \\
\hline
\end{tabular}
\end{table}

While the first stage screening significantly reduces the dimension in Table~\ref{tb:mnumv},
it indeed includes inactive variables as expected. In the
  second stage screening, the mean numbers of selected variables are close to
  the true numbers of active variables for all subsampling
  methods.
However, the mean number of selected variables from uniform sampling is smaller
than the true number of active variables especially when the sampling rate is low.
This indicates that the second-stage screening of uniform sampling may exclude
active variables. We present the rates of missing active variables in
Table~\ref{tb:covermodel} for Case C. It shows that uniform sampling has higher
rates of excluding active variables than optimal subsampling 
procedures, so optimal sampling may be preferable
in practice. Results for Cases A and B are similar and are put in Section~\ref{sec:vari-select-resultsS}. We also investigate the rates
of selecting the true model in that section.

\begin{table}[htp]
\caption{Rates of excluding active variables (false negative rate) in Case C.}
\label{tb:covermodel}
\centering
\begin{tabular}{lcccc}
  \hline
$\rho$ & Uni          & A-OS         & L-OS         & P-OS         \\\hline
0.0025 & 0.168(0.017) & 0.086(0.013) & 0.088(0.013) & 0.084(0.013) \\
 0.005 & 0.100(0.013) & 0.068(0.011) & 0.066(0.011) & 0.066(0.011) \\
0.0075 & 0.066(0.011) & 0.046(0.009) & 0.048(0.010) & 0.046(0.009) \\
  0.01 & 0.068(0.011) & 0.052(0.010) & 0.054(0.010) & 0.054(0.010) \\
\hline
\end{tabular}
\end{table}

\subsubsection{Computational time}\label{sec:computation}

We present the mean computational times of different algorithms in
Table~\ref{tb:time}. Our codes are written in the \texttt{Julia} programming
language \citep{julia} and
implemented on a Linux workstation. The lasso paths are solved with 
\texttt{Lasso.jl} \citep{juliastats2022lasso}.
As shown in Table~\ref{tb:time}, subsampling algorithms
significantly reduce the computational times compared with full data
estimators. Although optimal sampling requires calculating sampling
probabilities, it uses only about 0.77\% of the computational
time that the full data adaptive lasso requires. %
Optimal sampling
algorithms reduce both the sample size and the data dimension as we discussed in
Section~\ref{sec:comp}. Therefore, the
computational cost of the coordinate descent algorithm, which often requires
  a large number of iterations, is significantly reduced. 
\begin{table}[H]
  \centering
  \caption{Mean computational time (seconds).}
  \label{tb:time}
  \begin{tabular}{ccccccc}\hline
    Case & Uni & A-OS & L-OS & P-OS & A-lasso (full) & Lasso (full)\\\hline
    A & 0.29 & 1.09 & 0.91 & 1.06 & 129.62 & 112.97 \\
    B & 0.31 & 1.23 & 1.20 & 1.27 & 129.89 & 122.40 \\
    C & 0.31 & 1.02 & 0.93 & 1.00 & 130.33 & 121.29 \\\hline
  \end{tabular}
\end{table}

\subsubsection{Comparison of different pilot estimators}\label{sec:pilot}

In this section, we investigate the effects of different pilot estimators. As
mentioned earlier, we recommend the lasso for first-stage screening, because it
performs estimation along with variable selection and tends to over-select
variables. In general, any approach that combines estimation and variable
selection is suitable for our first-stage screening. Examples include sure
independence screening (SIS) and the lasso followed by an MLE step. We compare
the performance of different pilot estimators using Case C as an example.

Besides lasso, we consider SIS for the first-stage screening. This approach
selects a pre-determined number of variables based on marginal correlations.  As
shown in Table~\ref{tb:mnumv}, lasso selects approximately 13 active variables
on average during the first-stage screening. Therefore, we set the
pre-determined number of selected variables for SIS to around 13. Specifically,
we examine SIS with 6, 13, and 20 variables. Additionally, we consider SIS with
3 variables because the true number of active variables in Case C is 3. After
applying SIS, we perform an MLE step to obtain pilot estimates for selected
covariates.

We also consider the lasso followed by an MLE (Las-mle) step for the first
 stage-screening, where we apply lasso for variable selection and then apply
the MLE to obtain parameter estimates with selected variables. Since the
relative performances across different subsampling probabilities are similar, we
only present results for the P-OS subsampling probabilities. The results based
on 500 iterations are shown in Table~\ref{tb:sis1}.

\begin{table}[H]
\centering
\caption{Table of eMSEs with different pilot estimators for Case C.}
\label{tb:sis1}
\begin{tabular}{ccccccc}
  \hline
    $\rho$ & Las & SIS-3 & SIS-6 & SIS-13 & SIS-20 & Las-mle \\\hline
    0.0025 & 0.025 & 0.729 & 0.023 & 0.034 & 0.047 & 0.030 \\
    0.0050 & 0.017 & 0.728 & 0.014 & 0.023 & 0.031 & 0.020 \\
    0.0075 & 0.014 & 0.726 & 0.012 & 0.016 & 0.025 & 0.016 \\\hline
\end{tabular}
\end{table}

As shown in Table~\ref{tb:sis1}, $\hbeta_{\mathrm{Las}}^{\rp}$ performs better
than $\hbeta_{\mathrm{SIS13}}^{\rp}$, indicating that
$\hbeta_{\mathrm{Las}}^{\rp}$ is a better choice when selecting similar numbers
of variables. When selecting six variables in the pilot stage,
$\hbeta_{\mathrm{SIS6}}^{\rp}$ performs slightly better than
$\hbeta_{\mathrm{Las}}^{\rp}$. This is because $\hbeta_{\mathrm{SIS6}}^{\rp}$
may result in more accurate variable selection than
$\hbeta_{\mathrm{Las}}^{\rp}$, given that the true number of active variables is
3. This suggests that SIS can outperform lasso as a pilot estimator when the
number of variables is set appropriately.  However, setting an inappropriate
number of selected variables for SIS may lead to significantly worse
performance.  For example, $\hbeta_{\mathrm{SIS3}}^{\rp}$ results in a large
eMSE because it excludes active variables with high probability. Conversely, if
we attempt to avoid excluding important variables by selecting a larger number
of variables (e.g., $\hbeta_{\mathrm{SIS20}}^{\rp}$), SIS may not perform as
well as lasso.

Between $\hbeta_{\mathrm{Las-mle}}^{\rp}$ and $\hbeta_{\mathrm{Las}}^{\rp}$, the
former performs worse than the latter, as shown in Table~\ref{tb:sis1}.  This
occurs in our simulations because lasso selects some inactive variables.  For
these inactive variables, the MLE estimates have larger magnitudes than the
corresponding lasso estimates, since lasso shrinks all estimates toward zero.
As a result, the adaptive weights $|\hat{\beta}_{\rp(j)}|^{-\gamma}$ for the
inactive variables are smaller when using the Las-MLE pilot compared to the
lasso pilot.

In summary, the lasso is a robust choice for the pilot estimator, as it performs
well in various scenarios, and we recommend using it to construct the pilot
estimator in general. Although SIS can outperform lasso when the number of
selected variables is set appropriately, this information is often not
available in practice.

\subsection{Real data}\label{sec:real}

\subsubsection{Public benchmark data sets}
We evaluate performances of the proposed estimators with two
open source data sets.
\begin{enumerate}[(i)]
\item{\textbf{Covtype data set:}} It is available at
  \url{https://archive.ics.uci.edu/ml/datasets/covertype}, with $N=581012$
  observations and 54 covariates -- 10 being quantitative and 44 being
  qualitative with dummy coding. We drop the 14th and 54th columns to avoid
  exact colinearity of the dummy variables. Our goal is to classify whether the
  forest cover type is Cottonwood/Willow (labeled as 1) or not (labeled as 0).
  The proportion of Cottonwood/Willow is 0.473\%, which is highly imbalanced.
\item{\textbf{Font data set:}} It is available at
  \url{https://archive.ics.uci.edu/ml/datasets/Character+Font+Images}, with
  0.50\% of the $N=832670$ responses being the GADUGI font. The first 10
  covariates are about the value, size, and style of the characters and there
  are additional 400 pixel values of the $20\times20$ images.  We remove the
  4th, 9th, and 10th covariates because they are constants.
\end{enumerate}

For both data sets, we apply
Algorithm~\ref{alg:adplas} on the logarithmic-transformed data. We use
pilot samples of size $N_{\rp}=1000$ for the covtype data and $N_{\rp}=1500$ for the font
data due to its higher dimension. Since we do not know
the true parameter for real data, we use area under the
  curve (AUC) to
measure the performances of subsampling algorithms. %
We repeat the experiment for
$S=500$ and compute the empirical median AUC using the full data. The results are
  summarized in Figure~\ref{fig:realdata}.  
\begin{figure}[t]
  \centering 
  \begin{subfigure}{0.315\textwidth}
    \includegraphics[width=\textwidth]{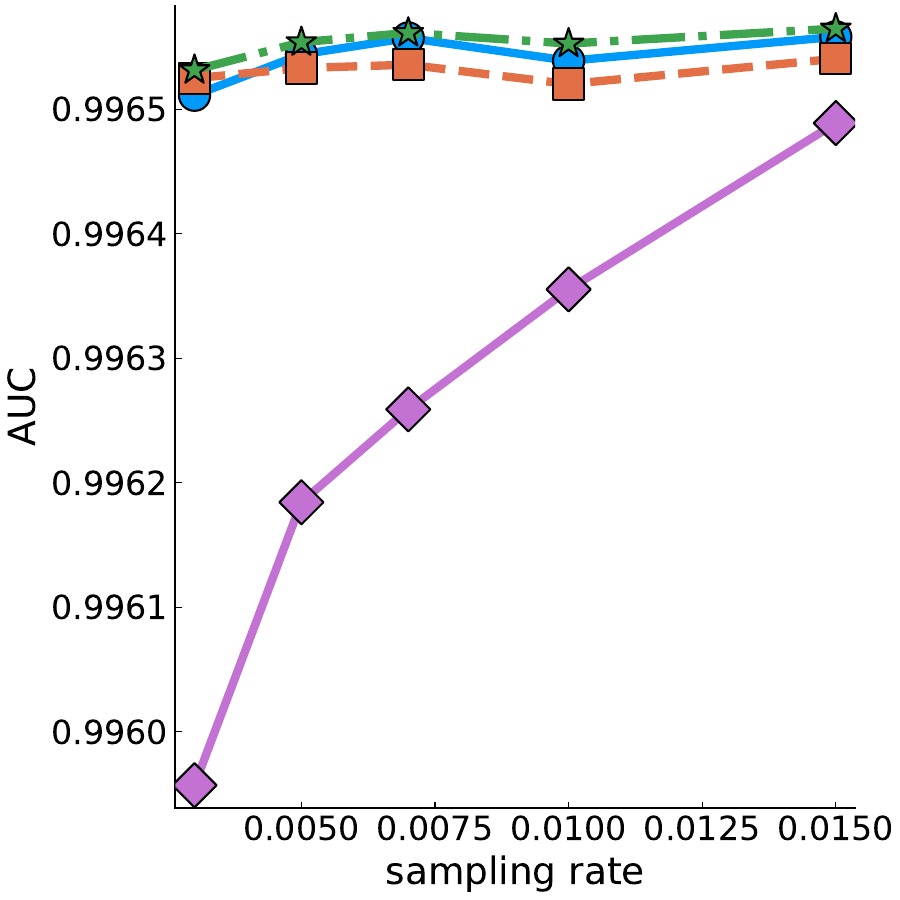}
    \caption{AUCs of the covtype data}
  \end{subfigure}
  \begin{subfigure}{0.315\textwidth}
    \includegraphics[width=\textwidth]{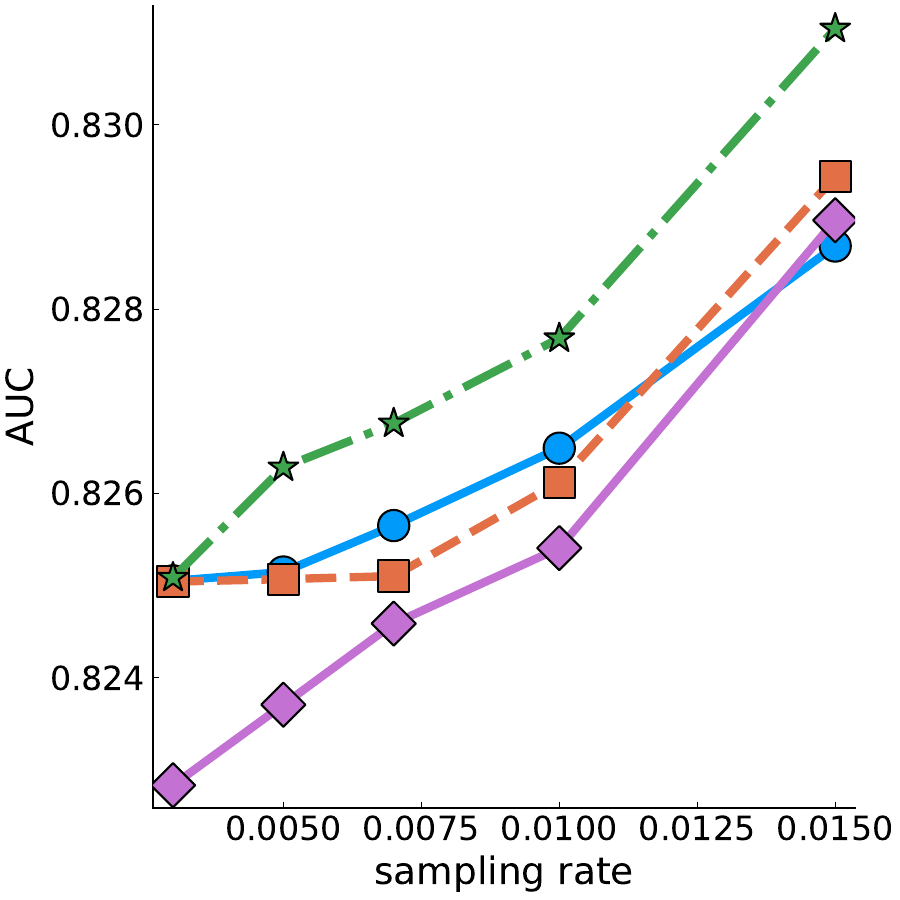}
    \caption{AUCs of the font data}
  \end{subfigure}
  \begin{subfigure}{0.1\textwidth}
    \includegraphics[width=\textwidth]{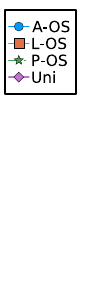}
  \end{subfigure}
  \caption{Empirical median AUCs for two real data sets.}
  \label{fig:realdata}
\end{figure}
As shown in Figure~\ref{fig:realdata}, nonuniform sampling outperforms uniform
sampling in general. There is one case for font data set that
$\hbeta_{\mmse}^{\radp}$ is worse than the uniform sampling when the sampling
rate is high. For the covtype data set, among the three estimators based on
optimal sampling, $\hbeta^{\radp}_{\mpr}$ performs the best and
$\hbeta^{\radp}_{\mvc}$ is worst. For the font data set,
$\hbeta^{\radp}_{\mmse}$ and $\hbeta^{\radp}_{\mvc}$ are similar, and
$\hbeta^{\radp}_{\mpr}$ based on the scale-invariant optimal sampling function
is significantly better.

\subsubsection{American Academy of Ophthalmology IRIS® Registry data}
Launched
on March 24, 2014, 
the IRIS® (Intelligent Research in Sight) Registry
\footnote{\url{https://www.aao.org/iris-registry}}
 is the first electronic health record-based comprehensive
eye disease and condition registry in the United States.
As of July. 1, 2023, the IRIS Registry has data on
over 70/423.4 million patients/visits, and it is growing rapidly daily.
We applied our proposed methods on a dataset from IRIS, in which all patients in
the IRIS Registry from 2013 to 2018 between 18-90 years of age as of last
procedure or documented diagnosis were included \citep{ramesh2023thyroid}. We
are interested in the thyroid eye disease in
this analysis. The presence of thyroid eye disease was defined as
having at least two visits to an ophthalmology practice coded with International
Classification of Diseases, Ninth/Tenth Revision (ICD-9: 242.00, ICD-10:
E05.00). Individuals not meeting these criteria were considered to be
non-cases. We considered the following variables in this analysis: age
(numerical), gender (categorical), race (categorical), ethnicity (categorical),
region (categorical), and smoking status (categorical). The numerical age was
standardized and all categorical variables were coded with dummy variables. The
resulting full data contains 47,872,555 observations with sixteen potential
predictors.

The prevalence of thyroid eye disease is 0.0861\% in this cohort. We obtain the
pilot estimator from a pilot subsample of $N_{\rp}=1000$ using the SIS
algorithm. We set the number of variables for the pilot estimator as 90\% of the
total number of covariates. We repeat the experiments $S=500$. The results for
classification are summarized in Figure~\ref{fig:eyedata}.
\begin{figure}[h]
  \centering 
  \begin{subfigure}{0.385\textwidth}
    \includegraphics[width=\textwidth]{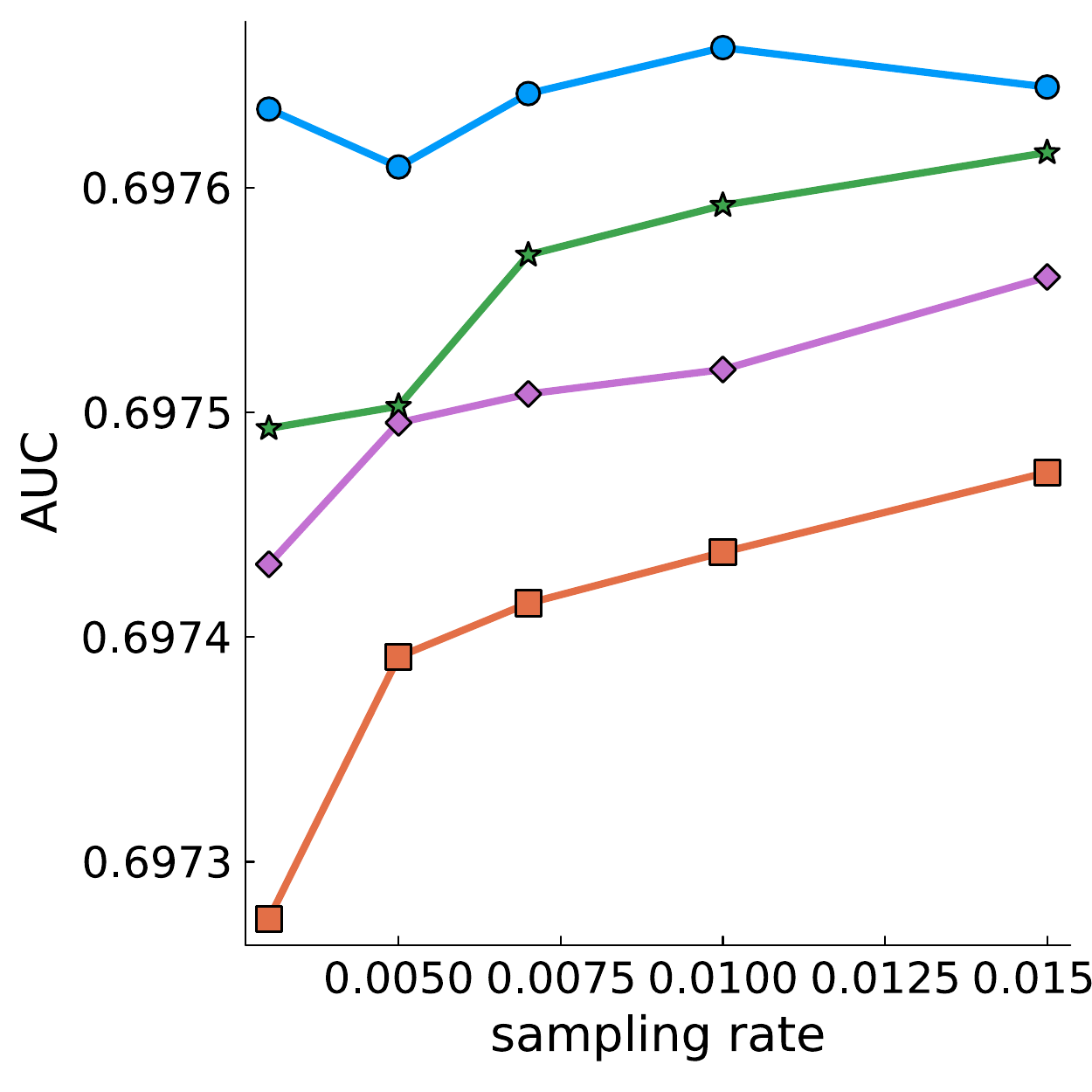}
  \end{subfigure}
  \begin{subfigure}{0.11\textwidth}
    \includegraphics[width=\textwidth]{legr1.pdf}
  \end{subfigure}
  \caption{Empirical median AUC for the IRIS® data set.}
  \label{fig:eyedata}
\end{figure}

As shown in Figure~\ref{fig:eyedata}, both $\hbeta^{\radp}_{\mmse}$ and
$\hbeta^{\radp}_{\mpr}$ outperform uniform subsampling. For this data set,
$\hbeta^{\radp}_{\mmse}$ performs best, while $\hbeta^{\radp}_{\mvc}$ performs
worst — even worse than uniform subsampling. This suggests that scale-dependent
sampling probabilities (used for $\hbeta^{\radp}_{\mmse}$ and
$\hbeta^{\radp}_{\mvc}$) are not robust: they may increase or decrease
estimation efficiency relative to simple uniform subsampling, and unfortunately
this cannot be known before taking subsamples. Scale-invariant subsampling
yields the estimator $\hbeta^{\radp}_{\mpr}$, which behaves more robustly: its
performance lies between that of $\hbeta^{\radp}_{\mmse}$ and
$\hbeta^{\radp}_{\mvc}$. As noted earlier, although $\hbeta^{\radp}_{\mpr}$ may
not be optimal in every case, it is a more robust choice and is never the worst.

We also provide variable selection results for the thyroid eye disease data. In
each of 500 repetitions, we select variables using subsampling estimators,
allowing us to compute the selection frequency for each
variable. Figure~\ref{fig:eyeselect} summarizes these results. The variable
selection patterns are similar across different subsampling probabilities, so we
present only the results for $\hbeta^{\radp}_{\mpr}$ with $\rho=0.01$. Notably,
age—the only numerical predictor—is not the most frequently selected variable,
which is consistent with medical evidence that thyroid eye disease can occur with
wide age ranges. Gender is always selected, aligning with the higher prevalence of the
disease in females. Although the ``Unknown'' category of Race is not of the primary
interest, it is consistently selected; this supports the practice of
including an ``Unknown'' race group in analyses to maintain comparability of samples and/or models and
improve the results. It is also important to collect detailed information such as race so that the results
can be explained accordingly. The
selection of region and smoking-status indicators suggests that these factors may be
associated with differing risks of thyroid eye disease.

\begin{figure}[H]
    \centering
    \includegraphics[width=\textwidth]{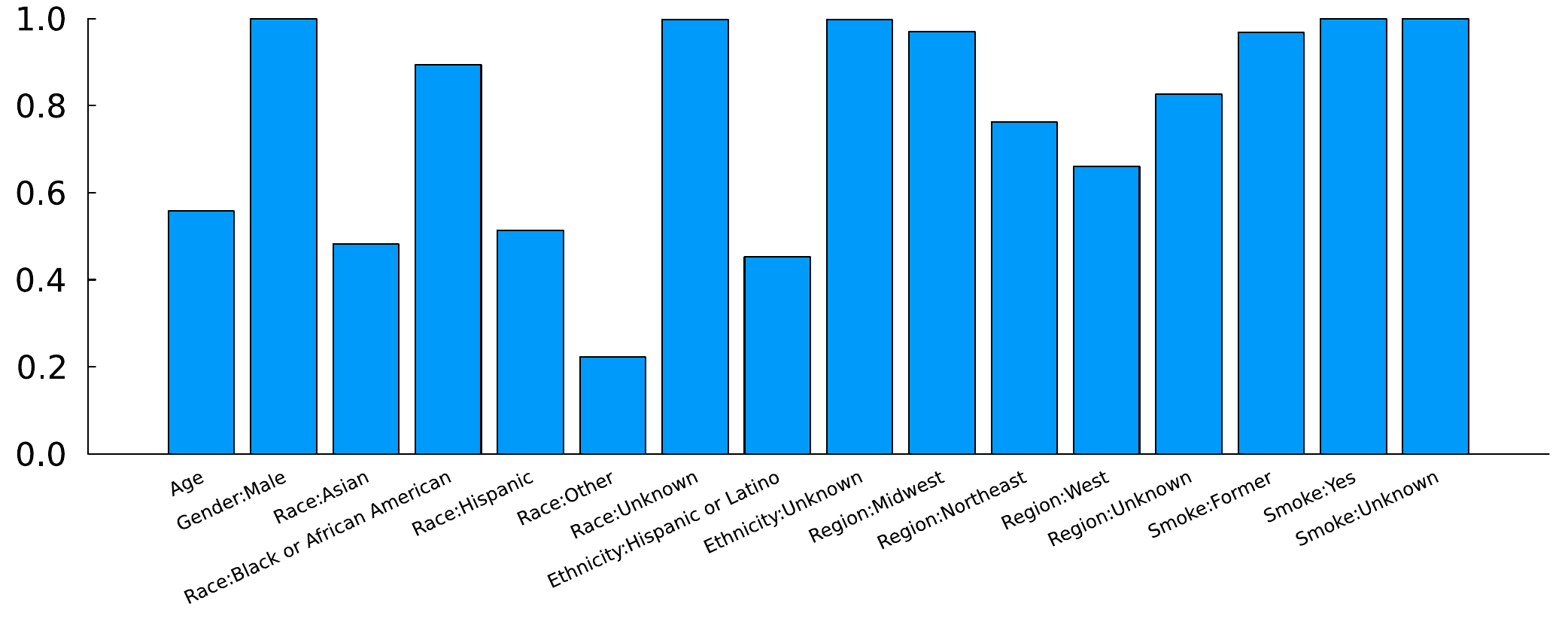}
    \caption{Frequencies of selected variables.}
    \label{fig:eyeselect}
\end{figure}

\section{Conclusion and limitations}\label{sec:conclusion}

In this paper, we investigated the problem of scale-invariant optimal subsampling
in the context of variable selection for rare-events
data. %
We derived optimal probabilities based on the A- and L- optimality criteria, and
discussed their limitations. Furthermore, we proposed scale-invariant optimal
probabilities based on prediction errors to overcome the
limitations of scale-dependence.
Both analytical and numerical results show the desirable properties of the
proposed methods.
As we have shown, the proposed sampling procedure results
in small and balanced subsamples without much information loss, which facilitates
the computation in practice.
The practical implementation of the proposed method requires selecting a pilot
estimator, a sampling rate $\rho$, and an importance function $\varphi(\x)$. For
the pilot estimator, we recommend using a lasso estimator in practice. For the
sampling rate $\rho$, we recommend setting it to be larger than but of
the same order of magnitude as $N_{1}/N_{0}$ to achieve a balance between
statistical efficiency and computational speed. Finally, for the importance
function $\varphi(\x)$, since the P-OS optimal function usually performs well
and shows robustness in simulations and real data analyses, we suggest
$\varphi_{\mpr}^{\radp}(\x)$ without pre-knowledge of the true model.
Our investigation has the following limitations.

\begin{itemize}
\item Our proposed criterion optimizes the probabilities by minimizing the
  asymptotic mean squared error in estimating rare-event probabilities. While
  this prioritizes the accuracy of estimation, it puts less emphasis on the
  quality of variable selection. Further research is needed to devise optimal
  probabilities that focus on variable selection performance metrics.
\item Our theoretical analysis is based on asymptotic properties, with optimal
  probabilities defined through the asymptotic normality. Although our results
  may hold for sufficiently sparse models, they may not generalize to cases
  where the model is dense or over-parameterized, because asymptotic normality
  may no longer be applicable. Also, our asymptotic results of prediction
  errors is limited to fixed numbers of active variables. 
  Therefore, an important direction for future
  research is to study the non-asymptotic properties of our estimators, such
  as prediction error bounds, especially when dimensions diverge with
  $N$. Non-asymptotic behaviors are particularly of interest in
  high-dimensional regimes.

\item We assume that the underlying full model is correctly specified and
  possesses a sparse structure. Our analysis does not account for model
  misspecification. Further research is required to address scenarios where the
  model is possibly misspecified or where the number of features vastly exceeds
  the number of observations.
\end{itemize}

\newpage
\appendix

\section{Details of mathematical proofs}\label{sec:proof}

In this section, we provide mathematical proofs of the theoretical results in
the main paper. We begin with some general assumptions used throughout this
paper.

\subsection{General assumptions}\label{sec:asm}

Here, we remind some notations used in the main
paper:
\begin{equation*}
p(\x;\btheta)=\frac{e^{\alpha+f(\x;\bbeta)}}{1+e^{\alpha+f(\x;\bbeta)}},
\phi(\x;\btheta)=p(\x;\btheta)\left\{1-p(\x;\btheta)\right\},
\M=\Exp\left\{e^{f(\x_i;\bbeta_{\rt})}\dg^{\otimes2}(\x_i,\btheta_{\rt})\right\},
\end{equation*}
and 
$
\bLambda_{\ruw}=\Exp\left[
\{e^{f(\x;\bbeta_{\rt})}\dg^{\otimes2}(\x;\btheta_{\rt})\}/\{1+c\varphi^{-1}(\x)e^{f(\x;\bbeta_{\rt})}\}
\right].
$
To ease the presentation in the following sections, we denote
$
\M_{\rw(\cA)}=\Exp\left[\left\{1+ce^{f(\x;\bbeta_{\rt})}/\varphi(\x)\right\}
e^{f(\x;\bbeta_{\rt})}\dg_{(\cA)}^{\otimes2}(\x;\btheta_{\rt})\right],
$
and $a_N=\sqrt{Ne^{\alpha_{\rt}}}$ in the appendix. Here 
$N_{1}^{*}=\sumn y_i=\Exp(N_{1}^{*})\{1+\oo\}$ almost surely, and
\begin{align*}
&N_1=\Exp(N_{1}^{*})=\Exp\left(\sumn y_i\right)
=N\Exp\left\{\frac{e^{\alpha_{\rt}+f(\x;\bbeta_{\rt})}}
{1+e^{\alpha_{\rt}+f(\x;\bbeta_{\rt})}}\right\}\\
&=Ne^{\alpha_{\rt}}\Exp\left\{e^{f(\x;\bbeta_{\rt})}\right\}
=a_N^2\Exp\left\{e^{f(\x;\bbeta_{\rt})}\right\}.
\end{align*}
In this Appendix, for two sequences $v_{N}$ and $w_{N}$, we use notation
$v_{N}\lesssim w_{N}$ to represent $v_{N}\le Cw_{N}$, and use notation
$v_{N}\gtrsim w_{N}$ to
represent $v_{N}\ge Cw_{N}$ where $C$ is a constant that does
not depend on $N$. Also, since $c_{N}=e^{\alpha_{\rt}}/\rho\to
c<\infty$ in Assumption~\ref{asm:a3}. We assume $c_{N}<\infty$ in the rest of
the Appendix without loss of generality.

\subsection{Proof of Theorem~\ref{thm:asym-ipw}}\label{sec:optprb}

\subsubsection{Proof when $p_{N}$ and $s_{N}$ are fixed}\label{sec:ipw-fixp}
We start with proving Theorem~\ref{thm:asym-ipw} when the dimension $p_{N}$ and
the number of active variables $s_{N}$ are fixed, e.g., $p_{N}=p$ and $s_{N}=s$.
We consider the target of IPW estimator: 
\begin{align*}
&\Q_{\rw}(\btheta)=
-\sumn\frac{\delta_i}{\pi(\x_i,y_i)}[y_ig(\x_i;\btheta)-\log\{1+e^{g(\x_i;\btheta)}\}]+
\lambda_N\sumjp\hat{w}_j|\beta_{(j)}|\\
&=-\ell_{\rw}(\btheta)+\lambda_N\sumjp\hat{w}_j|\beta_{(j)}|
=-\ell_{\rw}(\btheta)+\lambda_N\sumjp\frac{1}{|\hat{\beta}_{\rp(j)}|}|\beta_{(j)}|,     
\end{align*}
Then, we have that
$\tu_N=a_N(\htheta_{\rw}-\btheta_{\rt})$ is the minimizer of
$
\gamma_{\rw}^N(\bu)=Q_{\rw}(\btheta_{\rt}+a_N^{-1}\bu)-Q_{\rw}(\btheta_{\rt}).
$
\paragraph{Asymptotic normality:} We prove the asymptotic normality. By Taylor's expansion, 
\begin{align*}
\gamma_{\rw}^N(\bu)&=-\frac{1}{a_N}\bu\tp\dl_{\rw}(\btheta_{\rt})+\frac{1}{2a_N^2}
\sumn\frac{\delta_i}{\pi(\x_i,y_i)}\phi(\x_i;\btheta_{\rt})\{\bu\tp\dg(\x_i;\btheta_{\rt})\}^2-\Delta_{\rw}+R_{\rw}\\
&\quad+\frac{\lambda_N}{a_N}\sumjp \hat{w}_j a_N\left( \left|
\beta_{\rt(j)}+\frac{u_{(j)}}{a_N} \right|-|\beta_{\rt(j)}| \right).
\end{align*}
We first consider the limit
behavior of the IPW target function by prove the asymptotic normality.
In \cite{wang2021nonuniform}, the authors established that under
Assumptions~\ref{asm:a1} to~\ref{asm:a3},   
$a_N^{-1}\dl_{\rw}(\btheta_{\rt})\cvd\M_{\rw}^{1/2}\W$,
$1/a_N^2
\sumn\{\delta_i/\pi(\x_i,y_i)\}\phi(\x_i;\btheta_{\rt})\dg^{\otimes2}(\x_i;\btheta_{\rt})
\cvp\M$,
and
$\Delta_{\rw}=\op,R_{\rw}=\op$.
Thus, 
$-\ell_{\rw}(\btheta_{\rt})\cvd-\bu\tp\M_{\rw}^{1/2}\W+0.5\bu\tp\M\bu$.
Next, we consider the limit behavior of the adaptive lasso penalty. Since we
assume $\hat{\bbeta}_{\rp}$ to be a consistent estimator, we know that when
$j\in\cA$, i.e., $\bbeta_{\rt(j)}\neq0$, 
$\hat{w}_j=|\hat{\beta}_{\rp(j)}|^{-\gamma}\cvp|\beta_{\rt(j)}|^{-\gamma}>0$,
and 
$a_N \left( \left|\beta_{\rt(j)}+u_{(j)}/a_N\right|-|\beta_{\rt(j)}|
\right)\to\sgn(\beta_{\rt(j)})u_{(j)}$.
Therefore, for $j\in\cA$, we have that
$(\lambda_N/a_N)\hat{w}_ja_N \left( \left| \beta_{\rt(j)}+u_{(j)}/a_N
\right|-|\beta_{\rt(j)}| \right)=\op$, %
since $\lambda_N/a_N=\lambda_N/\sqrt{Ne^{\alpha_{\rt}}}\to0$.
On the other hand, when $j\in\cAc$, i.e., $\beta_{\rt(j)}=0$, we have that for
$u_{(j)}\neq 0$, %
\begin{equation*}
\frac{\lambda_N}{a_N}\hat{w}_ja_N \left( \left| \beta_{\rt(j)}+\frac{u_{(j)}}{a_N}
\right|-|\beta_{\rt(j)}| \right)
=\frac{\lambda_N}{a_N}\hat{w}_j|u_{(j)}|=\frac{\lambda_N}{a_N|\hat{\beta}_{\rp(j)}|^{\gamma}}
|u_{(j)}|\cvp\infty,
\end{equation*}
since $\lambda_N/(\sqrt{Ne^{\alpha_{\rt}}}|\hat{\beta}_{\rp(j)}|^{\gamma})\cvp\infty$.
Then, we have that $\gamma_{\rw}^N(\bu)\cvd\gamma_{\rw}(\bu)$, where 
\begin{equation*}
\gamma_{\rw}(\bu)=
\begin{cases}
\frac{1}{2}\bu_{(\cA)}\tp\M_{(\cA)}\bu_{(\cA)}-\bu_{(\cA)}\tp\M_{\rw(\cA)}^{1/2}\W_{(\cA)} &
\text{if }u_{(j)}=0, \forall j\in\cAc  \\
\infty   & \text{otherwise}.  \\
\end{cases}
\end{equation*}
Note that the unique minimizer of $\gamma_{\rw}^N(\bu)$ is
$(\M_{(\cA)}^{-1}\M_{\rw(\cA)}^{1/2}\W_{(\cA)}\tp,\0)\tp$ if we put all the indexes of active
variables in front. Thus, following the results of \cite{geyer1994asymptotics}
and \cite{fu2000asymptotics}, we have the minimizer of $\gamma_{\rw}^N(\bu)$,
i.e., $\tu_N$, satisfies that 
$\tu_{N(\cA)}\cvd\M_{(\cA)}^{-1}\M_{\rw(\cA)}^{1/2}\W_{(\cA)}$, and $\tu_{N(\cAc)}\cvd\0$.
Thus,
$\tu_{N(\cA)}=a_N(\htheta_{\rw(\cA)}-\btheta_{\rt(\cA)})
\cvd\Nor(\0,\M_{(\cA)}^{-1}\M_{\rw(\cA)}\M_{(\cA)}^{-1})$.
Since
$\sqrt{N_1}=a_N\Exp^{1/2}\left\{e^{f(\x;\bbeta_{\rt})}\right\}$,
we have 
$\sqrt{N_1}\V_{\rw(\cA)}^{-1/2}(\htheta_{\rw(\cA)}-\btheta_{\rt(\cA)})
\cvd\Nor(\0,\I)$.

\paragraph{Consistency in variable selection}

We prove the consistency in variable selection in this paragraph. From the
result of asymptotic normality, we know
that $\hat{\beta}_{\rw(j)}\cvp\beta_{\rt(j)}$ for every $j\in 
\cA$ and therefore $\Pr(j\in \hat{\cA}_{\rw})\to1$. Thus, we only consider
$j'\in\cAc$. When $j'\in \hat{\cA}_{\rw}$, we know that by K-K-T
optimality conditions, we have 
$\lambda_N\hat{w}_{j'}\sgn(\hat{\beta}_{(j')})=\dl_{\rw}(\htheta_{\rw})$,
which means
\begin{align*}
&\frac{\lambda_N\hat{w}_{j'}\sgn(\hat{\beta}_{(j')})}{a_N}    
=\frac{\dl_{\rw}(\htheta_{\rw})}{a_N}
=\frac{\dl_{\rw}(\btheta_{\rt})}{a_N}+\frac{a_N
  \left\{\dl_{\rw}(\htheta_{\rw})-\dl_{\rw}(\btheta_{\rt})
  \right\}}{a_N^2}=:I_1+I_2.
\end{align*}
We have known that $I_1=\dl_{\rw}(\btheta_{\rt})/a_N\cvd\W_{\rw}$. We now prove that
$I_2=\Op$. We apply Taylor expansion to the $k$-th element of
$\dl_{\rw}(\htheta_{\rw})$ and have that
\begin{equation*}
\frac{a_N \left\{
\dl_{(k)}(\htheta_{\rw})-\dl_{(k)}(\btheta_{\rt})
\right\}}{a_N^2}=-\frac{1}{a_N^2}\sumn\frac{\delta_i}{\pi(\x,y_i)}\phi(\x_i;\btheta_{\rt})
\dg_{(k)}(\x_i;\btheta_{\rt})\dg\tp(\x_i;\btheta_{\rt})\tu_N+\tilde{\Delta}_{\rw(k)}+\tilde{R}_{\rw(k)},
\end{equation*}
where,
$\tu_N=a_N(\htheta_{\rw}-\btheta_{\rt})=\Op$,
\begin{equation*}
\tilde{\Delta}_{\rw(k)}=\frac{1}{a_N^2}\sumn\frac{\delta_i}{\pi(\x_i,y_i)}
\left\{y_i-p(\x_i;\btheta_{\rt})\right\}\sum_{j=1}^d\ddg_{(kj)}(\x_i;\btheta_{\rt})\hat{u}_{N(j)},
\end{equation*}
and 
\begin{align*}
\tilde{R}_{\rw(k)}&=-\frac{1}{2a_N^3}\sumn\frac{\delta_i}{\pi(\x_i,y_i)}\phi(\x_i;\batheta_k)
\left\{1-2p(\x_i;\batheta_k)
\right\}\dg_{(k)}(\x_i;\batheta_k)\tu_N\tp\dg^{\otimes2}(\x_i;\batheta_k)\tu_N\\
&\quad-\frac{2}{2a_N^3}\sumn\frac{\delta_i}{\pi(\x_i,y_i)}\phi(\x_i;\batheta_k)
\left\{\tu_N\tp\frac{\partial\dg_{(k)}(\x_i;\batheta_k)}{\partial\btheta}\right\} \left\{
\tu_N\tp\dg(\x_i;\batheta_k) \right\}\\
&\quad-\frac{1}{2a_N^3}\sumn\frac{\delta_i}{\pi(\x_i,y_i)}\phi(\x_i;\batheta_k)\dg_{(k)}(\x_i;\batheta_k)
\left\{ \tu_N\tp\ddg(\x_i;\batheta_k)\tu_N \right\}\\
&\quad+\frac{1}{2a_N^3}\sumn\frac{\delta_i}{\pi(\x_i,y_i)}\left\{
y_i-p(\x_i;\batheta_k)
\right\}\tu_N\tp \frac{\partial^2\dg_{(k)}(\x_i;\batheta_k)}{\partial\btheta^2}\tu_N.
\end{align*}
where $\batheta_k$ is between $\htheta_{\rw}$ and $\btheta_{\rt}$. First, we prove
that $\tilde{R}_{\rw(k)}$ is $\op$. We have that  
\begin{align*}
|\tilde{R}_{\rw(k)}|
& \leq\frac{\|\tu_N\|^2}{2a_N^3}\sumn\frac{\delta_i}{\pi(\x_i,y_i)}\phi(\x_i;\batheta_k)
    \left|1-2p(\x_i;\batheta_k)
    \right|\left|\dg_{(k)}(\x_i;\batheta_k)\right|\left\|\dg(\x_i;\batheta_k)\right\|^2\\
&\quad+\frac{2\|\tu_N\|^2}{2a_N^3}\sumn\frac{\delta_i}{\pi(\x_i,y_i)}\phi(\x_i;\batheta_k)
    \left\|\frac{\partial\dg_{(k)}(\x_i;\batheta_k)}{\partial\btheta}\right\| \left\|
    \dg(\x_i;\batheta_k) \right\|\\
  &\quad+\frac{\|\tu_N\|^2}{2a_N^2}\sumn\frac{\delta_i}{\pi(\x_i,y_i)}\frac{}{}\phi(\x_i;\batheta_k)\left|
    \dg_{(k)}(\x_i;\batheta_k) \right| \left\|\ddg(\x_i;\batheta_k)\right\|\\
  &\quad+\frac{\|\tu_N\|^2}{2a_N^3}\sumn\frac{\delta_i}{\pi(\x_i,y_i)}p(\x_i;\batheta_k)\left\|
    \frac{\partial^2\dg_{(k)}(\x_i;\batheta_k)}{\partial\btheta^2}
    \right\|\\
  &\quad+\frac{\|\tu_N\|^2}{2a_N^3}\sumn\frac{\delta_i}{\pi(\x_i,y_i)}y_i\left\|
    \frac{\partial^2\dg_{(k)}(\x_i;\batheta_k)}{\partial\btheta^2} \right\|\\
  &\leq \frac{\|\tu_N\|^2}{2a_N^3}\sumn\frac{\delta_i}{\pi(\x_i,y_i)}
    p(\x_i;\batheta_k)C(\x_i;\batheta_k)+\frac{\|\tu_N\|^2}{2a_N^3}\sumn\frac{\delta_i}{\pi(\x_i,y_i)}y_iB(\x_i)\\ 
  &\leq
    \frac{\|\tu_N\|^2e^{\acute{\alpha}_k-\alpha_{\rt}}e^{\alpha_{\rt}}}{2a_N^3}\sumn\frac{\delta_i}{\pi(\x_i,y_i)}
    e^{f(\x_i;\acute{\bbeta}_k)}
    C(\x_i;\batheta_k)+\frac{\|\tu_N\|^2}{2a_N^3}\sumn\frac{\delta_i}{\pi(\x_i,y_i)}y_iB(\x_i)\\
  &\leq
    \frac{\|\tu_N\|^2e^{\acute{\alpha}_k-\alpha_{\rt}}}{2Na_N}\sumn\frac{\delta_i}{\pi(\x_i,y_i)}
    e^{f(\x_i;\acute{\bbeta}_k)}
    C(\x_i;\batheta_k)+\frac{\|\tu_N\|^2}{2a_N^3}\sumn\frac{\delta_i}{\pi(\x_i,y_i)}y_iB(\x_i)\\
  &=\op,
\end{align*}
where
\begin{align*}
C(\x_i;\batheta) &=
\left|\dg_{(k)}(\x_i;\batheta)\right|\left\{\left\|\dg(\x_i;\batheta_k)\right\|^2
                 + \left\|\ddg(\x_i;\batheta)\right\| \right\}\\
&\quad + \left\|\frac{\partial\dg_{(k)}(\x_i;\batheta_k)}{\partial\btheta}\right\|
\left\| \dg(\x_i;\batheta_k) \right\| + \left\|
  \frac{\partial^2\dg_{(k)}(\x_i;\batheta_k)}{\partial\btheta^2} \right\|.
\end{align*}
Therefore, we proved that $\tilde{R}_{\rw(k)}=\op$. Next, we prove that
$\tilde{\Delta}_{\rw(k)}=\op$. We know that $\Exp\left[ a_N^{-2}\sumn\delta_i/\pi(\x_i,y_i)\left\{
y_i-p(\x_i;\btheta_{\rt})\right\}\ddg(\x_i;\btheta_{\rt})\right]=\0$. We also have
that for the every element of
$a_N^{-2}\sumn\delta_i/\pi(\x_i,y_i)\left\{y_i-p(\x_i;\btheta_{\rt})\right\}\ddg(\x_i;\btheta_{\rt})$, we
have
\begin{align*}
  &\Var\left[ a_N^{-2}\sumn\frac{\delta_i}{\pi(\x_i,y_i)}\left\{
    y_i-p(\x_i;\btheta_{\rt})
                 \right\}\ddg_{(jl)}(\x_i;\btheta_{\rt})\right]\\
  &\leq \frac{1}{a_N^4}\sumn\Exp \left\{
    p(\x_i;\btheta_{\rt})\ddg_{(jl)}^2(\x_i;\btheta_{\rt})
    \right\}
    \leq \frac{1}{a_N^2}\Exp[e^{f(\x;\bbeta_{\rt})}\|\ddg(\x;\btheta_{\rt})\|^2]\to0.
\end{align*}
Thus, due to Chebyshev's inequality, we know that
$\tilde{\Delta}_{\rw}=\op$. Since we know that
$\frac{1}{a_N^2}\sumn\delta_i/\pi(\x_i,y_i)\phi(\x_i;\btheta_{\rt})\dg^{\otimes2}(\x_i;\btheta_{\rt})=\bH_{\rw}=\Op$. Hence,
we have that 
$\dl_{\rw}(\htheta_{\rw})/a_N=\Op$.
Note that we also have
$(\lambda_N\hat{w}_{j'})/a_N
=(\lambda_N/a_N)|\hat{\beta}_{\rp(j')}|^{-\gamma}\cvp\infty$.
Therefore, 
\begin{align*}
&\Pr(j'\in\hat{\cA}_{\rw})
\leq\Pr\left\{\lambda_N\hat{w}_{j'}\sgn(\hat{\beta}_{(j')})
=\dl_{\rw}(\htheta_{\rw})\right\}
=\Pr\left\{\frac{\lambda_N\hat{w}_{j'}\sgn(\hat{\beta}_{(j')})}{a_N}
=\frac{\dl_{\rw}(\htheta_{\rw})}{a_N}\right\}\to0.
\end{align*}
This finished the proof of consistency for variable selection. 

\subsubsection{Proof when $p_{N}$ and $s_{N}$ diverge with $N$}\label{sec:ipw-divergep}

In this section, we prove Theorem~\ref{thm:asym-ipw} when $p_{N}$ and $s_{N}$
diverge with $N$.
Let $\eta_{i(j)}=\{\delta_i/\pi(\x_i,y_i)\}\left\{
y_i-p(\x_i;\btheta_{\rt})\right\}\dg_{(j)}(\x_i;\btheta_{\rt})$, we have
that the $j$-th element of $\dl_{\rw}(\btheta_{\rt})$,
$\dl_{\rw(j)}(\btheta_{\rt})=\sumn\eta_{i(j)}$, satisfies $\Exp(\eta_{i(j)})=0$, and
\begin{align*}
  &\Var(\eta_{i(j)}|\x_{i})
  =p(\x_{i};\btheta_{\rt})\{1-p(\x_{i};\btheta_{\rt})\}
  \{\dg_{(j)}(\x_{i};\btheta_{\rt})\}^{2}
  \left\{1-p(\x_{i};\btheta_{\rt})+\frac{p(\x_{i};\btheta_{\rt})}{\rho\varphi(\x_{i})}\right\}\\
  &\le e^{\alpha_{\rt}}e^{f(\x_{i};\bbeta_{\rt})}
  \{\dg_{(j)}(\x_{i};\btheta_{\rt})\}^{2}
  +\frac{1}{\rho}e^{2\alpha_{\rt}}\frac{e^{2f(\x_{i};\bbeta_{\rt})}}{
    \varphi(\x_{i})}\{\dg_{(j)}(\x_{i};\btheta_{\rt})\}^{2}
  \le e^{\alpha_{\rt}}(1+c_{N})B^{2}(\x_{i}).
\end{align*}
Since $|y_{i}-p(\x_{i};\btheta_{\rt})|\le 1$, we have that
$|\eta_{i,(j)}|\le\{1+\rho^{-1}\varphi^{-1}(\x_{i})\}B(\x_{i})\le
R_{N}/e^{\alpha_{\rt}}$, where $R_{N}=\max_{i=1,...,N}\{1+c_N/\varphi(\x_i)\}B(\x_i)$.
Conditional on $\Dn=\mathcal{F}(\x_{1},\x_{2},\ldots,\x_{N})$ and applying
Bernstein's inequality, we have that for any sequence $k_{N}$,
\begin{align*}
  \Pr\left(\max_{j=1,\ldots,k_{N}}\left|\sumn\eta_{i(j)}\right|>t\Big|\Dn\right)
  \le 2k_{N}\exp\left\{\frac{-t^{2}}{2e^{\alpha_{\rt}}(1+c_{N})\sumn B^{2}(\x_{i})
  +\frac{2R_{N}t}{3e^{\alpha_{\rt}}}}\right\}.
\end{align*}
Then, by taking
$
T=\sqrt{4e^{\alpha_{\rt}}(1+c_{N})\log\left(2k_{N}/\delta_{N}\right)
\sum_{i=1}^{N}B^{2}(\x_{i})}+(4/3)e^{-\alpha_{\rt}}R_{N}\log\left(2k_{N}/\delta_{N}\right)
$,
we have that
$\Pr\left(\max_{j=1,\ldots,k_{N}}\left|\sumn\eta_{i,(j)}\right|>T\right)
  =\Exp\left\{\Pr\left(\max_{j=1,\ldots,k_{N}}
  \left|\sumn\eta_{i,(j)}\right|>T\Big|\Dn\right)\right\}
<\delta_{N}$.
Now, since $(1/N)\sumn B^{2}(\x_{i})\le R_{N}^{2}$ and 
$\alpha_{\rt}\gtrsim\log\{(\max\{\log p_{N}, s_{N}\})/N\}$, 
\begin{align}\label{eq:ipw-infnorm}
  \Pr\left(\left\|\frac{1}{a_{N}^{2}}\sumn\eeta_{i(\cA)}\right\|_{\infty}
  \gtrsim\sqrt{\frac{\log s_{N}}{a_{N}^{2}}}\right)
  <\frac{1}{s_{N}},
  \Pr\left(\left\|\frac{1}{a_{N}^{2}}\sumn\eeta_{i(\cAc)}\right\|_{\infty}
  \gtrsim\sqrt{\frac{\log p_{N}}{a_{N}^{2}}}\right)
  <\frac{1}{p_{N}},
\end{align}
by taking appropriate $k_{N}$ and $\delta_{N}$.
Using exactly the same arguments and the multi-variate version of Bernstein's
inequality, we can also show that 
$
\Pr\left(\left\|a_{N}^{-2}\sumn\eeta_{i(\cA)}\right\|\gtrsim\sqrt{a_{N}^{-2}s_{N}}\right)
<5^{-s_{N}},
$
by taking $\delta_{N}=5^{-s_{N}}$. 
Next, we need to show $\bH_{\rw}$, where
\begin{align*}
  \bH_{\rw}=\frac{1}{a_{N}^{2}}\sumn\frac{\delta_{i}}{\pi(\x_{i},y_{i})}
  \phi(\x_{i};\btheta_{\rt})\dg_{(\cA)}^{\otimes2}(\x_{i};\btheta_{\rt})
  =\frac{1}{N}\sumn\frac{\delta_{i}}{\pi(\x_{i},y_{i})}
  \frac{e^{f(\x_{i};\bbeta_{\rt})}}{\{1+e^{g_{(\cA)}(\x_{i};\btheta_{\rt})}\}^{2}}
  \dg_{(\cA)}^{\otimes2}(\x_{i};\btheta_{\rt}),
\end{align*}
is positive-definite with a high probability. Using
the same arguments in~\cite{wang2021nonuniform} of (S.12), we know that
$\Var(\bH_{\rw,i}|\x_{i})\le B^{2}(\x_{i})/e^{\alpha_{\rt}}$, where
$\bH_{\rw}=(1/N)\sumn\bH_{\rw,i}$, and $\|\bH_{\rw,i}\|\le R_{N}/e^{\alpha_{\rt}}$.
Again, we have that
$\Pr\left(\left\|\sum\{\bH_{\rw,i}-\Exp(\bH_{\rw,i})\}\right\|
\gtrsim\sqrt{s_{N}e^{-\alpha_{\rt}}}+s_{N}e^{-\alpha_{\rt}}\right)<5^{-s_{N}}$,
which implies that $\|\bH_{\rw}-\Exp(\bH_{\rw})\|\lesssim
  \sqrt{s_{N}/a_{N}^{2}}+s_{N}/a_{N}^{2}$ and $\|\bH_{\rw}-\Exp(\bH_{\rw})\|_{\infty}\lesssim
  \sqrt{s_{N}^{2}/a_{N}^{2}}+s_{N}^{3/2}/a_{N}^{2}$.
Thus, it is easy to know that $\|\bH_{\rw}\|$ and $\|\bH_{\rw}\|_{\infty}$ are
bounded away from zero and bounded above with a probability tends to 1.
Now, due to the K-K-T conditions, we know that
$\htheta=(\htheta_{\rw(\cA)}\tp,\0\tp)\tp$ is the unique adaptive lasso
estimator if and only if
\begin{align*}
  \dl_{\rw}(\hat{\theta}_{(j)})=\lambda_{N}w_{j}\sgn(\hat{\theta}_{(j)}),
  \text{ for }j\in\cA, \text{ and }
  |\dl_{\rw}(\hat{\theta}_{(j)})|\le\lambda_{N}w_{j}, \text{ for }j\notin\cA.
\end{align*}
To prove the K-K-T conditions, it is sufficient to show that
\begin{align}\label{eq:hd-kkt}
\begin{split}
  &\sgn(\theta_{\rt(j)})(\theta_{\rt(j)}-\hat{\theta}_{\rw(j)})<|\theta_{\rt(j)}|, \forall j\in\cA, 
  \text{ and } \\
  &\frac{1}{a_{N}}\left|\sum_{i=1}^{N}
  \frac{\delta_{i}\{y_{i}-p(\x_{i(\cA)};\htheta_{\rw(\cA)})\}}{\pi(\x_{i},y_{i})}
  \dg_{(j)}(\x_{i(\cA)};\htheta_{\rw,(\cA)})\right|\le\frac{\lambda_{N}w_{j}}{a_{N}}, \forall j\notin\cA,
\end{split}
\end{align}
where $\htheta_{\rw(\cA)}$ is the solution of 
\begin{align*}
  \sumn\frac{\delta_{i}}{\pi(\x_{i},y_{i})}
  \{y_{i}-p(\x_{i(\cA)};\htheta_{\rw(\cA)})\}\dg_{(j)}(\x_{i(\cA)};\htheta_{\rw(\cA)})
  =\lambda_{N}w_{j}\sgn(\theta_{\rt(j)}), \forall j\in\cA.
\end{align*}
Now, we define
$\htheta_{(\cA)}^{*}=\arg\min_{\btheta_{(\cA)}}\gamma_{\rw}(\btheta_{(\cA)})$, where 
\begin{align*}
  \gamma_{\rw}(\btheta_{(\cA)})
  =\frac{1}{2}\btheta_{(\cA)}\tp(a_{N}^{2}\bH_{\rw})\btheta_{(\cA)}
  -\sumn\eeta_{i(\cA)}\tp(\btheta_{(\cA)}-\btheta_{\rt(\cA)})
  +\lambda_{N}\sum_{j\in\cA}w_{j}\sgn(\theta_{\rt(j)})\theta_{(j)}.
\end{align*}
Letting $\bm{\xi}$ be the vector of components to be
$w_{j}\sgn(\theta_{\rt(j)})$, for $j\in\cA$, we have that
\begin{align}\label{eq:ipw-theta-star}
  \btheta_{(\cA)}^{*}-\btheta_{\rt(\cA)}=(a_{N}^{2}\bH_{\rw})^{-1}
  \left\{\sumn\eeta_{i,(\cA)}-\lambda_{N}\bm{\xi}\right\}.
\end{align}
Therefore, 
\begin{align*}
  &\|\btheta_{(\cA)}^{*}-\btheta_{\rt(\cA)}\|
  \le\|\bH_{\rw}^{-1}\|\frac{1}{a_{N}^{2}}\left\|\sum_{i=1}^{N}\eeta_{i(\cA)}\right\|
  +\|\bH_{\rw}^{-1}\|\frac{1}{a_{N}^{2}}\lambda_{N}\|\bm{\xi}\|
  \lesssim\sqrt{\frac{s_{N}}{a_{N}^{2}}}
  +\frac{\lambda_{N}\sqrt{s_{N}}}{a_{N}^{2}b_{N}^{\gamma}},\\ 
  &\|\btheta_{(\cA)}^{*}-\btheta_{\rt(\cA)}\|_{\infty}
  \lesssim\frac{1}{a_{N}^{2}}\left\|\sum_{i=1}^{N}\eeta_{i(\cA)}\right\|_{\infty}
  +\frac{1}{a_{N}^{2}}\lambda_{N}\|\bm{\xi}\|_{\infty}
  \lesssim\sqrt{\frac{\log s_{N}}{a_{N}^{2}}}
  +\frac{\lambda_{N}}{a_{N}^{2}b_{N}^{\gamma}}
\end{align*}
Here, by the same arguments as~\cite{huang2008iterated} and Lemma 2
in~\cite{hjort2011asymptotics}, we can show that
\begin{align}\label{eq:ipw-theta}
  \|\htheta_{\rw(\cA)}-\btheta_{(\cA)}^{*}\|^{2}
  =o_{P}\left(\frac{s_{N}}{a_{N}^{2}}+\frac{\lambda_{N}^{2}s_{N}}{a_{N}^{4}b_{N}^{2\gamma}}\right),
\end{align}
which implies that 
$\|\htheta_{\rw(\cA)}-\btheta_{\rt,(\cA)}\|
\lesssim\sqrt{s_{N}/a_{N}^{2}}+(\lambda_{N}\sqrt{s_{N}})/(a_{N}^{2}b_{N}^{\gamma})$. Thus,
condition
$\sgn(\hat{\theta}_{\rw(j)})(\theta_{\rt(j)}-\hat{\theta}_{\rw(j)})<|\theta_{\rt(j)}|,\forall
j\in\cA$ satisfies under Assumption~\ref{asm:a6} and the fact that
$\|\htheta_{\rw(\cA)}-\btheta_{\rt(\cA)}\|_{\infty}\lesssim\sqrt{\log
s_{N}/a_{N}^{2}}+\lambda_{N}/(a_{N}^{2}b_{N}^{\gamma})$. Then, we turn to the second condition in~\eqref{eq:hd-kkt}.
From the previous proof of Theorem~\ref{thm:asym-ipw}, using the same notations, we have
that 
\begin{align*}
  &\frac{a_{N}\left\{\dl_{(j)}(\htheta_{\rw(\cA)})-\dl_{(j)}(\btheta_{\rt(\cA)})\right\}}{a_{N}^{2}}
  +\frac{1}{a_{N}^{2}}\sumn\frac{\delta_{i}}{\pi(\x_{i},y_{i})}\phi(\x_{i};\btheta_{\rt(\cA)})
  \dg\tp(\x_{i};\btheta_{\rt(\cA)})\hat{\bu}_{N(\cA)}\dg_{(j)}(\x_{i};\btheta_{\rt(\cA)})\\
  &=\tilde{\Delta}_{\rw(j)}+\tilde{R}_{\rw(j)}.
\end{align*}
From the proof of Theorem~\ref{thm:asym-ipw}, we know that
\begin{align*}
  |\tilde{\Delta}_{\rw(j)}|\le\frac{1}{a_{N}^{2}}
  \left|\sumn\frac{\delta_{i}}{\pi(\x_{i},y_{i})}
  \{y_{i}-p(\x_{i};\btheta_{\rt(\cA)})\}B(\x_{i})\right|\|\hat{\bu}_{N(\cA)}\|.
\end{align*}
Using the same arguments as we bound $\sumn\eta_{i(j)}$, we have that with
probability at least $1-a_{N}^{-1}$,
  $|\tilde{\Delta}_{\rw(j)}|\lesssim\sqrt{\log a_{N}/a_{N}^{2}}\|\hat{\bu}_{N(\cA)}\|$.
From the proof of Theorem~\ref{thm:asym-ipw}, we also have that
\begin{align*}
  |\tilde{R}_{\rw(j)}|\lesssim\frac{\|\hat{\bu}_{N(\cA)}\|^{2}}{Na_{N}}
  \sumn\frac{\delta_{i}}{\pi(\x_{i};y_{i})}B(\x_{i})
  +\frac{\|\hat{\bu}_{N(\cA)}\|^{2}}{a_{N}^{3}}
  \sumn\frac{\delta_{i}}{\pi(\x_{i};y_{i})}y_{i}B(\x_{i}).
\end{align*}
Note that since $\Var\{\delta_{i}/\pi(\x_{i};y_{i})B(\x_{i})|\x_{i}\}
=[\{1-\pi(\x_{i};y_{i})\}/\pi(\x_{i};y_{i})]B(\x_{i})\lesssim
B(\x_{i})/\{e^{\alpha_{\rt}}\varphi(\x_{i})\}$, we also have that with
probability at least $1-a_{N}^{-1}$,
\begin{align*}
  |\tilde{R}_{\rw(j)}|\lesssim\frac{\|\hat{\bu}_{N(\cA)}\|^{2}\log a_{N}}{a_{N}}
  \left\{\frac{1}{N}\sumn B(\x_{i})+\sqrt{\sumn\frac{B(\x_{i})}{N^{2}e^{\alpha_{\rt}}}}\right\}.
\end{align*}
Thus, we only need to consider the dominating term
\begin{align*}
  &\left|\frac{1}{a_{N}^{2}}\sumn\frac{\delta_{i}}{\pi(\x_{i},y_{i})}\phi(\x_{i};\btheta_{\rt(\cA)})
  \dg\tp(\x_{i};\btheta_{\rt(\cA)})\hat{\bu}_{N(\cA)}\dg_{(j)}(\x_{i};\btheta_{\rt(\cA)})\right|\\
  &\le\left|\frac{1}{a_{N}^{2}}\sumn\frac{\delta_{i}}{\pi(\x_{i};y_{i})}
  \phi(\x_{i};\btheta_{\rt(\cA)})B(\x_{i})\right|\|\hat{\bu}_{N(\cA)}\|_{\infty}
  \lesssim\|\hat{\bu}_{N(\cA)}\|_{\infty},
\end{align*}
where the last inequality can be obtained using arguments as we
bound $\bH_{\rw}$. Therefore,
\begin{align*}
  &\left|\frac{\dl_{(j)}(\htheta_{\rw(\cA)})}{a_{N}}\right|
  \lesssim\left|\frac{\dl_{(j)}(\btheta_{\rt(\cA)})}{a_{N}}\right|
  +\|\htheta_{\rw(\cA)}-\btheta_{\rt(\cA)}\|_{\infty}\\
  &\quad+\sqrt{\frac{\log a_{N}}{a_{N}^{2}}}
  \left(\sqrt{\frac{s_{N}}{a_{N}^{2}}}+\frac{\lambda_{N}\sqrt{s_{N}}}{a_{N}^{2}b_{N}^{\gamma}}\right)
  +\frac{\log a_{N}}{a_{N}}\left(\frac{s_{N}}{a_{N}^{2}}+\frac{\lambda_{N}^{2}s_{N}}{a_{N}^{4}b_{N}^{2\gamma}}\right)\\
  &\lesssim\sqrt{\log p_{N}}+\sqrt{\frac{\log s_{N}}{a_{N}^{2}}}
  +\frac{\log a_{N}}{a_{N}}
  \left(\sqrt{\frac{s_{N}}{a_{N}^{2}}}+\frac{\lambda_{N}\sqrt{s_{N}}}{a_{N}^{2}b_{N}^{\gamma}}\right)
  \lesssim \sqrt{\log p_{N}}.
\end{align*} 
The last inequality is because $s_{N}<p_{N}$ and Assumption~\ref{asm:a6}. Then,
we have that $\forall j\notin\cA$, and $a_{N}$ sufficiently large, since $\log
s_{N}\le\log p_{N}$ and $\sqrt{s_{N}} / a_{N}\lesssim\sqrt{\log p_{N}}$,
\begin{align*}
  \frac{a_{N}}{\lambda_{N}w_{j}}\left|\frac{\dl_{(j)}(\htheta_{\rw(\cA)})}{a_{N}}\right|
  \lesssim\frac{a_{N}\sqrt{\log p_{N}}}{\lambda_{N}w_{j}}
  +\frac{\sqrt{\log s_{N}}}{\lambda_{N}w_{j}}
  +\frac{\log a_{N}}{a_{N}}
  \left(\frac{\sqrt{s_{N}}}{\lambda_{N}w_{j}}
  +\frac{\sqrt{s_{N}}}{a_{N}w_{j}b_{N}^{\gamma}}\right)\to 0.
\end{align*}
Then, combining~\eqref{eq:ipw-theta-star},~\eqref{eq:ipw-theta}, and
$(\sqrt{s_{N}}\lambda_{N})/\sqrt{N_{1}}=\oo$, we have that
$\forall \bm{e}_{N}$, $\|\bm{e}_{N}\|=1$,
\begin{align*}
  \sqrt{N_{1}}\bm{e}_{N}\tp\V_{\rw(\cA)}^{-1/2}(\htheta_{\rw(\cA)}-\btheta_{\rt,(\cA)})
  =\bm{e}_{N}\tp\frac{\sqrt{\Exp\{e^{f(\x;\bbeta_{\rt})}\}}}{a_{N}}
  \V_{\rw(\cA)}^{-1/2}\bH_{\rw}^{-1}\sumn\eeta_{i,(\cA)}+\op\cvd\Nor(0,1),
\end{align*}
due to Lindeberg-Feller's central limit theorem indicated in the previous proof of
Theorem~\ref{thm:asym-ipw}.

\subsection{Proof of Proposition~\ref{prop:optlas}}

We first give a lemma for general optimal functions.
\begin{lemma}
\label{lem:optgeneral}
Assume that $h(\x)^2$ and $\varphi(\x)$ are integrable function with
$\Exp\{\varphi(\x)\}=1$. The optimal function $\varphi^{**}(\x)$ that minimize
the value $\Exp\left\{\frac{h^2(\x)}{\varphi(\x)}\right\}$ is given as
$\varphi^{**}(\x)=\frac{h(\x)}{\Exp\{h(\x)\}}$. 
\end{lemma}
\begin{proof} of Lemma~\ref{lem:optgeneral}: 
Appying Cauchy-Schwartz inequality, we have that 
\begin{equation*}
\Exp\{h(\x)\}^2=\Exp\left\{\frac{h(\x)}{\sqrt{\varphi(\x)}}\sqrt{\varphi(\x)}\right\}^2
\leq\Exp\left\{\frac{h^2(\x)}{\varphi(\x)}\right\}\Exp\{\varphi(\x)\}
=\Exp\left\{\frac{h^2(\x)}{\varphi(\x)}\right\}.
\end{equation*}
Therefore, we have that
$\Exp\left\{\frac{h^2(\x)}{\varphi(\x)}\right\}\geq\Exp\{h(\x)\}^2$ and the
equality holds if and only if $\sqrt{\varphi(\x)}=Kh(\x)/\sqrt{\varphi(\x)}$,
where $K$ is a constant. Therefore, $\varphi^{**}(\x)=Kh(\x)$, and since
$\Exp\{\varphi^{**}(\x)\}=1$, we know that $\varphi^{**}(\x)=h(\x)/\Exp\{h(\x)\}$.
\end{proof}

Now, we prove Proposition~\ref{prop:optlas}.
\begin{proof}
We first calculate the optimal function that minimizes
$\text{tr}(\V_{\rw(\cA)})$. We have that 
\begin{align*}
&\text{tr}(\V_{\rw(\cA)})
=\text{tr}\left\{\M_{(\cA)}^{-1}\M_{\rw(\cA)}\M_{(\cA)}^{-1}\right\}\\
&=\text{tr}\left\{\M_{(\cA)}^{-1}\Exp\left[\left\{1+\frac{ce^{f(\x;\bbeta_{\rt})}}{\varphi(\x)}\right\}
e^{f(\x;\bbeta_{\rt})}\dg_{(\cA)}^{\otimes2}(\x;\btheta_{\rt})\right]\M_{(\cA)}^{-1}
\right\}.
\end{align*}
We focus on the values that related to $\varphi(\x)$. We know that 
\begin{equation*}
\Exp\left\{\varphi^{-1}(\x)e^{2f(\x;\bbeta_{\rt})}\dg_{(\cA)}^{\otimes2}(\x;\btheta_{\rt})\right\}
=e^{-2\alpha_{\rt}}\{1+\op\}\Exp\left\{\varphi^{-1}(\x)p^2(\x;\btheta_{\rt})\dg_{(\cA)}^{\otimes2}(\x;\btheta_{\rt})\right\}.
\end{equation*}
Therefore, we need to minimize 
\begin{align*}
&\text{tr}\left[\Exp\left\{\frac{p^2(\x;\btheta_{\rt})\M_{(\cA)}^{-1}
\dg_{(\cA)}^{\otimes2}(\x;\btheta_{\rt})\M_{(\cA)}^{-1}}{\varphi(\x)}\right\}\right]\\
&=\Exp\left[\text{tr}\left\{\frac{p^2(\x;\btheta_{\rt})\M_{(\cA)}^{-1}
\dg_{(\cA)}^{\otimes2}(\x;\btheta_{\rt})\M_{(\cA)}^{-1}}{\varphi(\x)}\right\}\right]
=\Exp\left[
\frac{p^2(\x;\btheta_{\rt})\|\M_{(\cA)}^{-1}
\dg_{(\cA)}(\x;\btheta_{\rt})\|^2}{\varphi(\x)}\right].
\end{align*}
Applying Lemma~\ref{lem:optgeneral}, we know that the minimizer is given as 
\begin{equation*}
\varphi_{\mmse}(\x)=\frac{p(\x;\btheta_{\rt})\|\M_{(\cA)}^{-1}
\dg_{(\cA)}(\x;\btheta_{\rt})\|}{\Exp\left\{p(\x;\btheta_{\rt})\|\M_{(\cA)}^{-1}
\dg_{(\cA)}(\x;\btheta_{\rt})\|\right\}}.
\end{equation*}
Next, we calculate the optimal function that minimize $\text{tr}(\M_{\rw(\cA)})$. We have
that  
\begin{align*}
\text{tr}(\M_{\rw(\cA)})
=\text{tr}\left\{\Exp\left[\left\{1+\frac{ce^{f(\x;\bbeta_{\rt})}}{\varphi(\x)}\right\}
e^{f(\x;\bbeta_{\rt})}\dg_{(\cA)}^{\otimes2}(\x;\btheta_{\rt})\right]
\right\}.
\end{align*}
Therefore, we need to minimize 
$
\text{tr}\left[\Exp\left\{\frac{p^2(\x;\btheta_{\rt})
\dg_{(\cA)}^{\otimes2}(\x;\btheta_{\rt})}{\varphi(\x)}\right\}\right]
=\Exp\left[
\frac{p^2(\x;\btheta_{\rt})\|
\dg_{(\cA)}(\x;\btheta_{\rt})\|^2}{\varphi(\x)}\right].
$
Applying Lemma~\ref{lem:optgeneral}, we know that the minimizer is given as 
\begin{equation*}
\varphi_{\mvc}(\x)=\frac{p(\x;\btheta_{\rt})\|
\dg_{(\cA)}(\x;\btheta_{\rt})\|}{\Exp\left\{p(\x;\btheta_{\rt})\|
\dg_{(\cA)}(\x;\btheta_{\rt})\|\right\}}.
\end{equation*}
\end{proof}

\subsection{Proof of Theorem~\ref{thm:optprb}}

In the proof of Theorem~\ref{thm:asym-ipw}, we know that 
$a_N(\htheta^{\radp}_{\rw(\cA)}-\btheta_{\rt(\cA)})\cvd\M_{(\cA)}^{-1}\M_{\rw(\cA)}^{1/2}\W_{(\cA)}$.
To simplify the representation, letting
$h(\btheta)=e^{-2\alpha_{\rt}}\text{MSPE}(\btheta)
=e^{-2\alpha_{\rt}}\Exp\left[\left\{p(\x;\btheta)-p(\x;\btheta_{\rt})\right\}^2\right]$,
we have that 
$\partial\left\{p(\x;\btheta)-p(\x;\btheta_{\rt})\right\}^2/\partial\btheta
=2\left\{p(\x;\btheta)-p(\x;\btheta_{\rt})\right\}\phi(\x;\btheta)\dg(\x;\btheta)$,
and 
\begin{align*}
\frac{\partial^2\left\{p(\x;\btheta)-p(\x;\btheta_{\rt})\right\}^2}{\partial\btheta\partial\btheta\tp}
&=2\phi^2(\x;\btheta)\dg^{\otimes2}(\x;\btheta)
+2\left\{p(\x;\btheta)-p(\x;\btheta_{\rt})\right\}\frac{\partial\phi(\x;\btheta)}{\partial\btheta}\dg(\x;\btheta)\\
&\quad+2\left\{p(\x;\btheta)-p(\x;\btheta_{\rt})\right\}\phi(\x;\btheta)\ddg(\x;\btheta).
\end{align*}
Note that $|p(\x;\btheta)-p(\x;\btheta_{\rt})|\leq 2$ and thus,
\begin{equation*}
\left|\frac{\partial\left\{p(\x;\btheta)-p(\x;\btheta_{\rt})\right\}^2}{\partial\btheta}\right|
\leq
2\left|p(\x;\btheta)-p(\x;\btheta_{\rt})\right|\left|\phi(\x;\btheta)\dg(\x;\btheta)\right|\leq
4B(\x),
\end{equation*}
and 
\begin{align*}
&\left|\frac{\partial^2\left\{p(\x;\btheta)-p(\x;\btheta_{\rt})\right\}^2}{\partial\btheta\partial\btheta\tp}\right|\\
&\leq 2\left|\phi^2(\x;\btheta)\dg^{\otimes2}(\x;\btheta)\right|
+2\left|\left\{p(\x;\btheta)-p(\x;\btheta_{\rt})\right\}\right|\left|\frac{\partial\phi(\x;\btheta)}{\partial\bt}\dg(\x;\btheta)\right|\\
&\quad+2\left|\left\{p(\x;\btheta)-p(\x;\btheta_{\rt})\right\}\right|\left|\phi(\x;\btheta)\ddg(\x;\btheta)\right|
\leq 10B(\x).
\end{align*}
Hence, due to dominating convergence theorem, we know that the expectation and
derivatives are interchangeable. Thus, we have that 
\begin{equation*}
\frac{\partial h(\btheta)}{\partial\btheta}
=e^{-2\alpha_{\rt}}\Exp\left[2\left\{p(\x;\btheta)-p(\x;\btheta_{\rt})\right\}\phi(\x;\btheta)\dg(\x;\btheta)\right],
\end{equation*}
and 
\begin{align*}
\frac{\partial h(\btheta)}{\partial\btheta\partial\btheta\tp}
&=e^{-2\alpha_{\rt}}\Exp\Big[2\phi^2(\x;\btheta)\dg^{\otimes2}(\x;\btheta)
+2\left\{p(\x;\btheta)-p(\x;\btheta_{\rt})\right\}\frac{\partial\phi(\x;\btheta)}{\partial\btheta}\dg(\x;\btheta)\\
&\quad+2\left\{p(\x;\btheta)-p(\x;\btheta_{\rt})\right\}\phi(\x;\btheta)\ddg(\x;\btheta)
\Big].
\end{align*}
Due to dominating convergence theorem, we also know that the first and second
derivative of $h(\btheta)$ are continuous. We have that 
\begin{equation*}
\frac{\partial h(\btheta)}{\partial\btheta}\Big|_{\btheta=\btheta_{\rt}}=\0,
\frac{\partial^2h(\bt)}{\partial\btheta\partial\btheta\tp}\Big|_{\btheta=\btheta_{\rt}}
=2e^{-2\alpha_{\rt}}\Exp\left[\phi^2(\x;\btheta_{\rt})\dg^{\otimes2}(\x;\btheta_{\rt})\right]
\end{equation*}
Note that
$e^{-2\alpha_{\rt}}\Exp\left[\phi^2(\x;\btheta_{\rt})\dg^{\otimes2}(\x;\btheta_{\rt})\right]\to\bOmega$
by the dominating convergence theorem. Now
applying Theorem 1.12(ii) in \cite{Shao2003}, we have that 
\begin{align*}
&a_N^2\left\{h(\htheta^{\radp}_{\rw(\cA)})-h(\btheta_{\rt(\cA)})\right\}
=a_N^2e^{-2\alpha_{\rt}}\Exp\left[\left\{p(\x;\htheta_{\rw(\cA)})-p(\x;\btheta_{\rt(\cA)})\right\}^2\right]\\
&\cvd\frac{1}{2!}2\W_{(\cA)}\tp\M_{\rw(\cA)}^{1/2}\M_{(\cA)}^{-1}\bOmega_{(\cA)}\M_{(\cA)}^{-1}\M_{\rw(\cA)}^{1/2}\W_{(\cA)}
=\W_{(\cA)}\tp\M_{\rw(\cA)}^{1/2}\M_{(\cA)}^{-1}\bOmega_{(\cA)}\M_{(\cA)}^{-1}\M_{\rw(\cA)}^{1/2}\W_{(\cA)}.
\end{align*}
Considering $N_1=a_N^2\Exp\left\{e^{f(\x;\bbeta_{\rt})}\right\}$,
applying Slutsky's theorem, we have that 
\begin{align*}
&N_1e^{-2\alpha_{\rt}}\Exp\left[\left\{p(\x;\htheta_{\rw(\cA)})-p(\x;\btheta_{\rt(\cA)})\right\}^2\right]\\
&\cvd
\Exp^{-1}\left\{e^{f(\x;\bbeta_{\rt})}\right\}\W_{(\cA)}\tp\M_{\rw(\cA)}^{1/2}\M_{(\cA)}^{-1}
\bOmega_{(\cA)}\M_{(\cA)}^{-1}\M_{\rw(\cA)}^{1/2}\W_{(\cA)}.
\end{align*}
Since $\W_{\rw(\cA)}=\Nor(\0,\I)$, we have 
\begin{align*}
&\Exp\left\{\W_{(\cA)}\tp\M_{\rw(\cA)}^{1/2}\M_{(\cA)}^{-1}\bOmega_{(\cA)}\M_{(\cA)}^{-1}\M_{\rw(\cA)}^{1/2}\W_{(\cA)}\right\}
=\text{tr}\left\{\M_{(\cA)}^{-1}\bOmega_{(\cA)}\M_{(\cA)}^{-1}\M_{\rw(\cA)}\right\}\\
&=\text{tr}\left\{\M_{(\cA)}^{-1}\bOmega_{(\cA)}\M_{(\cA)}^{-1}\Exp\left[\left\{1+\frac{ce^{f(\x;\bbeta_{\rt})}}{\varphi(\x)}\right\}
e^{f(\x;\bbeta_{\rt})}\dg_{(\cA)}^{\otimes2}(\x;\btheta_{\rt})\right].
\right\}
\end{align*}
We focus on the values that related to $\varphi(\x)$. We know that 
\begin{equation*}
\Exp\left\{\varphi^{-1}(\x)e^{2f(\x;\bbeta_{\rt})}\dg_{(\cA)}^{\otimes2}(\x;\btheta_{\rt})\right\}
=e^{-2\alpha_{\rt}}\{1+\op\}\Exp\left\{\varphi^{-1}(\x)p^2(\x;\btheta_{\rt})\dg_{(\cA)}^{\otimes2}(\x;\btheta_{\rt})\right\}.
\end{equation*}
Therefore, we need to minimize 
\begin{align*}
&\text{tr}\left[\Exp\left\{\M_{(\cA)}^{-1}\bOmega_{(\cA)}\M_{(\cA)}^{-1}
\frac{p^2(\x;\btheta_{\rt})\dg_{(\cA)}^{\otimes2}(\x;\btheta_{\rt})}{\varphi(\x)}\right\}\right]\\
&=\Exp\left[\text{tr}\left\{\M_{(\cA)}^{-1}\bOmega_{(\cA)}\M_{(\cA)}^{-1}
\frac{p^2(\x;\btheta_{\rt})\dg_{(\cA)}^{\otimes2}(\x;\btheta_{\rt})}{\varphi(\x)}\right\}\right]\\
&=\Exp\left[\text{tr}\left\{
\frac{p^2(\x;\btheta_{\rt})\bOmega_{(\cA)}^{1/2}\M_{(\cA)}^{-1}
\dg_{(\cA)}^{\otimes2}(\x;\btheta_{\rt})\M_{(\cA)}^{-1}\bOmega_{(\cA)}^{1/2}}{\varphi(\x)}\right\}\right]\\
&=\Exp\left[
\frac{p^2(\x;\btheta_{\rt})\|\bOmega_{(\cA)}^{1/2}\M_{(\cA)}^{-1}
\dg_{(\cA)}(\x;\btheta_{\rt})\|^2}{\varphi(\x)}\right].
\end{align*}
Applying Lemma~\ref{lem:optgeneral}, we know that the minimizer is given as 
\begin{equation*}
\varphi_{\mpr}(\x)=\frac{p(\x;\btheta_{\rt})\|\bOmega_{(\cA)}^{1/2}\M_{(\cA)}^{-1}
\dg_{(\cA)}(\x;\btheta_{\rt})\|}{\Exp\left\{p(\x;\btheta_{\rt})\|\bOmega_{(\cA)}^{1/2}\M_{(\cA)}^{-1}
\dg_{(\cA)}(\x;\btheta_{\rt})\|\right\}}.
\end{equation*}

\subsection{Proof of Proposition~\ref{prop:invar}}\label{sec:prfprop1}

First, we know that $g(\x;\btheta)=g(\A\x;\B\tp\btheta)$. Since the
equation holds for all $\x$ and $\btheta$, if we take the derivative with respect to
$\btheta$ on both sides, the equation still holds. Thus, we have that 
$\dg(\x;\btheta)=\B\dg(\A\x;\B\tp\btheta)$.
If we re-scale the whole covariate variable $\x$ to $\tx=\A\x$, we need to
reparameterize $\btheta_{\rt}$ to $\ttheta_{\rt}=\B\tp\btheta_{\rt}$ to remain the
problem invariant. We have that for $\tx$ and $\ttheta_{\rt}$
$\dg(\tx;\ttheta_{\rt})=\dg(\A\x;\B\tp\btheta_{\rt})
=\B^{-1}\dg(\x;\btheta_{\rt})$.
Now, we know that 
\begin{equation*}
\tilde{\M}=\Exp\{e^{f(\tx;\tbeta_{\rt})}\dg^{\otimes2}(\tx;\ttheta_{\rt})\}
=\B^{-1}\Exp\{e^{f(\x;\bbeta_{\rt})}\dg^{\otimes2}(\x;\btheta_{\rt})\}(\B\tp)^{-1}
=\B^{-1}\M(\B\tp)^{-1},
\end{equation*}
and 
\begin{equation*}
\tilde{\bOmega}=\Exp\{e^{2f(\tx;\tbeta_{\rt})}\dg^{\otimes2}(\tx;\ttheta_{\rt})\}
=\B^{-1}\Exp\{e^{2f(\x;\bbeta_{\rt})}\dg^{\otimes2}(\x;\btheta_{\rt})\}(\B\tp)^{-1}
=\B^{-1}\bOmega(\B\tp)^{-1}.
\end{equation*}
Thus, we have that 
\begin{align*}
&\|\tilde{\bOmega}^{1/2}\tilde{\M}^{-1}\dg(\tx;\ttheta_{\rt})\|^2
=\dg\tp(\tx;\ttheta_{\rt})\tilde{\M}^{-1}
\tilde{\bOmega}\tilde{\M}^{-1}\dg(\tx;\ttheta_{\rt})\\
&=\dg\tp(\x;\btheta_{\rt})(\B^{-1})\tp(\B\tp)\M^{-1}\B
\B^{-1}\bOmega(\B\tp)^{-1}\B\tp\M^{-1}\B\B^{-1}\dg(\x;\btheta_{\rt})\\
&=\dg\tp(\x;\btheta_{\rt})\M^{-1}\bOmega\M^{-1}\dg(\x;\btheta_{\rt})
=\|\bOmega^{1/2}\M^{-1}\dg(\x;\btheta_{\rt})\|^2.
\end{align*}
Therefore, the leveraging term is invariant. For the probability term, we know
that is only related to value $g(\x;\btheta_{\rt})=g(\tx;\ttheta_{\rt})$, we know
that it does not change after scaling inactive variables. This complete the
proof.

\subsection{Proof of Theorem~\ref{thm:asym-las}}\label{sec:prfthm2}

\subsubsection{Proof when $p_{N}$ and $s_{N}$ are fixed}\label{sec:lik-fixp}

We start with computing the probability of $y_{i}=1$ in the resultant
subsamples. Given the condition that the data point $(y_{i},\x_{i})$ is
subsampled, e.g., $\delta_{i}=1$, we have that
\begin{align*}
  &\Pr(y_{i}=1|\x_{i},\delta_{i}=1)
  =\frac{\Pr(\delta_{i}=1|y_{i}=1,\x_{i})\Pr(y_{i}=1|\x_{i})}{
    \sum_{\nu=0,1}\Pr(\delta_{i}=1|y_{i}=\nu,\x_{i})\Pr(y_{i}=\nu|\x_{i})}\\
  &=\frac{\frac{e^{g(\x_{i};\btheta)}}{1+e^{g(\x_{i};\btheta)}}}{\frac{e^{g(\x_{i};\btheta)}}{1+e^{g(\x_{i};\btheta)}}
  +\rho\varphi(\x_{i})\frac{1}{1+e^{g(\x_{i};\btheta)}}}
  =\frac{e^{g(\x_{i};\btheta)+l_{i}}}{1+e^{g(\x_{i};\btheta)+l_{i}}},
\end{align*}
where $l_{i}=-\log\{\rho\varphi(\x_{i})\}$. Therefore,
consider the maximum sampled conditional likelihood function with adaptive lasso
penalty:
\begin{align*}
&\Q_{\ruw}^{\htheta_{\rp}}(\btheta)
=-\sumn\delta_i^{\htheta_{\rp}}[y_ig(\x_i;\btheta)-\log\{1+e^{g(\x_i;\btheta)+l_i}\}]+
\lambda_N\sumjp\hat{w}_j|\beta_{(j)}|\\    
&=-\ell_{\ruw}^{\htheta_{\rp}}(\btheta)+\lambda_N\sumjp\hat{w}_j|\beta_{(j)}|,
-\ell_{\ruw}^{\htheta_{\rp}}(\btheta)+\lambda_N\sumjp\frac{1}{|\hat{\beta}_{\rp(j)}|^{\gamma}}|\beta_{(j)}|,
\end{align*}
We have
$\tu_N=a_N(\htheta_{\ruw}^{\htheta_{\rp}}-\btheta_{\rt})$ is 
the minimizer of 
$
\gamma_{\ruw}^{\htheta_{\rp}}(\bu)=Q^{\htheta_{\rp}}_{\ruw}(\btheta_{\rt}+a_N^{-1}\bu)-Q^{\htheta_{\rp}}_{\ruw}(\btheta_{\rt}).
$
\paragraph{Asymptotic normality:} We prove the asymptotic normality part in this
paragraph. By Taylor's expansion, 
\begin{align*}
\gamma_{\ruw}^{\htheta_{\rp}}(\bu)&=-\frac{1}{a_N}\bu\tp\dl_{\ruw}^{\htheta_{\rp}}(\btheta_{\rt})+\frac{1}{2a_N^2}
\sumn\delta_i^{\htheta_{\rp}}\phi_{\pi}^{\htheta_{\rp}}(\x_i;\btheta_{\rt})\{\bu\tp\dg(\x_i;\btheta_{\rt})\}^2-\Delta_{\ruw}^{\htheta_{\rp}}+R_{\ruw}^{\htheta_{\rp}}\\
&\quad+\frac{\lambda_N}{a_N}\sumjp \hat{w}_j a_N\left( \left|
\beta_{\rt(j)}+\frac{u_{(j)}}{a_N} \right|-|\beta_{\rt(j)}| \right).
\end{align*}
First, we consider the limit behavior of the MSCL function. In
\cite{wang2021nonuniform}, the authors proved that under
Assumptions~\ref{asm:a1} and~\ref{asm:a3},   
$%
a_N^{-1}\dl_{\ruw}^{\htheta_{\rp}}(\btheta_{\rt})\cvd(\bm{\Lambda}^{\rp}_{\ruw})^{1/2}\W,
$%
\begin{equation*}
\frac{1}{a_N^2}
\sumn\delta_i^{\htheta_{\rp}}\phi_{\pi}^{\htheta_{\rp}}(\x_i;\btheta_{\rt})\dg^{\otimes2}(\x_i;\btheta_{\rt})
\cvp\bm{\Lambda}^{\rp}_{\ruw},
\end{equation*}
and  
$\Delta_{\ruw}^{\htheta_{\rp}}=\op,\quad R_{\ruw}^{\htheta_{\rp}}=\op$. Thus,
$-\ell_{\ruw}^{\htheta_{\rp}}(\btheta_{\rt})\cvd-\bu\tp(\bm{\Lambda}^{\rp}_{\ruw})^{1/2}\W+0.5\bu\tp\bm{\Lambda}^{\rp}_{\ruw}\bu+\op$.
Next, we consider the limit behavior of the adaptive lasso penalty. Since we
assume $\hat{\bbeta}_{\rp}$ to be a consistent estimator, we know that when
$j\in\cA$, i.e., $\beta_{\rt(j)}\neq0$, 
$\hat{w}_j=|\hat{\beta}_{\rp(j)}|^{-\gamma}\cvp|\beta_{\rt(j)}|^{-\gamma}>0$,
and $a_N \left( \left|\beta_{\rt(j)}+u_{(j)}/a_N\right|-|\beta_{\rt(j)}|
\right)\to\sgn(\beta_{\rt(j)})u_{(j)}$.
Therefore, for $j\in\cA$, we have that
$(\lambda_N/a_N)\hat{w}_ja_N \left( \left| \beta_{\rt(j)}+u_{(j)}/a_N
\right|-|\beta_{\rt(j)}| \right)\cvp0$,
since $\lambda_N/a_N=\lambda_N/\sqrt{Ne^{\alpha_{\rt}}}\to0$.
On the other hand, when $j\in\cAc$, i.e., $\beta_{\rt(j)}=0$, we have that for
$u_{(j)}\neq 0$, 
\begin{equation*}
\frac{\lambda_N}{a_N}\hat{w}_ja_N \left( \left| \beta_{\rt(j)}+\frac{u_{(j)}}{a_N}
\right|-|\beta_{\rt(j)}| \right)
=\frac{\lambda_N}{a_N}\hat{w}_j|u_{(j)}|=\frac{\lambda_N}{a_N|\hat{\beta}_{\rp(j)}|^{\gamma}}
|u_{(j)}|\cvp\infty,
\end{equation*}
since $\lambda_N/(\sqrt{Ne^{\alpha_{\rt}}}|\hat{\beta}_{\rp(j)}|^{\gamma})\cvp\infty$.
Then, we have that $\gamma_{\ruw}^{\htheta_{\rp}}(\bu)\cvd\gamma_{\ruw}(\bu)$,
where  
\begin{equation*}
\gamma_{\ruw}(\bu)=
\begin{cases}
\frac{1}{2}\bu_{(\cA)}\tp\bm{\Lambda}^{\rp}_{\ruw}\bu_{(\cA)}-\bu_{(\cA)}\tp(\bm{\Lambda}^{\rp}_{\ruw})^{1/2}\W_{(\cA)} &
\text{if }u_{(j)}=0, \forall j\notin\cA  \\
\infty   & \text{otherwise}.  \\
\end{cases}
\end{equation*}
Note that the unique minimizer of $\gamma_{\ruw}(\bu)$ is
$((\bm{\Lambda}^{\rp}_{\ruw})^{-1}\W_{(\cA)}\tp,\0)\tp$ if we put all
the indexes of active 
variables in front. Thus, following the results of \cite{geyer1994asymptotics}
and \cite{fu2000asymptotics}, we have the minimizer of
$\gamma_{\ruw}^{\htheta_{\rp}}(\bu)$, $\tu_N$, satisfies that 
$\tu_{N(\cA)}\cvd(\bm{\Lambda}^{\rp}_{\ruw})^{1/2}\W_{(\cA)}$, and
$\tu_{N(\cAc)}\cvd\0$. 
Thus,
$\tu_{N(\cA)}=a_N(\htheta_{\ruw(\cA)}-\btheta_{\rt(\cA)})
\cvd\Nor(\0,(\bm{\Lambda}^{\rp}_{\ruw})^{-1})$.
We know that
\begin{equation*}
\sqrt{N_1}\V_{\ruw(\cA)}^{-1/2}(\bm{\Lambda}^{\rp}_{\ruw})^{-1/2}
=a_{N}\I.
\end{equation*}
Hence, applying Slusky's theorem, we have 
$\sqrt{N_1}\V_{\ruw(\cA)}^{-1/2}(\htheta_{\ruw(\cA)}-\btheta_{\rt(\cA)})
\cvd\Nor(\0,\I)$.
Therefore, we prove the part of asymptotic normality.

\paragraph{Consistency in variable selection}

We prove the consistency in variable selection in this paragraph. From the
result of asymptotic normality, we know
that $\hat{\beta}_{\ruw(j)}\cvp\beta_{\rt(j)}$ for every $j\in 
\cA$ and therefore $\Pr(j\in \hat{\cA}_{\ruw})\to1$. Thus, we only consider
$j'\in\cAc$. When $j'\in \hat{\cA}_{\ruw}$, we know that by K-K-T
optimality conditions, we have 
$\lambda_N\hat{w}_{j'}\sgn(\hat{\beta}_{\ruw(j')})
=\dl_{\ruw}^{\htheta_{\rp}}(\htheta_{\ruw})$,
which means
\begin{align*}
&\frac{\lambda_N\hat{w}_{j'}\sgn(\hat{\beta}_{\ruw(j')})}{a_N}=\frac{\dl_{\ruw}^{\htheta_{\rp}}
(\htheta_{\ruw})}{a_N}
=\frac{\dl_{\ruw}^{\htheta_{\rp}}(\btheta_{\rt})}{a_N}+\frac{a_N
\left\{\dl_{\ruw}^{\htheta_{\rp}}(\htheta_{\ruw}^{\htheta_{\rp}})-\dl_{\ruw}^{\htheta_{\rp}}(\btheta_{\rt})
\right\}}{a_N^2}=I_1+I_2.
\end{align*}
We have known that
$I_1=\dl_{\ruw}^{\htheta_{\rp}}(\btheta_{\rt})/a_N\cvd\W_{\ruw}$. We 
now prove that
proof that $I_2=\Op$. We apply Taylor expansion to the $k$-th element of
$\dl_{\ruw}^{\htheta_{\rp}}(\htheta_{\ruw}^{\htheta_{\rp}})$ and have that
\begin{equation*}
\frac{a_N \left\{
    \dl_{(k)}^{\htheta_{\rp}}(\htheta_{\ruw}^{\htheta_{\rp}})-\dl_{(k)}^{\htheta_{\rp}}(\btheta_{\rt})
  \right\}}{a_N^2}=-\frac{1}{a_N^2}\sumn\delta_i^{\htheta_{\rp}}\phi_{\pi}^{\htheta_{\rp}}(\x_i;\btheta_{\rt})
\dg_{(k)}(\x_i;\btheta_{\rt})\dg\tp(\x_i;\btheta_{\rt})\tu_N+\tilde{\Delta}_{(k)}^{\htheta_{\rp}}+\tilde{R}_{(k)}^{\htheta_{\rp}},
\end{equation*}
where,
$\tu_N=a_N(\htheta_{\ruw}^{\htheta_{\rp}}-\btheta_{\rt})=\Op$,
\begin{equation*}
\tilde{\Delta}_{\ruw(k)}^{\htheta_{\rp}}=\frac{1}{a_N^2}\sumn\delta_i^{\htheta_{\rp}}
\left\{ y_i-p_{\pi}^{\htheta_{\rp}}(\x_i;\btheta_{\rt}) \right\}\sum_{j=1}^d\ddg_{(kj)}(\x_i;\btheta_{\rt})\tilde{u}_{(j)},
\end{equation*}
and 
\begin{align*}
\tilde{R}_{\ruw(k)}^{\htheta_{\rp}}&=-\frac{1}{2a_N^3}\sumn\delta_i^{\htheta_{\rp}}\phi_{\pi}^{\htheta_{\rp}}(\x_i;\batheta_k)
\left\{1-2p_{\pi}^{\htheta_{\rp}}(\x_i;\batheta_k)
\right\}\dg_{(k)}(\x_i;\batheta_k)\tu_N\tp\dg^{\otimes2}(\x_i;\batheta_k)\tu_N\\
&\quad-\frac{2}{2a_N^3}\sumn\delta_i^{\htheta_{\rp}}\phi_{\pi}^{\htheta_{\rp}}(\x_i;\batheta_k)
\left\{\tu_N\tp\frac{\partial\dg_{(k)}(\x_i;\batheta_k)}{\partial\btheta}\right\} \left\{\tu_N\tp\dg(\x_i;\batheta_k) \right\}\\
&\quad-\frac{1}{2a_N^2}\sumn\delta_i^{\htheta_{\rp}}\phi_{\pi}^{\htheta_{\rp}}(\x_i;\batheta_k)\dg_{(k)}(\x_i;\batheta_k)
\left\{ \tu_N\tp\ddg(\x_i;\batheta_k)\tu_N \right\}\\
&\quad+\frac{1}{2a_N^3}\sumn\delta_i^{\htheta_{\rp}}\left\{
y_i-p_{\pi}^{\htheta_{\rp}}(\x_i;\batheta_k)
\right\}\tu_N\tp \frac{\partial^2\dg_{(k)}(\x_i;\batheta_k)}{\partial\btheta^2}\tu_N.
\end{align*}
where $\batheta_k$ is between $\htheta_{\ruw}$ and $\btheta_{\rt}$. First, we
prove 
that $\tilde{R}_{\ruw(k)}$ is $\op$. We have that  
\begin{align*}
  |\tilde{R}_{\ruw(k)}^{\htheta_{\rp}}|
  & \leq\frac{\|\tu_N\|^2}{2a_N^3}\sumn\delta_i^{\htheta_{\rp}}\phi_{\pi}^{\htheta_{\rp}}(\x_i;\batheta_k)
    \left|1-2p_{\pi}^{\htheta_{\rp}}(\x_i;\batheta_k)
    \right|\left|\dg_{(k)}(\x_i;\batheta_k)\right|\left\|\dg(\x_i;\batheta_k)\right\|^2\\
  &\quad+\frac{2\|\tu_N\|^2}{2a_N^3}\sumn\delta_i^{\htheta_{\rp}}\phi_{\pi}^{\htheta_{\rp}}(\x_i;\batheta_k)
    \left\|\frac{\partial\dg_{(k)}(\x_i;\batheta_k)}{\partial\btheta}\right\| \left\|
    \dg(\x_i;\batheta_k) \right\|\\
  &\quad+\frac{\|\tu_N\|^2}{2a_N^3}\sumn\delta_i^{\htheta_{\rp}}\phi_{\pi}^{\htheta_{\rp}}(\x_i;\batheta_k)\left|
    \dg_{(k)}(\x_i;\batheta_k) \right| \left\|\ddg(\x_i;\batheta_k)\right\|\\
  &\quad+\frac{\|\tu_N\|^2}{2a_N^3}\sumn\delta_i^{\htheta_{\rp}}p_{\pi}^{\htheta_{\rp}}(\x_i;\batheta_k)\left\|
    \frac{\partial^2\dg_{(k)}(\x_i;\batheta_k)}{\partial\btheta^2}
    \right\|+\frac{\|\tu_N\|^2}{2a_N^3}\sumn\delta_i^{\htheta_{\rp}}y_i\left\|
    \frac{\partial^2\dg_{(k)}(\x_i;\batheta_k)}{\partial\btheta^2} \right\|\\
  &\leq \frac{\|\tu_N\|^2}{2a_N^3}\sumn\delta_i^{\htheta_{\rp}}
    p_{\pi}^{\htheta_{\rp}}(\x_i;\batheta_k)C(\x_i;\batheta_k)+\frac{\|\tu_N\|^2}{2a_N^3}\sumn\delta_i^{\htheta_{\rp}}y_iB(\x_i)\\ 
  &\leq
    \frac{\|\tu_N\|^2e^{\acute{\alpha}_k-\alpha_{\rt}}e^{\alpha_{\rt}}}{2a_N^3}\sumn\delta_i^{\htheta_{\rp}}e^{f(\x_i;\acute{\bbeta}_k)
    - \log\{ \rho\varphi(\x_i)
    \}}C(\x_i;\batheta_k)+\frac{\|\tu_N\|^2}{2a_N^3}\sumn\delta_i^{\htheta_{\rp}}y_iB(\x_i)\\
  &\leq
    \frac{\|\tu_N\|^2e^{\acute{\alpha}_k-\alpha_{\rt}}}{2Na_N}\sumn\delta_i^{\htheta_{\rp}}e^{f(\x_i;\acute{\bbeta}_k)-\log\{
    \rho\varphi(\x_i)
    \}}C(\x_i;\batheta_k)+\frac{\|\tu_N\|^2}{2a_N^3}\sumn\delta_i^{\htheta_{\rp}}y_iB(\x_i)\\
  &=\frac{\|\tu_N\|^2e^{\acute{\alpha}_k-\alpha_{\rt}}}{2Na_N\rho}\sumn\delta_i^{\htheta_{\rp}}\varphi^{-1}(\x_i)e^{f(\x_i;\acute{\bbeta}_k)}C(\x_i;\batheta_k)
    +\frac{\|\tu_N\|^2}{2a_N^3}\sumn\delta_i^{\htheta_{\rp}}y_iB(\x_i)\\
  &\leq\frac{\|\tu_N\|^2e^{\acute{\alpha}_k-\alpha_{\rt}}}{2Na_N\rho}\sumn\delta_i^{\htheta_{\rp}}\varphi^{-1}(\x_i)B(\x_i)
    +\frac{\|\tu_N\|^2}{2a_N^3}\sumn y_iB(\x_i)
  =\op,
\end{align*}
where
\begin{align*}
C(\x_i;\batheta) &=
\left|\dg_{(k)}(\x_i;\batheta)\right|\left\{\left\|\dg(\x_i;\batheta_k)\right\|^2
                 + \left\|\ddg(\x_i;\batheta)\right\| \right\}\\
&\quad + \left\|\frac{\partial\dg_{(k)}(\x_i;\batheta_k)}{\partial\btheta}\right\|
\left\| \dg(\x_i;\batheta_k) \right\| + \left\|
  \frac{\partial^2\dg_{(k)}(\x_i;\batheta_k)}{\partial\btheta^2} \right\|.
\end{align*}
Therefore, we proved that $\tilde{R}_{\ruw(k)}=\op$. Next, we prove that
$\tilde{\Delta}_{\ruw(k)}=\op$. We know that $\Exp\left[ a_N^{-2}\sumn\delta_i^{\htheta_{\rp}}\left\{
    y_i-p_{\pi}^{\htheta_{\rp}}(\x_i;\btheta_{\rt})\right\}\ddg(\x_i;\btheta_{\rt})
  \Bigm|\htheta_{\rp} \right]=\0$. We also have
that for the every element of
$a_N^{-2}\sumn\left\{y_i-p(\x_i;\btheta_{\rt})\right\}\ddg(\x_i;\btheta_{\rt})$, we
have
\begin{align*}
&\Var\left[ a_N^{-2}\sumn\delta_i^{\htheta_{\rp}}\left\{
y_i-p_{\pi}^{\htheta_{\rp}}(\x_i;\btheta_{\rt})
\right\}\ddg_{(jl)}(\x_i;\btheta_{\rt})\Bigm|\htheta_{\rp} \right]\\
&\leq \frac{1}{a_N^4}\sumn\Exp \left\{
\delta_i^{\htheta_{\rp}}p_{\pi}^{\htheta_{\rp}}(\x_i;\btheta_{\rt})\ddg_{(jl)}^2(\x_i;\btheta_{\rt})\Bigm|\htheta_{\rp}\right\}
\leq \frac{1}{a_N^2}\Exp[e^{f(\x;\bbeta_{\rt})}\|\ddg(\x;\btheta_{\rt})\|^2]\to0.
\end{align*}
Thus, due to Chebyshev's inequality, we know that
$\tilde{\Delta}_{\ruw}^{\htheta_{\rp}}=\op$. Since we know that
$\frac{1}{a_N^2}
    \sumn\delta_i^{\htheta_{\rp}}\phi_{\pi}^{\htheta_{\rp}}(\x_i;\btheta_{\rt})\dg^{\otimes2}(\x_i;\btheta_{\rt})=\Op$. Hence,
we have that 
$\dl_{\ruw}(\htheta_{\ruw})/a_N=\Op$.
Note that we also have
$(\lambda_N\hat{w}_{j'})/a_N=(\lambda_N/a_N)|\hat{\beta}_{\rp(j')}|^{-\gamma}\cvp\infty$.
Therefore, 
\begin{align*}
&\Pr(j'\in \cA_N)\leq\Pr\left\{\lambda_N\hat{w}_{j'}\sgn(\hat{\beta}_{\ruw(j')})=\dl_{\ruw}^{\htheta_p}(\htheta_{\ruw})\right\}\\
&=\Pr\left\{\frac{\lambda_N\hat{w}_{j'}\sgn(\hat{\beta}_{\ruw(j')})}{a_N}
=\frac{\dl_{\ruw}^{\htheta_p}(\htheta_{\ruw})}{a_N}\right\}\to0.
\end{align*}
Thus, we prove the part of consistency of variable selection. 

\subsubsection{Proof when $p_{N}$ and $s_{N}$ diverge with $N$}\label{sec:lik-divergep}

In this section, we prove Theorem~\ref{thm:asym-las} when $p_{N}$ and $s_{N}$
diverge with $N$. Let $\eta_{i(j)}=\delta_i^{\htheta_{\rp}}\left\{
y_i-p_{\pi}^{\htheta_{\rp}}(\x_i;\btheta_{\rt})\right\}\dg_{(j)}(\x_i;\btheta_{\rt})$,
we have that the $j$-th element of $\dl_{\ruw}(\btheta_{\rt})$,
$\dl_{\ruw(j)}(\btheta_{\rt})=\sumn\eta_{i(j)}$, $\Exp(\eta_{i(j)})=0$, and
\begin{align*}
  &\Var(\eta_{i(j)}|\x_{i})
  =[\{1-p_{\pi}^{\htheta_{\rp}}(\x_{i};\btheta_{\rt})\}^{2}
  +\phi_{\pi}^{\htheta_{\rp}}(\x_{i};\btheta_{\rt})]
  p_{\pi}^{\htheta_{\rp}}(\x_{i};\btheta_{\rt})
  \{\dg_{(j)}(\x_{i};\btheta_{\rt})\}^{2}\\
  &\le e^{\alpha_{\rt}}e^{f(\x_{i};\bbeta_{\rt})}
  \{\dg_{(j)}(\x_{i};\btheta_{\rt})\}^{2}
  \le e^{\alpha_{\rt}}B^{2}(\x_{i}).
\end{align*}
Therefore, since $|y_{i}-p(\x_{i};\btheta_{\rt})|\le 1$, we have that
$|\eta_{i,(j)}|\le\max_{i=1,\ldots,N}B(\x_{i})=R_{N}$. 
Thus, using the same deduction, we have that
\begin{align}\label{eq:las-infnorm}
  \Pr\left(\left\|\frac{1}{a_{N}^{2}}\sumn\eeta_{i(\cA)}\right\|_{\infty}
  \gtrsim\sqrt{\frac{\log s_{N}}{a_{N}^{2}}}\right)
  <\frac{1}{s_{N}},
  \Pr\left(\left\|\frac{1}{a_{N}^{2}}\sumn\eeta_{i(\cAc)}\right\|_{\infty}
  \gtrsim\sqrt{\frac{\log p_{N}}{a_{N}^{2}}}\right)
  <\frac{1}{p_{N}},
\end{align}
and
$\Pr\left(\left\|a_{N}^{-2}\sumn\eeta_{i(\cA)}\right\|
>\sqrt{a_{N}^{-2}s_{N}}\right)<5^{-s_{N}}$.
Next, we similarly need to show that 
$\bH_{\ruw}=a_{N}^{-2}\sumn\delta_{i}^{\htheta_{\rp}}
\phi_{\pi}^{\htheta_{\rp}}(\x_{i};\btheta_{\rt})\dg_{(\cA)}^{\otimes2}(\x_{i};\btheta_{\rt})$
is positive-definite with a high probability.
Using the same arguments in~\cite{wang2021nonuniform} of (S.28), we know that
$\Var(\bH_{\ruw,i}|\x_{i})\le e^{\alpha_{\rt}}B(\x_{i})$, where
$\bH_{\ruw}=(1/a_{N}^{2})\sumn\bH_{\ruw,i}$, and $\|\bH_{\ruw,i}\|\le R_{N}$.
Again, applying Bernstein's inequality, we have that
$\|\bH_{\ruw}-\Exp(\bH_{\ruw})\|\lesssim\sqrt{s_{N}/a_{N}^{2}}+s_{N}/a_{N}^{2}$
and
$\|\bH_{\ruw}-\Exp(\bH_{\ruw})\|_{\infty}\lesssim\sqrt{s_{N}^{2}/a_{N}^{2}}+s_{N}^{3/2}/a_{N}^{2}$,
which implies that $\|\bH_{\ruw}\|$ and $\|\bH_{\ruw}\|_{\infty}$ are bounded
from both below and above with high probabilities.
Now, due to the K-K-T conditions, we know that
$\htheta=(\htheta_{\ruw(\cA)}\tp,\0\tp)\tp$ is the unique adaptive lasso
estimator if and only if
\begin{align*}
  \dl_{\ruw}(\hat{\theta}_{(j)})=\lambda_{N}w_{j}\sgn(\hat{\theta}_{(j)}),
  \text{ for }j\in\cA, \text{ and }
  |\dl_{\ruw}(\hat{\theta}_{(j)})|\le\lambda_{N}w_{j}, \text{ for }j\notin\cA.
\end{align*}
To prove the K-K-T conditions, it is sufficient to show that
\begin{align}\label{eq:hd-kkt2}
\begin{split}
  &\sgn(\theta_{\rt(j)})(\theta_{\rt(j)}-\hat{\theta}_{\ruw(j)})<|\theta_{\rt(j)}|, \forall j\in\cA, 
  \text{ and } \\
  &\frac{1}{a_{N}}\left|\sum_{i=1}^{N}
  \delta_{i}^{\htheta_{\rp}}\{y_{i}-p_{\pi}^{\htheta_{\rt}}(\x_{i(\cA)};\htheta_{\ruw(\cA)})\}
  \dg_{(j)}(\x_{i(\cA)};\htheta_{\ruw(\cA)})\right|\le\frac{\lambda_{N}w_{j}}{a_{N}}, \forall j\notin\cA,
\end{split}
\end{align}
where $\htheta_{\ruw(\cA)}$ is the solution of 
\begin{align*}
  \sumn\delta_{i}^{\htheta_{\rp}}
  \{y_{i}-p_{\pi}^{\htheta_{\rt}}(\x_{i(\cA)};\htheta_{\ruw(\cA)})\}
  \dg_{(j)}(\x_{i(\cA)};\htheta_{\ruw(\cA)})
  =\lambda_{N}w_{j}\sgn(\hat{\theta}_{\rt,(j)}), \forall j\in\cA.
\end{align*}
Now, we define
$\htheta_{(\cA)}^{*}=\arg\min_{\btheta_{(\cA)}}\gamma_{\ruw}(\btheta_{(\cA)})$, where 
\begin{align*}
  \gamma_{\ruw}(\btheta_{(\cA)})
  =\frac{1}{2}\btheta_{(\cA)}\tp(a_{N}^{2}\bH_{\ruw})\btheta_{(\cA)}
  -\sumn\eeta_{i(\cA)}\tp(\btheta_{(\cA)}-\btheta_{\rt(\cA)})
  +\lambda_{N}\sum_{j\in\cA}w_{j}\sgn(\theta_{\rt(j)})\theta_{(j)}.
\end{align*}
We also have that
\begin{align}\label{eq:lik-theta-star}
  \btheta_{(\cA)}^{*}-\btheta_{\rt(\cA)}=(a_{N}^{2}\bH_{\ruw})^{-1}
  \left\{\sumn\eeta_{i(\cA)}-\lambda_{N}\xi\right\},
\end{align}
and $\|\btheta_{(\cA)}^{*}-\btheta_{\rt(\cA)}\|\lesssim\sqrt{s_{N}/a_{N}^{2}}
+(\lambda_{N}\sqrt{s_{N}})/(a_{N}^{2}b_{N}^{\gamma})$, 
$\|\btheta_{(\cA)}^{*}-\btheta_{\rt(\cA)}\|_{\infty}\lesssim\sqrt{\log s_{N}/a_{N}^{2}}
+(\lambda_{N})/(a_{N}^{2}b_{N}^{\gamma})$ as we have shown in the proof of Theorem~\ref{thm:asym-ipw}.
By the same arguments as~\cite{huang2008iterated} and Lemma 2
in~\cite{hjort2011asymptotics}, we can show that
\begin{align}\label{eq:lik-theta}
  \|\htheta_{\ruw(\cA)}-\btheta_{(\cA)}^{*}\|^{2}
  =o_{P}\left(\frac{s_{N}}{a_{N}^{2}}+\frac{\lambda_{N}^{2}s_{N}}{a_{N}^{4}b_{N}^{2\gamma}}\right),
\end{align}
which implies that 
$\|\htheta_{\ruw(\cA)}-\btheta_{\rt(\cA)}\|
\lesssim\sqrt{s_{N}/a_{N}^{2}}+(\lambda_{N}\sqrt{s_{N}})/(a_{N}^{2}b_{N})$.
Thus, the first condition in~\eqref{eq:hd-kkt2} satisfies under
Assumption~\ref{asm:a6} and the fact that
$\|\htheta_{\ruw(\cA)}-\btheta_{\rt(\cA)}\|_{\infty}
\lesssim\sqrt{\log s_{N}^{2}/a_{N}^{2}}+\lambda_{N}/(a_{N}^{2}b_{N}^{\gamma})$. Then, we turn to the second condition
in~\eqref{eq:hd-kkt2}. From the proof of Theorem~\ref{thm:asym-las}, using the
same notations, we have that 
\begin{align*}
  &\frac{a_{N}\left\{\dl_{(j)}(\htheta_{\ruw(\cA)})-\dl_{(j)}(\btheta_{\rt(\cA)})\right\}}{a_{N}^{2}}
  +\frac{1}{a_{N}^{2}}\sumn\delta_{i}^{\htheta_{\rp}}\phi_{\pi}^{\htheta_{\rp}}(\x_{i};\btheta_{\rt(\cA)})
  \dg\tp(\x_{i};\btheta_{\rt(\cA)})\hat{\bu}_{N(\cA)}\dg_{(j)}(\x_{i};\btheta_{\rt(\cA)})\\
  &=\tilde{\Delta}_{\ruw(j)}^{\htheta_{\rp}}+\tilde{R}_{\ruw(j)}^{\htheta_{\rp}}.
\end{align*}
From the proof of Theorem~\ref{thm:asym-las}, we know that
\begin{align*}
  |\tilde{\Delta}_{\ruw(j)}^{\htheta_{\rp}}|\le\frac{1}{a_{N}^{2}}
  \left|\sumn\delta_{i}^{\htheta_{\rp}}
  \{y_{i}-p_{\pi}^{\htheta_{\rp}}(\x_{i};\btheta_{\rt(\cA)})\}B(\x_{i})\right|\|\hat{\bu}_{N(\cA)}\|.
\end{align*}
Using the same arguments as we bound $\sumn\eta_{i(j)}$, we have that with
probability at least $1-a_{N}^{-1}$,
$|\tilde{\Delta}_{\ruw(j)}^{\htheta_{\rp}}|\lesssim\sqrt{\log a_{N}/a_{N}^{2}}\|\hat{\bu}_{N(\cA)}\|$.
From the proof of Theorem~\ref{thm:asym-las}, we also have that
\begin{align*}
  |\tilde{R}_{\ruw(j)}^{\htheta_{\rp}}|\lesssim\frac{\|\hat{\bu}_{N(\cA)}\|^{2}}{Na_{N}\rho}
  \sumn\delta_{i}^{\htheta_{\rp}}\varphi^{-1}(\x_{i})B(\x_{i})
  +\frac{\|\hat{\bu}_{N(\cA)}\|^{2}}{a_{N}^{3}}
  \sumn y_{i}B(\x_{i}).
\end{align*}
Note that since 
\begin{align*}
  &\Var\left\{\delta_{i}^{\htheta_{\rp}}\varphi^{-1}(\x_{i})B(\x_{i})\Big|\x_{i}\right\}
  \le\frac{B^{2}(\x_{i})}{\varphi^{2}(\x_{i})}\Exp(\delta_{i}^{\htheta_{\rp}}|\x_{i}) \\
  &=\frac{B^{2}(\x_{i})}{\varphi^{2}(\x_{i})}
  \left[p(\x_{i};\btheta_{\rt})\varphi^{-1}(\x_{i})B(\x_{i})
  +\rho\{1-p(\x_{i};\btheta_{\rt})\}B(\x_{i})\right]
  \lesssim e^{\alpha_{\rt}}\frac{B^{3}(\x_{i})}{\varphi^{3}(\x_{i})}
  +\rho B(\x_{i})
\end{align*}
Therefore, we also have that with probability at least $1-a_{N}^{-1}$,
\begin{align*}
  |\tilde{R}_{\ruw(j)}|\lesssim\frac{\|\hat{\bu}_{N(\cA)}\|^{2}\log a_{N}}{a_{N}}
  \left\{\frac{1}{N}\sumn B(\x_{i})
  +\sqrt{\sumn\frac{\max\{R_{N}^{3}(\x_{i}),R_{N}(\x_{i})\}}{N^{2}}}\right\}.
\end{align*}
Thus, we only need to consider the dominating term
\begin{align*}
  &\left|\frac{1}{a_{N}^{2}}\sumn\delta_{i}^{\htheta_{\rt}}
  \phi_{\pi}^{\htheta_{\rp}}(\x_{i};\btheta_{\rt(\cA)})
  \dg\tp(\x_{i};\btheta_{\rt(\cA)})\hat{\bu}_{N(\cA)}\dg_{(j)}(\x_{i};\btheta_{\rt(\cA)})\right|\\
  &\le\left|\frac{1}{a_{N}^{2}}\sumn\delta_{i}^{\htheta_{\rt}}
  \phi_{\pi}^{\htheta_{\rp}}(\x_{i};\btheta_{\rt(\cA)})B(\x_{i})\right|\|\hat{\bu}_{N(\cA)}\|_{\infty}
  \lesssim\|\hat{\bu}_{N(\cA)}\|_{\infty},
\end{align*}
where the last inequality can be obtained using arguments as we
bound $\bH_{\ruw}$. Therefore,
\begin{align*}
  &\left|\frac{\dl_{(j)}(\htheta_{\ruw(\cA)})}{a_{N}}\right|
  \lesssim\left|\frac{\dl_{(j)}(\btheta_{\rt(\cA)})}{a_{N}}\right|
  +\|\htheta_{\ruw(\cA)}-\btheta_{\rt(\cA)}\|_{\infty}\\
  &\quad+\sqrt{\frac{\log a_{N}}{a_{N}^{2}}}
  \left(\sqrt{\frac{s_{N}}{a_{N}^{2}}}+\frac{\lambda_{N}\sqrt{s_{N}}}{a_{N}^{2}b_{N}^{\gamma}}\right)
  +\frac{\log a_{N}}{a_{N}}\left(\frac{s_{N}}{a_{N}^{2}}+\frac{\lambda_{N}^{2}s_{N}}{a_{N}^{4}b_{N}^{2\gamma}}\right)\\
  &\lesssim\sqrt{\log p_{N}}+\sqrt{\frac{\log s_{N}}{a_{N}^{2}}}
  +\frac{\log a_{N}}{a_{N}}
  \left(\sqrt{\frac{s_{N}}{a_{N}^{2}}}+\frac{\lambda_{N}\sqrt{s_{N}}}{a_{N}^{2}b_{N}^{\gamma}}\right)
  \lesssim \sqrt{\log p_{N}}.
\end{align*} 
Then, we have that $\forall j\notin\cA$,
\begin{align*}
  \frac{a_{N}}{\lambda_{N}w_{j}}\left|\frac{\dl_{(j)}(\htheta_{\ruw(\cA)})}{a_{N}}\right|
  \lesssim\frac{a_{N}\sqrt{\log p_{N}}}{\lambda_{N}w_{j}}
  +\frac{\sqrt{\log s_{N}}}{\lambda_{N}w_{j}}
  +\frac{\log a_{N}}{a_{N}}
  \left(\frac{\sqrt{s_{N}}}{\lambda_{N}w_{j}}
  +\frac{\sqrt{s_{N}}}{a_{N}w_{j}b_{N}^{\gamma}}\right).
\end{align*}
Then, combining~\eqref{eq:lik-theta-star},~\eqref{eq:lik-theta}, and 
$(\sqrt{s_{N}}\lambda_{N})/\sqrt{N_{1}}=\oo$, we have that
$\forall\bm{e}_{N}$,$\|\bm{e}_{N}\|=1$,
\begin{align*}
  \sqrt{N_{1}}\bm{e}_{N}\tp\V_{\ruw(\cA)}^{-1/2}(\htheta_{\ruw(\cA)}-\btheta_{\rt,(\cA)})
  =\bm{e}_{N}\tp\frac{1}{a_{N}}(\bLambda_{\ruw(\cA)}^{\rp})^{1/2}
  \bH_{\ruw}^{-1}\sumn\eeta_{i,(\cA)}+\op\cvd\Nor(0,1),
\end{align*}
due to Lindeberg-Feller's central limit theorem indicated in the previous proof of
Theorem~\ref{thm:asym-las}.

\subsection{Proof of Theorem~\ref{thm:effi-comp}}\label{sec:effi-comp}

Let $h=1+c\{\varphi(\x)\}^{-1}e^{f(\x;\bbeta_{\rt})}$, $\bm{v}=\sqrt{e^{f(\x;\bbeta_{\rt})}}\dg_{(\cA)}(\x;\ttheta)$, $\bm{f}=h^{\frac{1}{2}}\bm{v}$,
and $\bm{g}=h^{-\frac{1}{2}}\bm{v}$. We have that
$
\Exp(\bm{g}\bm{f}\tp)=\Exp(\bm{f}\bm{g}\tp)=\Exp(\bm{v}\bm{v}\tp)
=\Exp\left\{e^{f(\x;\bbeta_\rt)}\dg_{(\cA)}^{\otimes2}(\x;\btheta_{\rt})\right\}=\M_{(\cA)},
$
$
\Exp(\bm{f}\bm{f}\tp)=\Exp(h\bm{v}\bm{v}\tp)
=\Exp\left[\left\{1+ce^{f(\z;\bbeta_\rt)}/\varphi(\x)\right\}
e^{f(\x;\bbeta_\rt)}\dg_{(\cA)}^{\otimes2}(\x;\btheta_{\rt})\right]=\M_{\rw(\cA)},
$
$
\Exp(\bm{g}\bm{g}\tp)=\Exp(h^{-1}\bm{v}\bm{v}\tp)
=\Exp\left[\{e^{f(\x;\bbeta_\rt)}\dg_{(\cA)}^{\otimes2}(\x;\btheta_{\rt})\}/\{1+
c\varphi^{-1}(\x)e^{f(\x;\bbeta_{\rt})}\}
\right]=\bLambda_{\ruw(\cA)}.
$
Now, applying the matrix form of Cauchy-Schwartz's inequality (see \cite{tripathi1999matrix}), we have that 
\begin{align*}
\bLambda_{\ruw(\cA)}=\Exp(\bm{g}\bm{g}\tp)
&\geq\Exp(\bm{g}\bm{f}\tp)\{\Exp(\bm{f}\bm{f}\tp\}^{-1}\Exp(\bm{f}\bm{g}\tp)\\
&=\M_{(\cA)}\{\M_{\rw(\cA)}\}^{-1}\M_{(\cA)}=\Exp\left\{e^{f(\x;\bbeta_{\rt})}\right\}
\{\V_{\rw(\cA)}\}^{-1}.
\end{align*}
The equality holds if and only if $h$ is a constant, which implies $c=0$. Therefore, simple algebra 
shows that
$\V_{\ruw(\cA)}=\Exp\left\{e^{f(\x;\bbeta_{\rt})}\right\}\{\bLambda_{\ruw(\cA)}\}^{-1}
\le\V_{\rw(\cA)}$.

\subsection{Proof of Theorem~\ref{thm:cr}}\label{sec:cr-bound}

For any asymptotic unbiased estimator such that 
$\htheta_{U(\cA)}=\U_{(\cA)}(\btheta_{\rt};\Ds)+\op/\sqrt{N_1}$,
where
\begin{equation*}
\Exp\{\U_{(\cA)}(\btheta_{\rt};\Ds)|\X\}=\btheta_{\rt(\cA)}, \text{ and }
\Exp\left\{\frac{\partial\U_{(\cA)}(\btheta_{\rt};\Ds)}{\partial\btheta_{(\cA)}\tp}\Big|\X\right\}=\0,
\end{equation*} 
we first prove that $Cov\{\U_{(\cA)}(\btheta_{\rt};\Ds),\dl_{\ruw(\cA)}(\btheta_\rt)\}=\I$. 
Note that S.4 in \cite{wang2021nonuniform} tells us that for sampled data,
the joint density of the responses given the covariates is 
$\exp\{\ell_{\ruw}(\btheta)+\sumn\delta_iy_il_l\}$. Therefore, we have that
\begin{align*}
&\frac{\partial}{\partial\btheta_{(\cA)}\tp}
\Exp\{\U_{(\cA)}(\btheta_{\rt};\Ds)|\X\}
=\frac{\partial}{\partial\btheta_{(\cA)}\tp}\int
\U_{(\cA)}(\btheta_{\rt};\Ds)e^{\ell_{\ruw}(\btheta_{\rt})+\sumn\delta_iy_il_l}\dd\y\\
&=\int\frac{\partial\U_{(\cA)}(\btheta_{\rt};\Ds)}{\partial\btheta_{(\cA)}\tp}
e^{\ell_{\ruw}(\btheta_{\rt})+\sumn\delta_iy_il_l}\dd\y
+\int\U_{(\cA)}(\btheta_{\rt};\Ds)\dl_{\ruw(\cA)}\tp(\btheta_{\rt})
e^{\ell_{\ruw}(\btheta_{\rt})+\sumn\delta_iy_il_l}\dd\y\\
&=\Exp\left\{\frac{\partial\U_{(\cA)}(\btheta_{\rt};\Ds)}{\partial\btheta_{(\cA)}\tp}\Big|\X\right\}
+\Exp\{\U_{(\cA)}(\btheta_{\rt};\Ds)\dl_{\ruw(\cA)}\tp(\btheta_{\rt})|\X\}
=\Exp\{\U_{(\cA)}(\btheta_{\rt};\Ds)\dl_{\ruw(\cA)}\tp(\btheta_{\rt})|\X\}.
\end{align*}
Since $\Exp\{\U_{(\cA)}(\btheta_{\rt};\Ds)|\X\}=\btheta_{\rt(\cA)}$, we have 
that $\Exp\{\U_{(\cA)}(\btheta_{\rt};\Ds)\dl_{\ruw(\cA)}\tp(\btheta_{\rt})|\X\}
=\I$, which implies that $\Exp\{\U_{(\cA)}
(\btheta_{\rt};\Ds)\dl_{\ruw(\cA)}\tp(\btheta_{\rt})\}=\I$. Therefore
\begin{align*}
&Cov\{\U_{(\cA)}
(\btheta_{\rt};\Ds),\dl_{\ruw(\cA)}(\btheta_{\rt})\}\\
&=\Exp\{\U_{(\cA)}(\btheta_{\rt};\Ds)\dl_{\ruw(\cA)}\tp(\btheta_{\rt})\}
-\Exp\{\U_{(\cA)}(\btheta_{\rt};\Ds)\}
\Exp\{\dl_{\ruw(\cA)}(\btheta_{\rt})\}\tp=\I,
\end{align*}
where the last inequality is due to (S.19) in~\cite{wang2021nonuniform}, such 
that $\Exp\{\dl_{\ruw(\cA)}(\btheta_{\rt})\}=\0$. Now, we apply
the matrix form of Cauchy-Schwartz's inequality and obtain that
\begin{align*}
&\Var\{\U_{(\cA)}(\btheta_{\rt};\Ds)\}\\
&\geq Cov\{\U_{(\cA)}(\btheta_{\rt};\Ds),\dl_{\ruw(\cA)}(\btheta_{\rt})\}
\Var\{\dl_{\ruw(\cA)}(\btheta_{\rt})\}^{-1}
Cov\{\U_{(\cA)}(\btheta_{\rt};\Ds),\dl_{\ruw(\cA)}(\btheta_{\rt})\}\\
&=\frac{1}{a_N^2}\{\bLambda_{\ruw(\cA)}^{-1}+\op\}
=\frac{1}{N_1}\V_{\ruw(\cA)}+\frac{\op}{N_1}.
\end{align*}
Taking limits on both sides, we obtain that $\V_{U(\cA)}\geq\V_{\ruw(\cA)}$.

\section{Details and complexity of practical algorithm}
\label{sec:dtlalg}
In this section, we provide more details of the practical algorithm in
Section~\ref{sec:theory}.

\subsection{Two-step algorithm}\label{sec:algo}

We take a pilot sample by uniform sampling
with the sampling rate $\rho_1=N_{\rp}/2N_1^{*}$ for the ones and
$\rho_0=N_{\rp}/2N_0^{*}$ for the zeros.
Denote a pilot sample of actual sample size $N_{\rp}^{*}$ as
$\{(\x_i^{\rp},y_i^{\rp})\}_{i=1}^{N_{\rp}^{*}}$, the pilot estimate
of $\btheta$ as $\htheta_{\rp}$, and the pilot estimate of the active set as
$\hat{\cA}_{\rp}=\{j:\hat{\beta}_{\rp(j)}\neq0\}$. We propose the following
moment estimators of $\M_{(\cA)}$ and $\bOmega_{(\cA)}$:   
\begin{equation}
\label{eq:estM}
\hat{\M}_{\hA}^{\rp}=\frac{1}{N_{\rp}}\sum_{i=1}^{N_{\rp}^{*}}\frac{e^{f(\x^{\rp T}_i\hat{\bbeta}_{\rp})}
\dg_{\hA}^{\otimes2}(\x^{\rp}_i;\htheta_{\rp})}{\rho_0+y_i^{\rp}(\rho_1-\rho_0)},
\end{equation}
\begin{equation}
\label{eq:estSig}
\hat{\bOmega}_{\hA}^{\rp}=\frac{1}{N_{\rp}}\sum_{i=1}^{N_{\rp}^{*}}\frac{e^{2f(\x^{\rp T}_i\hat{\bbeta}_{\rp})}
\dg_{\hA}^{\otimes2}(\x^{\rp}_i;\htheta_{\rp})}{\rho_0+y_i^{\rp}(\rho_1-\rho_0)}.
\end{equation}
We also use the following moment estimator to estimate the denominator of~\eqref{eq:optA}:
\begin{equation}
\label{eq:estdenom}
\frac{1}{N_{\rp}}\sum_{i=1}^{N_{\rp}^{*}}\frac{\omega_i^{\mmse}}{\rho_0+y_i^{\rp}(\rho_1-\rho_0)},
\end{equation}
where
$\omega_i^{\mmse}=p(\x_i^{\rp};\htheta_{\rp})\|(\hat{\M}_{\hA}^{\rp})^{-1}
\dg_{\hA}(\x^{\rp}_i;\htheta_{\rp})\|$. 
If using~\eqref{eq:optL} or
\eqref{eq:optpr}, we use 
\begin{align*}
&\omega_i^{\mvc}=p(\x_i^{\rp};\htheta_{\rp})\|
\dg_{\hA}(\x^{\rp}_i;\htheta_{\rp})\|,\text{ or }\\
&
\omega_i^{\mpr}=p(\x_i^{\rp};\htheta_{\rp})\|(\hat{\bOmega}_{\hA}^{\rp})^{1/2}(\hat{\M}_{\hA}^{\rp})^{-1}
\dg_{\hA}(\x^{\rp}_i;\htheta_{\rp})\|,
\end{align*}
respectively, instead of $\omega_i^{\mmse}$ in \eqref{eq:estdenom}. Now, we
present the proposed two-step procedure in Algorithm~\ref{alg:adplasS} with more
details than Algorithm~\ref{alg:adplas} in Section~\ref{sec:theory}.
\begin{algorithm}[H]%
  \caption{Subsampling adaptive lasso algorithm}
  \label{alg:adplasS}
  \begin{algorithmic}[1]
    \STATE
    \begin{itemize}
    \item Take a pilot sample $\{(\x_i^{\rp},y_i^{\rp})\}_{i=1}^{N_{\rp}^{*}}$ of
      expected sample size $N_{\rp}$ using
      $\{\pi(y_i)=\rho_0+y_i(\rho_1-\rho_0)\}_{i=1}^N$ and obtain a pilot
      estimator
    \begin{equation}
    \label{eq:pl-las}
    \htheta_{\rp}:=\mathop{\arg\max}_{\btheta}\left\{\sum_{i=1}^{N_{\rp}^{*}}[y_i^{\rp}g(\x_i^{\rp};\btheta) -
    \log\{1+e^{g(\x_i^{\rp};\btheta)+l}\}] - \lambda_{\rp}\sumjp|\beta_{(j)}|\right\},
    \end{equation}
    where $N_{\rp}^{*}$ is the actual pilot sample
    size and $l=\log(N_0^{*}/N_1^{*})$. We call this first stage
    screening.
  \item Calculate approximate optimal sampling probabilities
      $\{\hat{\pi}(\x_i,y_i)=y_i+(1-y_i)\rho\hat\varphi(\x_i)\}_{i=1}^N$ by
      replacing $\hat\varphi(\x_i)$ with
      $\varphi^{\radp}_{\mmse}(\x_i;\htheta_{\rp})$,
      $\varphi^{\radp}_{\mvc}(\x_i;\htheta_{\rp})$, or
      $\varphi^{\radp}_{\mpr}(\x_i;\htheta_{\rp})$, based on \eqref{eq:optA},
      \eqref{eq:optL}, or \eqref{eq:optpr}, respectively. The denominator of
      \eqref{eq:optA} is estimated using \eqref{eq:estdenom}, and we replace
      $\omega_i^{\mmse}$ with $\omega_i^{\mvc}$ or $\omega_i^{\mpr}$ for the
      denominator of \eqref{eq:optL} or \eqref{eq:optpr}, respectively.  If
      using $\pi^{\radp}_{\mmse}(\x)$ or $\pi^{\radp}_{\mpr}(\x)$, estimate
      $\M_{(\cA)}$ and $\bOmega_{(\cA)}$ using the moment estimators in
      \eqref{eq:estM} and \eqref{eq:estSig}, respectively.
    \end{itemize}
    \STATE Use Algorithm~\ref{alg:poi}  with the estimated optimal sampling
      probabilities to obtain a subsample
      $\{(\x_i^{\rs},y_i^{\rs})\}_{i=1}^{N_{\rs}^{*}}$ and compute the adaptive
      lasso estimator: 
    \begin{equation*}
    \htheta_{\ruw}^{\radp}:=\mathop{\arg\max}_{\btheta}\left\{\sum_{i=1}^{N_{\rs}^{*}}[y_i^{\rs}g(\x_i^{\rs};\btheta) -
    \log\{1+e^{g(\x_i^{\rs};\btheta)+l_i^{\rs}}\}] -
    \lambda_N\sum_{j\in\hat{\cA}_{\rp}}\frac{|\beta_{(j)}|}{|\hat{\beta}_{\rp(j)}|^{\gamma}}\right\},
    \end{equation*}
    where $N_{\rs}^{*}$ is the actual subsample size, based on the smaller model
    obtained from the first stage screening. We call this step the second stage
    screening.
  \end{algorithmic}
\end{algorithm}
\begin{remark}
Our algorithm naturally integrates the MSCL function with the adaptive
lasso penalty. It can also be implemented when $p>N$ as long as the
dimension of selected variables is smaller than $N$ in the first-stage screening. If the model is sparse and the data are massive, this is usually possible in practice. Screening algorithms such as sure independence
screening \citep{fan2008sure} can also be used for the first stage screening to guarantee that the dimension of second-stage screening is smaller than the subsample size. Furthermore, the first stage screening can help to speed up the computation, as shown by the analysis of computational complexity in Section~\ref{sec:comp}.
\end{remark}

We consider a coordinate descent method to calculate the estimators defined in
Algorithm~\ref{alg:adplasS}. (see \cite{friedman2007pathwise},
\cite{friedman2010Regularization} and \cite{yuan2012improved}). In each cycle, we
need to find an optimal direction $\bd$ at a starting point $\ttheta$. We
consider the quadratic approximation of
$Q_{\ruw}^{\htheta_{\rp}}(\ttheta+\bd)-Q_{\ruw}^{\htheta_{\rp}}(\ttheta)$, which is 
\begin{align*}
&Q_{\ruw}^{\htheta_{\rp}}(\ttheta+\bd)-Q_{\ruw}^{\htheta_{\rp}}(\ttheta)\\
&=\sumn\delta_i[-y_ig(\x_i;\ttheta+\bd)+
\log\{1+e^{g(\x_i;\ttheta+\bd)+l_i}\}]+\lambda_N\sumjp\hat{w}_j|\beta_{(j)}+d_{(j)}|\\
&\quad-\sumn\delta_i[-y_ig(\x_i;\ttheta) -
\log\{1+e^{g(\x_i;\ttheta)+l_i}\}] + \lambda_N\sumjp\hat{w}_j|\beta_{(j)}|\\
&\approx\dl_{\ruw}^{\htheta_{\rp}}(\ttheta)^T\bd+\frac{1}{2}\bd^T\ddl_{\ruw}^{\htheta_{\rp}}(\ttheta)\bd +
\lambda_N\sumjp\left\{\hat{w}_j|\tilde{\beta}_{(j)}+d_{(j)}| - \hat{w}_j|\tilde{\beta}_{(j)}|\right\},
\end{align*}
$\dl_{\ruw}^{\htheta_{\rp}}(\ttheta)=-\sumn\delta_i^{\htheta_{\rp}}\left\{
  y_i-p_{\pi}^{\htheta_{\rp}}(\x_i,\ttheta)\right\}\dg(\x_i;\ttheta)$,
$\ddl_{\ruw}^{\htheta_{\rp}}(\ttheta)=\sumn\delta_i^{\htheta_{\rp}}\phi_{\pi}^{\htheta_{\rp}}(\x_i,\ttheta)\dg^{\otimes2}
    (\x_i;\ttheta)$.
Thus, using coordinate descent to obtain the optimal direction, the quadratic
approximation for the $j$-th element is given as
\begin{align*}
Q_{\ruw}^{\htheta_{\rp}}(\bd+z\bm{e}_j)-Q_{\ruw}^{\htheta_{\rp}}(\bd)
&=\dl_{\ruw(j)}(\ttheta)z+\left\{\ddl_{\ruw}(\ttheta)\bd\right\}_{(j)}z+\frac{1}{2}\ddl_{\ruw(jj)}(\ttheta)z^2\\
&\quad+\lambda_N\hat{w}_j|\tilde{\beta}_{(j)}+d_{(j)}+z| - \lambda_N\hat{w}_j|\tilde{\beta}_{(j)}+d_{(j)}|.
\end{align*}
Then, we have the value of $z$ that minimizes
$Q_{\ruw}^{\htheta_{\rp}}(\bd+z\bm{e}_j)-Q_{\ruw}^{\htheta_{\rp}}(\bd)$ is 
\begin{equation*}
z^{**}=
\begin{cases}
\frac{\dl^{\htheta_{\rp}}_{\ruw(j)}(\ttheta)+\left\{\ddl_{\ruw}^{\htheta_{\rp}}(\ttheta)\bd\right\}_{(j)}+\lambda
  \hat{w}_j}{-\ddl_{\ruw(jj)}^{\htheta_{\rp}}(\ttheta)} & \text{if }\tilde{\beta}_{(j)}+d_{(j)}+z\geq 0  \\
\frac{\dl_{\ruw(j)}^{\htheta_{\rp}}(\ttheta)+\left\{\ddl_{\ruw}^{\htheta_{\rp}}(\ttheta)\bd\right\}_{(j)}-\lambda
  \hat{w}_j}{-\ddl_{\ruw(jj)}^{\htheta_{\rp}}(\ttheta)} & \text{if }\tilde{\beta}_{(j)}+d_{(j)}+z\leq 0  \\
-\tilde{\beta}_{(j)}-d_{(j)} & \text{otherwise},
\end{cases}
\end{equation*}
which is the same as
$z^{**}=\max\left\{ z_1,-\tilde{\beta}_{(j)}-d_{(j)}
\right\} - \max \left\{ -z_2,
\tilde{\beta}_{(j)}+d_{(j)}\right\} + \tilde{\beta}_{(j)}+d_{(j)}$,
where 
\begin{equation*}
z_1=\frac{\dl_{\ruw(j)}^{\htheta_{\rp}}(\ttheta)+\left\{\ddl_{\ruw}^{\htheta_{\rp}}(\ttheta)\bd\right\}_{(j)}+\lambda
  \hat{w}_j}{-\ddl_{\ruw(jj)}^{\htheta_{\rp}}(\ttheta)},
z_2=\frac{\dl_{\ruw(j)}^{\htheta_{\rp}}(\ttheta)+\left\{\ddl_{\ruw}^{\htheta_{\rp}}(\ttheta)\bd\right\}_{(j)}-\lambda
  \hat{w}_j}{-\ddl_{\ruw(jj)}^{\htheta_{\rp}}(\ttheta)}.
\end{equation*}
For $g(\x;\btheta)=\alpha+f(\x\tp\bbeta)$, we know that 
$\ddl_{\ruw}^{\htheta_{\rp}}(\ttheta)=\sumn\delta_i^{\htheta_{\rp}}\phi_{\pi}^{\htheta_{\rp}}(\x_i,\ttheta)\dg^{\otimes2}
    (\x_i;\ttheta)=\G\tp\bm{\Phi}\G$,
where $\G=(\G_{(1)},\G_{(2)},\ldots,\G_{(N)})\tp$, $\G_{i}=(1,\df(\x_i\tp\tilde{\bbeta})\x_i\tp)\tp$,
and
$\bm{\Phi}=diag\{\delta_i^{\htheta_{\rp}}\phi_{\pi}^{\htheta_{\rp}}(\x_i,\ttheta)
\}$. Thus, we have 
$\left\{\ddl_{\ruw}^{\htheta_{\rp}}(\ttheta)\bd\right\}_{(j)}=\left( \G\tp\bm{\Phi}\G\bd
\right)\tp \bm{e}_j = (\G\bd)\tp\bm{\Phi}(\G\bm{e}_j) =
(\G\bd)\tp\bm{\Phi}(\G\bm{e}_j) = (\G\bd)\tp\bm{\Phi}\G_{(j)}$.
Therefore, we can store $\G\bd$ and keep updating $\G\bd$ with 
$\G(\bd+z\bm{e}_j)=\G\bd+z\G\bm{e}_j=\G\bd+\G_{(j)}z$.
Thus, we do not need to obtain the full matrix
$\ddl_{\ruw}(\ttheta)=\G\tp\bm{\Phi}\G$. We only need to calculate the diagonal
elements: $\ddl_{\ruw(jj)}(\ttheta)=\G_{(j)}\tp\bm{\Phi}\G_{(j)},j=1,...,p+1$
and
$\left\{\dl_{\ruw}(\ttheta)\bd\right\}_{(j)}=(\G\bd)\tp\bm{\Phi}\G_{(j)},
j=1,...,p+1$. From the analysis above, we can notice that the computational
complexity of one cycle calculating optimal direction $\bd$ is
$O(\zeta_{\inn}Np)$, where $\zeta_{\inn}$ denotes the number of inner iteration.

\subsection{Computational complexity}\label{sec:complexity}

Considering the
form of 
$g(\x;\btheta)=\alpha+f(\x\tp\bbeta)$, the computational complexity for coordinate descent with data
  of size $N$ and dimension $p$ is $O(\zeta_{\inn}Np)$ per inner-cycle where
  $\zeta_{\inn}$ represents the number of inner iterations
 (detailed derivations of
this complexity is presented Section~\ref{sec:dtlalg}). Thus, the computational
complexity of full data lasso is $O(\zeta_{\oi}Np)$, where
$\zeta_{\oi}=\zeta_{\out}\zeta_{\inn}$ and $\zeta_{\out}$ is the number of
outer iterations. The computational complexity of the full data adaptive lasso is
$O(\zeta_{\rp}^{\rmle}Np^2+\zeta_{\oi}Np)$ with the MLE as the pilot estimator, and
$O(\zeta_{\rp,\oi}^{\rlas}Np+\zeta_{\oi}Nq)$ with lasso as the pilot estimator,
where $\zeta_{\rp}^{\mle}$ and $\zeta_{\rp,\oi}^{\rlas}$ are the iteration
numbers in the two pilot estimators, respectively. The coordinate descent algorithm often requires a large $\zeta_{\oi}$ or $\zeta_{\rp,\oi}^{\rlas}$ while Newton's
algorithm requires a small $\zeta_{\rp}^{\rmle}$, so it is often the case that
$\zeta_{\rp}^{\rmle}p<\zeta_{\oi}$. Therefore, the time complexity of the
adaptive lasso is $O(\zeta_{\oi} Np)$, which is the same as the full data lasso
estimator.

Now, we analyze the time complexity of Algorithm~\ref{alg:adplasS}. 
We start with the computational complexity of the optimal probabilities, for
which the main computational cost
is to approximate $\|\M_{(\cA)}^{-1}\dg_{(\cA)}(\x_i;\btheta)\|$ or
$\|\bOmega_{(\cA)}^{1/2}\M_{(\cA)}^{-1}\dg_{(\cA)}(\x_i;\btheta)\|$, respectively, for
$i=1,...,N$.
Since
$\dg_{\hA}(\x_i;\btheta)=(1,\df(\x_i\tp\bbeta)\x_{i\hA}\tp)\tp$,
the complexity of calculating
$\dg_{\hA}(\x_i;\btheta)$'s is $O(Nq)$, and the computational
complexity of $\hat{\M}_{\hA}^{\rp}$ or $\hat{\bOmega}_{\hA}^{\rp}$ is
$O\left\{N_{\rp}(q+1)^2\right\}=O(N_{\rp}q^2)$. Taking the inverse
$(\hat{\M}_{\hA}^{\rp})^{-1}$ and finding the square root
$(\hat{\bOmega}_{\hA}^{\rp})^{1/2}$ both take $O(q^3)$ time.
Thus, the computational
complexity of calculating $\|(\hat{\M}_{\hA}^{\rp})^{-1}\dg_{\hA}(\x_i;\btheta)\|$'s or
$\|(\hat{\bOmega}_{\hA}^{\rp})^{1/2}(\hat{\M}_{\hA}^{\rp})^{-1}\dg_{\hA}(\x_i;\btheta)\|$'s is
$O(Nq+N_{\rp}q^2+q^3+Nq^2)=O(Nq^2)$. Therefore, the complexity of approximating the optimal
probabilities in \eqref{eq:optA} or \eqref{eq:optpr} is $O(Nq^2)$.
The computational complexity of approximating the optimal probabilities
in~\eqref{eq:optL} is only $O(Nq)$, because there is no need to
compute $\hat{\M}_{\hA}^{\rp}$ or $\hat{\bOmega}_{\hA}^{\rp}$. Next, we analyze
the complexity of parameter estimation. The
average subsample size with the sampling rate $\rho$ is on average 
\begin{equation*}
\Exp(N_1^{*})+\rho\{N-\Exp(N_1^{*})\}=
N\Exp\{f(\x;\btheta_{\rt})\}\{(1-\rho)e^{\alpha_{\rt}}+\rho+\oo\}
=O\{N(e^{\alpha_{\rt}}+\rho)\}.
\end{equation*}
Thus, 
the computational complexity of the two-step algorithm is
$O\{\zeta_{\rp,\oi}^{\rlas}N_{\rp}p+Nq^2+\zeta_{\oi}N(e^{\alpha_{\rt}}+\rho)q\}$
using optimal probabilities in \eqref{eq:optA} 
or \eqref{eq:optpr}, and it is $O\{\zeta_{\rp,\oi}^{\rlas}N_{\rp}p+Nq+\zeta_{\oi}
  N(e^{\alpha_{\rt}}+\rho)q\}$ using the optimal probabilities
  in~\eqref{eq:optL}. 

\section{Details of simulation settings} 
\label{sec:apsimu}

\subsection{Simulation details for Section~\ref{sec:intro}}
\label{sec:simuintro}

We first present the detailed settings of the simulations in
Section~\ref{sec:intro}, where we 
illustrate the scale-dependent issues of optimal subsampling probabilities. Our
simulation based on logistic regression models with the true parameter
$\bbeta_{\rt}$ to be 6-dimensional vectors and covariates $\x\sim
lognormal(\0,\bSigma)$ with the $(i,j)$-th element of $\bSigma$ is given as
$\bSigma_{ij}=0.5^{|i-j|},1\le i,j\le 6$. We consider two cases of parameters: 
\begin{enumerate}[(a)]
\item{\textbf{Non-sparse parameter}:}
  $\bbeta_{\rt}=(-1,-1,-0.01,-0.01,-0.01,-0.01)$ and $\alpha_{\rt}=-4$.
\item{\textbf{Sparse parameter}:} $\bbeta_{\rt}=(-1,0,0,0,0,0)$ and
  $\alpha_{\rt}=-5$.
\end{enumerate}
We generate full data of size $N=500000$ according to the above logistic models.
To investigate the effects of scale transformation, we multiply
the $\x_{(6)}$ with $s$ ($s=0.01,0.1,1,10,100$) and
divide $\bbeta_{\rt(6)}$ with the same $s$ to remain $\x\tp\bbeta_{\rt}$ to be
the same and thus the logistics regression model does not change. We obtain
subsamples with optimal subsampling probabilities described in
\cite{wang2021nonuniform} with transformed $\x$ under each $s$ and calculate the
resultant subsampling estimators. We set the nominal pilot sample size to
$N_{\rp}=800$ and nominal subsample size $N_{\rs}=1000$ (see details in
\cite{wang2021nonuniform}). We repeat the experiment for 500 times under each
scale and compute the mean prediction error.

\subsection{Simulation details for Section~\ref{sec:numeric}}
\label{sec:simunum}

For the estimation procedures in the second step of our two-step algorithm, we
choose $\gamma=1$, which means that the weights in the adaptive lasso
penalty are $\hat{w}_j=1/|\hat{\beta}_{\rp(j)}|, 1\leq j\leq q$, with $q$ being the number of
selected variables in the first stage screening. Furthermore, we consider
uniform sampling, the full data lasso and the full data adaptive lasso as baselines for
comparison. For the uniform sampling method, we use a similar two-step algorithm as
presented in Algorithm~\ref{alg:adplas} but set the sampling function in the
second step as $\varphi(\x)=1$, which means the sampling
probabilities are a constant $\rho$. We use lasso to implement the first stage screening and adaptive lasso
with $\gamma=1$ to implement the second stage screening for a fair comparison. For
full data lasso, we apply the lasso algorithm to the full
data. For the full data
adaptive lasso, we use the full data MLE estimator as the pilot estimator to
construct the weights and then apply the adaptive lasso algorithm to the
full data set. %

\section{Additional numerical results}\label{sec:addsimu}

\subsection{Results on variable selection omitted in Section~\ref{sec:numeric}}
\label{sec:vari-select-resultsS}

\begin{table}[ht]
\caption{Mean number of selected variables in Case A and Case B.}
\label{tb:mnumvS}
\centering
\begin{tabular}{lccccc}\hline
       & \multicolumn{5}{c}{Case A (five active variables)}              \\\hline
$\rho$ & first-stage & Uni        & A-OS       & L-OS       & P-OS       \\\hline
0.0025 & 14.87(0.30) & 4.99(0.02) & 5.02(0.02) & 5.03(0.02) & 5.01(0.02) \\
 0.005 & 14.68(0.28) & 5.05(0.02) & 5.07(0.02) & 5.06(0.02) & 5.08(0.02) \\
0.0075 & 14.20(0.27) & 5.09(0.03) & 5.08(0.02) & 5.08(0.02) & 5.09(0.03) \\
  0.01 & 14.46(0.30) & 5.13(0.02) & 5.13(0.02) & 5.13(0.02) & 5.12(0.03) \\
\hline
\end{tabular}
\begin{tabular}{lccccc}
        &   \multicolumn{5}{c}{Case B (four active variables)}            \\\hline
$\rho$  & first-stage & Uni        & A-OS       & L-OS       & P-OS       \\\hline
 0.0025 & 16.81(0.33) & 3.91(0.02) & 4.01(0.03) & 4.00(0.02) & 4.01(0.03) \\
  0.005 & 17.88(0.34) & 4.04(0.03) & 4.10(0.03) & 4.08(0.04) & 4.07(0.03) \\
 0.0075 & 17.19(0.33) & 4.06(0.02) & 4.13(0.03) & 4.10(0.03) & 4.12(0.03) \\
   0.01 & 17.51(0.34) & 4.07(0.03) & 4.08(0.03) & 4.08(0.03) & 4.08(0.03) \\
\hline
\end{tabular}
\end{table}

\begin{table}[ht]
\caption{Rates of excluding active variables (false negative rate) in Case A and Case B.}
\label{tb:covermodelS}
\centering
\begin{tabular}{lcccc}\hline
       & \multicolumn{4}{c}{Case A}                               \\\hline
$\rho$ & Uni          & A-OS         & L-OS         & P-OS         \\\hline
0.0025 & 0.094(0.013) & 0.076(0.012) & 0.068(0.011) & 0.070(0.012) \\
 0.005 & 0.052(0.010) & 0.048(0.010) & 0.050(0.010) & 0.044(0.010) \\
0.0075 & 0.054(0.010) & 0.052(0.010) & 0.052(0.010) & 0.052(0.010) \\
  0.01 & 0.040(0.009) & 0.036(0.008) & 0.036(0.008) & 0.036(0.008) \\
\hline
\end{tabular}
\begin{tabular}{lcccc}
       & \multicolumn{4}{c}{Case B}                               \\\hline
$\rho$ & Uni          & A-OS         & L-OS         & P-OS         \\\hline
0.0025 & 0.108(0.014) & 0.074(0.012) & 0.074(0.012) & 0.070(0.012) \\
 0.005 & 0.046(0.009) & 0.034(0.008) & 0.034(0.008) & 0.036(0.008) \\
0.0075 & 0.062(0.011) & 0.046(0.009) & 0.046(0.009) & 0.044(0.009) \\
  0.01 & 0.056(0.010) & 0.050(0.010) & 0.052(0.010) & 0.050(0.010) \\
\hline
\end{tabular}
\end{table}

\begin{table}[ht]
\caption{Rates of selecting the true model.}
\label{tb:truemodelS}
\centering
\begin{tabular}{lcccc}\hline
       & \multicolumn{4}{c}{Case A}                               \\\hline
$\rho$ & Uni          & A-OS         & L-OS         & P-OS         \\\hline
0.0025 & 0.856(0.016) & 0.870(0.015) & 0.868(0.015) & 0.878(0.015) \\
 0.005 & 0.884(0.014) & 0.872(0.015) & 0.880(0.015) & 0.872(0.015) \\
0.0075 & 0.858(0.016) & 0.868(0.015) & 0.868(0.015) & 0.868(0.015) \\
  0.01 & 0.854(0.016) & 0.866(0.015) & 0.862(0.015) & 0.868(0.015) \\
\hline
\end{tabular}
\begin{tabular}{lcccc}
       & \multicolumn{4}{c}{Case B}                               \\\hline
$\rho$ & Uni          & A-OS         & L-OS         & P-OS         \\\hline
0.0025 & 0.862(0.015) & 0.872(0.015) & 0.864(0.015) & 0.876(0.015) \\
 0.005 & 0.898(0.014) & 0.888(0.014) & 0.898(0.014) & 0.902(0.013) \\
0.0075 & 0.848(0.016) & 0.850(0.016) & 0.854(0.016) & 0.848(0.016) \\
  0.01 & 0.874(0.015) & 0.874(0.015) & 0.876(0.015) & 0.872(0.015) \\
\hline
\end{tabular}
\begin{tabular}{lcccc}
       & \multicolumn{4}{c}{Case C}                               \\\hline
$\rho$ & Uni          & A-OS         & L-OS         & P-OS         \\\hline
0.0025 & 0.824(0.017) & 0.874(0.015) & 0.880(0.015) & 0.884(0.014) \\
 0.005 & 0.880(0.015) & 0.884(0.014) & 0.884(0.014) & 0.882(0.014) \\
0.0075 & 0.908(0.013) & 0.916(0.012) & 0.910(0.013) & 0.914(0.013) \\
  0.01 & 0.908(0.013) & 0.910(0.013) & 0.908(0.013) & 0.910(0.013) \\
\hline
\end{tabular}
\end{table}

It is seen in Table~\ref{tb:truemodelS} that, no
subsampling method dominates others.
Table~\ref{tb:covermodel} and Table~\ref{tb:covermodelS} show that uniform sampling has
higher rates of excluding active variables than optimal subsampling
procedures. Although uniform sampling may have
a higher rate of selecting the true model in some cases, given that it is
more likely to exclude important variables, optimal sampling may be preferable
in practice.

\subsection{Comparison with standardization}\label{sec:stand}

Another approach to avoid
scale-dependency is to standardize the data. We compare the proposed
scale-independent optimal probabilities with the approach of data
standardization here. For the data standardization approach, we standardize the data, calculate the
optimal probabilities, and then implement subsampled adaptive lasso algorithm.
We used the same pilot estimation methods for fair comparisons. 

We first compare the eMSE and eMSPE in Figure~\ref{fig:mses} and
Figure~\ref{fig:mspes}, respectively. We use sP-OS to denote the approach with
data standardization and use P-OS to denote the approach without data
standardization.

\begin{figure}[ht]
  \centering 
  \begin{subfigure}{0.235\textwidth}
    \includegraphics[width=\textwidth]{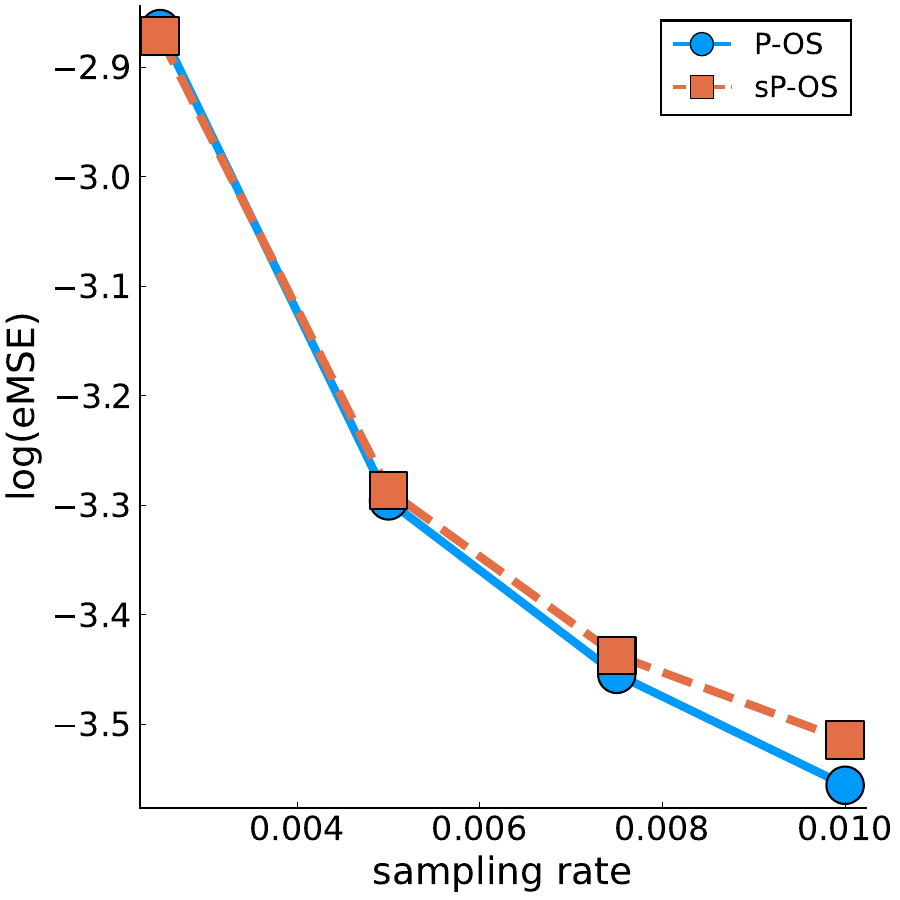}
    \caption{Case A}
  \end{subfigure}
  \begin{subfigure}{0.235\textwidth}
    \includegraphics[width=\textwidth]{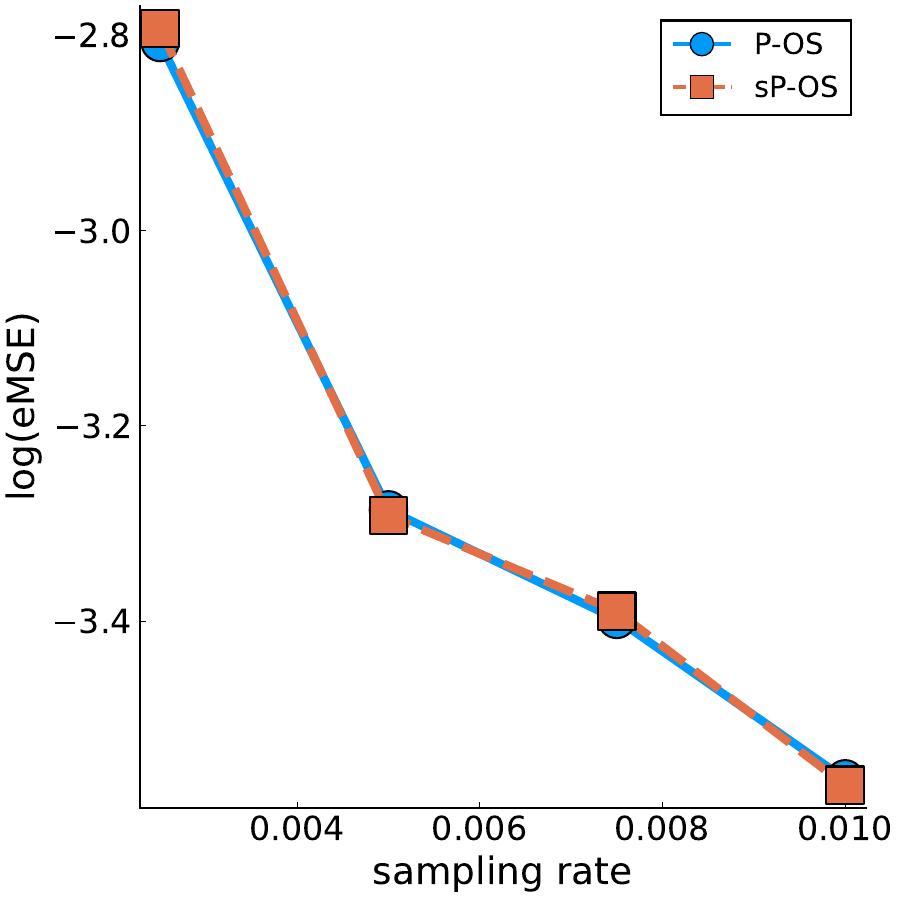}
    \caption{Case B}
  \end{subfigure}
  \begin{subfigure}{0.235\textwidth}
    \includegraphics[width=\textwidth]{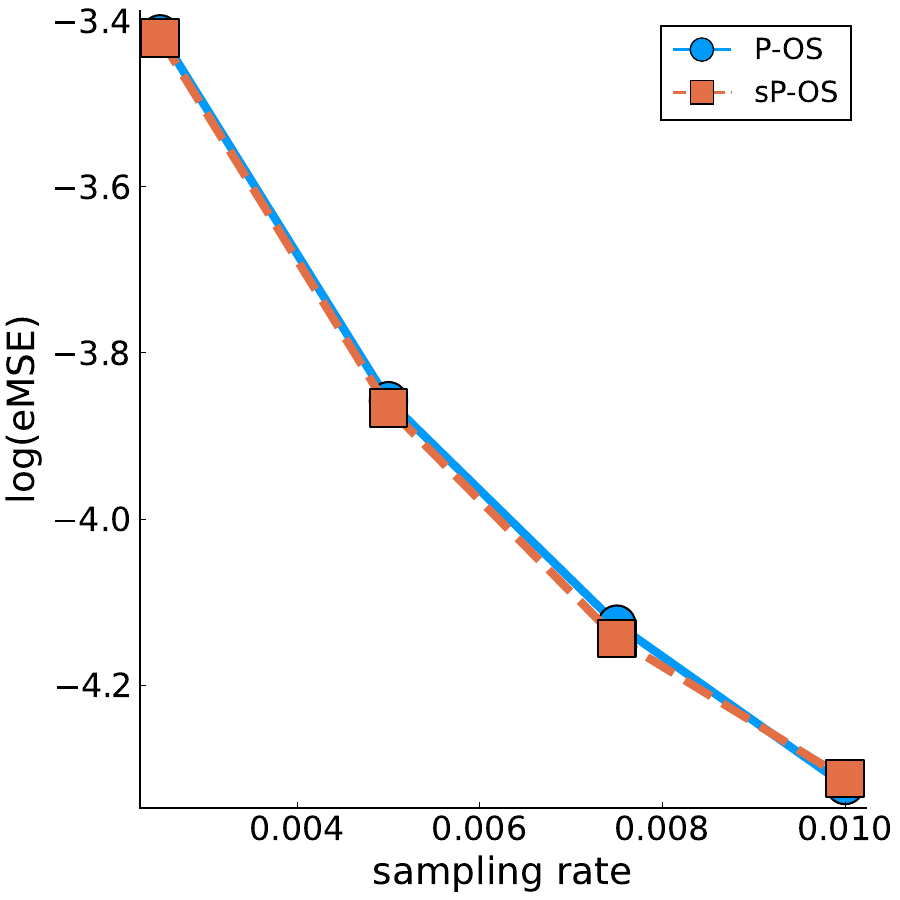}
    \caption{Case C}
  \end{subfigure}
  \caption{Empirical median squared error of estimated probability
    for different parameters with different sampling rates. The same pilot
    sample size is $N_{\rp}=500$.}
  \label{fig:mses}
\end{figure}

In Figure~\ref{fig:mses}, we notice that the performances of P-OS and sP-OS are
similar, this is also true for eMSPE.
However, standardization may decrease the rate of selecting the true model. We
present results of variable selection in Table~\ref{tb:scale-ocv} and
Table~\ref{tb:scale-cov}.

\begin{figure}[ht]
  \centering 
  \begin{subfigure}{0.235\textwidth}
    \includegraphics[width=\textwidth]{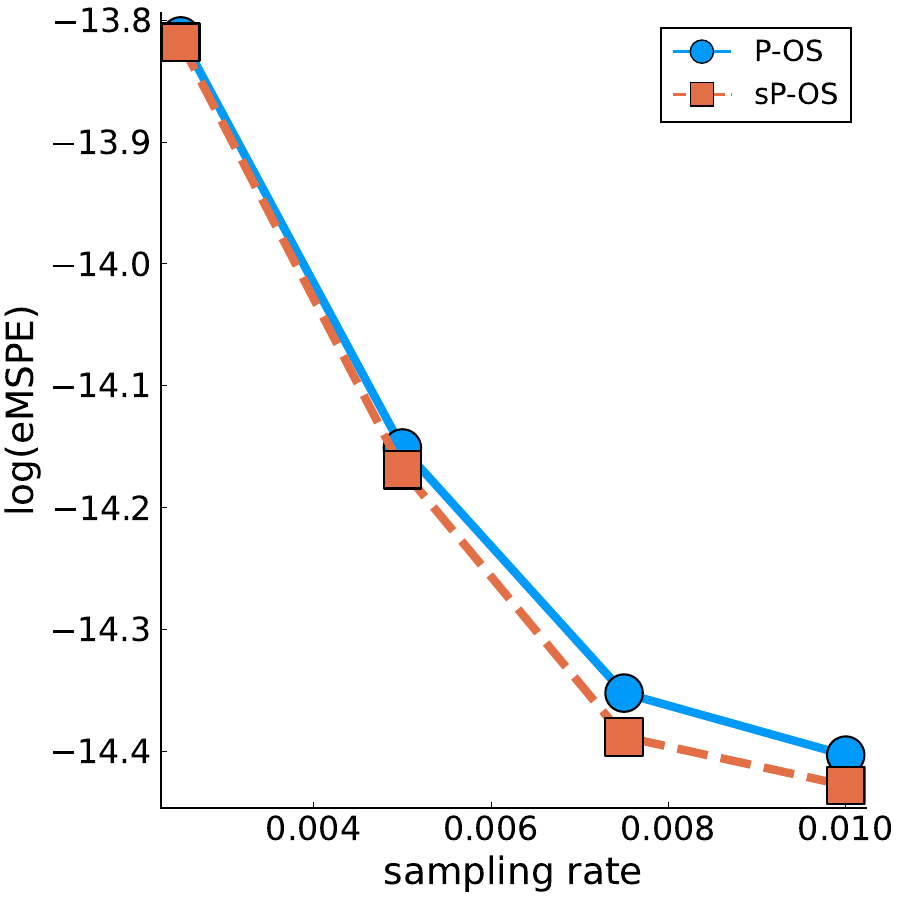}
    \caption{Case A}
  \end{subfigure}
  \begin{subfigure}{0.235\textwidth}
    \includegraphics[width=\textwidth]{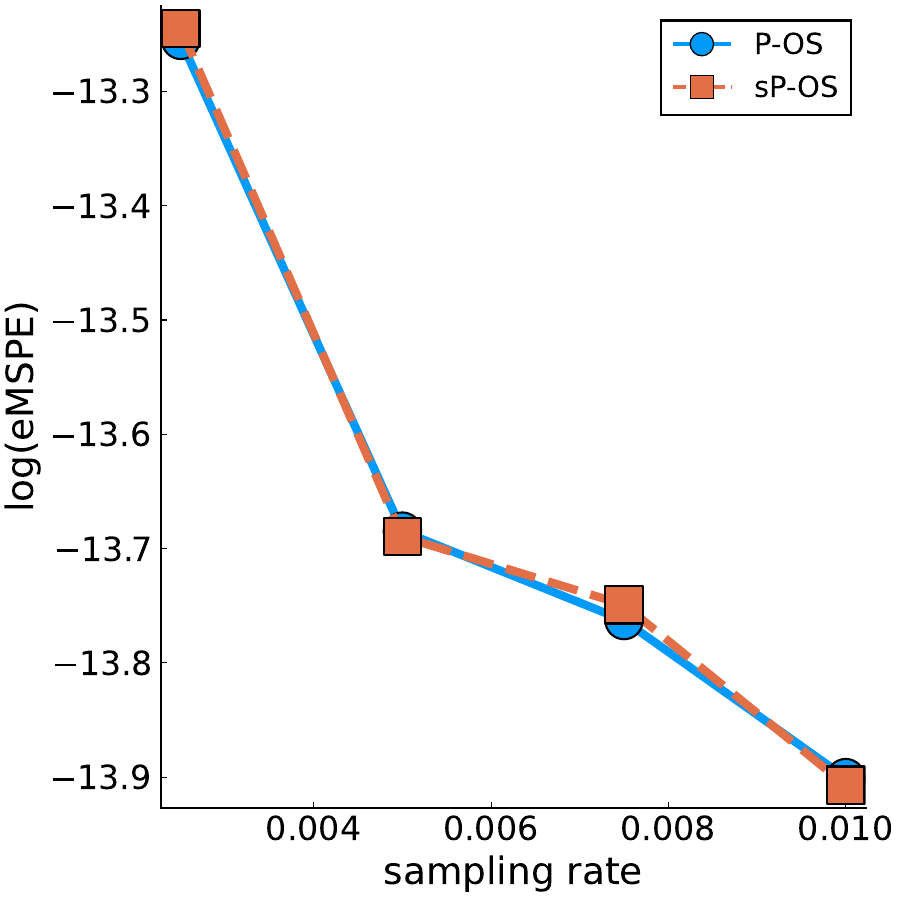}
    \caption{Case B}
  \end{subfigure}
    \begin{subfigure}{0.235\textwidth}
    \includegraphics[width=\textwidth]{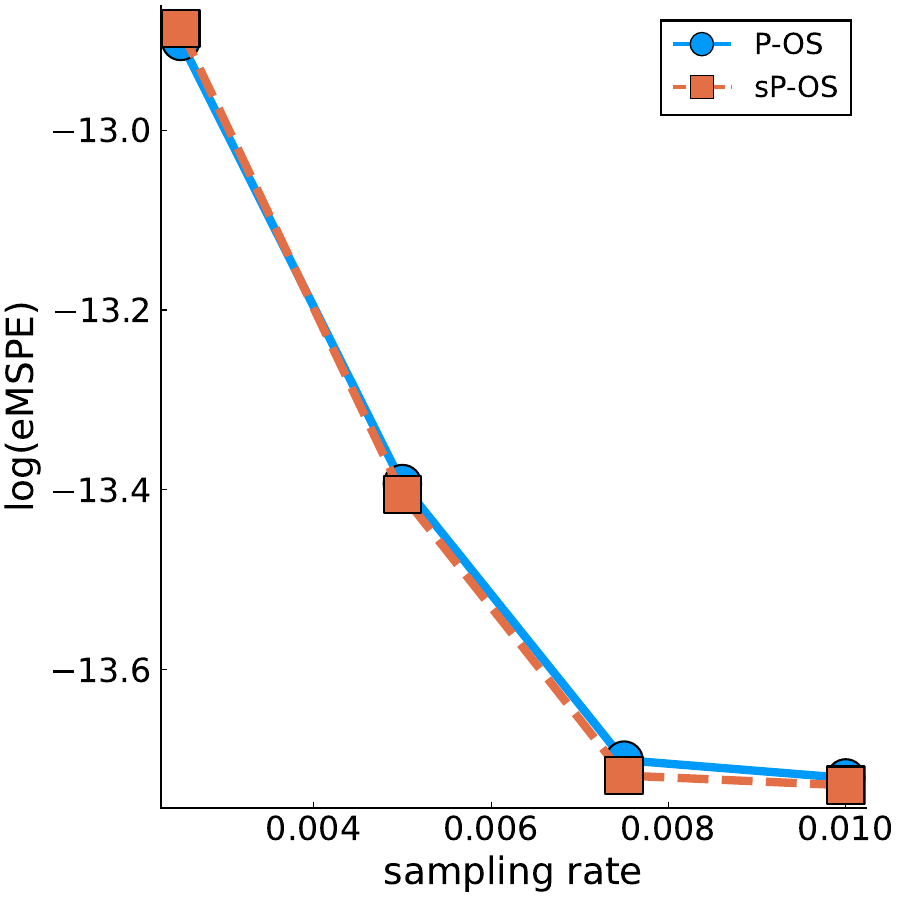}
    \caption{Case C}
  \end{subfigure}
  \caption{Empirical median squared error of estimated probability for
    different parameters with different sampling rates. The same pilot
    sample size is $N_{\rp}=500$.}
  \label{fig:mspes}
\end{figure}

We notice in Table~\ref{tb:scale-ocv} and Table~\ref{tb:scale-cov} that
the rates of selecting true models by $\hbeta_{\mathrm{P-OS}}^{\radp}$ is higher
than $\hbeta_{\mathrm{sP-OS}}^{\radp}$ without much increase on the rates of
excluding active variables. Therefore, although standardization is an approach
to solve the scale-dependency issues, it may decrease the rates of selecting
true models in practice.
\begin{table}[H]
\caption{Rates of selecting true models.}
\label{tb:scale-ocv}
\centering
\begin{tabular}{lcccccccccc}\hline
       & \multicolumn{3}{c}{A} & \multicolumn{3}{c}{B} & \multicolumn{3}{c}{C}\\\hline
$\rho$ & sUni  & P-OS  & sP-OS & sUni  & P-OS  & sP-OS & sUni  & P-OS  & sP-OS\\\hline %
0.0025 & 0.848 & 0.878 & 0.862 & 0.850 & 0.876 & 0.862 & 0.828 & 0.884 & 0.880\\ %
0.005  & 0.878 & 0.872 & 0.850 & 0.890 & 0.902 & 0.890 & 0.882 & 0.882 & 0.880\\ %
0.0075 & 0.850 & 0.868 & 0.846 & 0.850 & 0.848 & 0.842 & 0.906 & 0.914 & 0.910\\ %
0.01   & 0.840 & 0.868 & 0.844 & 0.868 & 0.872 & 0.862 & 0.900 & 0.910 & 0.904\\ %
\hline                 
\end{tabular}
\end{table}

\begin{table}[H]
\caption{Rates of excluding active variables (false negative rate).}
\label{tb:scale-cov}
\centering
\begin{tabular}{lcccccccccc}\hline
       & \multicolumn{3}{c}{A} & \multicolumn{3}{c}{B} & \multicolumn{3}{c}{C}\\\hline
$\rho$ & sUni  & P-OS  & sP-OS & sUni  & P-OS  & sP-OS & sUni  & P-OS  & sP-OS\\\hline %
0.0025 & 0.088 & 0.070 & 0.068 & 0.106 & 0.070 & 0.070 & 0.164 & 0.084 & 0.084\\ %
0.005  & 0.056 & 0.044 & 0.044 & 0.046 & 0.036 & 0.036 & 0.100 & 0.066 & 0.066\\ %
0.0075 & 0.052 & 0.052 & 0.052 & 0.062 & 0.044 & 0.044 & 0.062 & 0.046 & 0.046\\ %
0.01   & 0.040 & 0.036 & 0.036 & 0.056 & 0.050 & 0.050 & 0.068 & 0.054 & 0.054\\ %
\hline                 
\end{tabular}
\end{table}

\vskip 0.2in
\bibliographystyle{natbib}
\bibliography{references}

\newcommand{\noop}[1]{}
\begin{thebibliography}{}

\bibitem[Ai \emph{et~al.}(2021)Ai, Yu, Zhang, and Wang]{ai2019optimal}
Ai, M., Yu, J., Zhang, H., and Wang, H. (2021).
\newblock Optimal subsampling algorithms for big data regressions.
\newblock \emph{Statistica Sinica} \textbf{31}, 2, 749--772.

\bibitem[Bezanson \emph{et~al.}(2017)Bezanson, Edelman, Karpinski, and Shah]{julia}
Bezanson, J., Edelman, A., Karpinski, S., and Shah, V.~B. (2017).
\newblock Julia: A fresh approach to numerical computing.
\newblock \emph{SIAM review} \textbf{59}, 1, 65--98.

\bibitem[Chawla \emph{et~al.}(2002)Chawla, Bowyer, Hall, and Kegelmeyer]{chawla2002smote}
Chawla, N.~V., Bowyer, K.~W., Hall, L.~O., and Kegelmeyer, W.~P. (2002).
\newblock Smote: synthetic minority over-sampling technique.
\newblock \emph{Journal of artificial intelligence research} \textbf{16}, 321--357.

\bibitem[Douzas and Bacao(2017)]{douzas2017self}
Douzas, G. and Bacao, F. (2017).
\newblock Self-organizing map oversampling (somo) for imbalanced data set learning.
\newblock \emph{Expert systems with Applications} \textbf{82}, 40--52.

\bibitem[Drummond \emph{et~al.}(2003)Drummond, Holte, \emph{et~al.}]{drummond2003c4}
Drummond, C., Holte, R.~C., \emph{et~al.} (2003).
\newblock C4. 5, class imbalance, and cost sensitivity: why under-sampling beats over-sampling.
\newblock In \emph{Workshop on learning from imbalanced datasets II}, vol.~11.

\bibitem[Fan and Lv(2008)]{fan2008sure}
Fan, J. and Lv, J. (2008).
\newblock Sure independence screening for ultrahigh dimensional feature space.
\newblock \emph{Journal of the Royal Statistical Society: Series B (Statistical Methodology)} \textbf{70}, 5, 849--911.

\bibitem[Firth(1993)]{firth1993bias}
Firth, D. (1993).
\newblock {Bias reduction of maximum likelihood estimates}.
\newblock \emph{Biometrika} \textbf{80}, 1, 27--38.

\bibitem[Friedman \emph{et~al.}(2007)Friedman, Hastie, H{\"o}fling, and Tibshirani]{friedman2007pathwise}
Friedman, J., Hastie, T., H{\"o}fling, H., and Tibshirani, R. (2007).
\newblock {Pathwise coordinate optimization}.
\newblock \emph{The Annals of Applied Statistics} \textbf{1}, 2, 302 -- 332.

\bibitem[Friedman \emph{et~al.}(2010)Friedman, Hastie, and Tibshirani]{friedman2010Regularization}
Friedman, J.~H., Hastie, T., and Tibshirani, R. (2010).
\newblock Regularization paths for generalized linear models via coordinate descent.
\newblock \emph{Journal of Statistical Software} \textbf{33}, 1, 1--22.

\bibitem[Fu and Knight(2000)]{fu2000asymptotics}
Fu, W. and Knight, K. (2000).
\newblock Asymptotics for lasso-type estimators.
\newblock \emph{The Annals of statistics} \textbf{28}, 5, 1356--1378.

\bibitem[Geyer(1994)]{geyer1994asymptotics}
Geyer, C.~J. (1994).
\newblock On the asymptotics of constrained m-estimation.
\newblock \emph{The Annals of statistics}  1993--2010.

\bibitem[Han \emph{et~al.}(2005)Han, Wang, and Mao]{han2005borderline}
Han, H., Wang, W.-Y., and Mao, B.-H. (2005).
\newblock Borderline-smote: a new over-sampling method in imbalanced data sets learning.
\newblock In \emph{International conference on intelligent computing},  878--887. Springer.

\bibitem[Hjort and Pollard(2011)]{hjort2011asymptotics}
Hjort, N.~L. and Pollard, D. (2011).
\newblock Asymptotics for minimisers of convex processes.
\newblock \emph{arXiv preprint arXiv:1107.3806} .

\bibitem[Huang \emph{et~al.}(2008)Huang, Ma, and Zhang]{huang2008iterated}
Huang, J., Ma, S., and Zhang, C.-H. (2008).
\newblock The iterated lasso for high-dimensional logistic regression.
\newblock \emph{The University of Iowa, Department of Statistics and Actuarial Sciences} \textbf{7}.

\bibitem[JuliaStats \emph{et~al.}(2014-2025)]{juliastats2022lasso}
JuliaStats \emph{et~al.} (2014-2025).
\newblock {Lasso.jl}: Lasso/elastic net linear and generalized linear models.
\newblock Julia package for Lasso and Elastic Net models. Based on the paper by Friedman, Hastie, and Tibshirani (2010) in Journal of Statistical Software.

\bibitem[Keret and Gorfine(2023)]{keret2023analyzing}
Keret, N. and Gorfine, M. (2023).
\newblock Analyzing big ehr data—optimal cox regression subsampling procedure with rare events.
\newblock \emph{Journal of the American Statistical Association} \textbf{118}, 544, 2262--2275.

\bibitem[Liu \emph{et~al.}(2008)Liu, Wu, and Zhou]{liu2008exploratory}
Liu, X.-Y., Wu, J., and Zhou, Z.-H. (2008).
\newblock Exploratory undersampling for class-imbalance learning.
\newblock \emph{IEEE Transactions on Systems, Man, and Cybernetics, Part B (Cybernetics)} \textbf{39}, 2, 539--550.

\bibitem[Mathew \emph{et~al.}(2017)Mathew, Pang, Luo, and Leong]{mathew2017classification}
Mathew, J., Pang, C.~K., Luo, M., and Leong, W.~H. (2017).
\newblock Classification of imbalanced data by oversampling in kernel space of support vector machines.
\newblock \emph{IEEE transactions on neural networks and learning systems} \textbf{29}, 9, 4065--4076.

\bibitem[Pukelsheim(2006)]{pukelsheim2006optimal}
Pukelsheim, F. (2006).
\newblock \emph{Optimal design of experiments}.
\newblock SIAM.

\bibitem[Ramesh \emph{et~al.}(2023)Ramesh, Zhang, Sharpe, Penne, Haller, Lum, Lee, Lee, Pershing, Miller, Lorch, and Hyman]{ramesh2023thyroid}
Ramesh, S., Zhang, Q.~E., Sharpe, J., Penne, R., Haller, J., Lum, F., Lee, A.~Y., Lee, C.~S., Pershing, S., Miller, J.~W., Lorch, A., and Hyman, L. (2023).
\newblock Thyroid eye disease and its vision-threatening manifestations in the academy iris registry: 2014--2018.
\newblock \emph{American Journal of Ophthalmology} \textbf{253}, 74--85.
\newblock Epub 2023 May 17.

\bibitem[Shao(2003)]{Shao2003}
Shao, J. (2003).
\newblock \emph{Mathematical Statistics, 2nd}.
\newblock Springer-Verlag, New York.

\bibitem[Tripathi(1999)]{tripathi1999matrix}
Tripathi, G. (1999).
\newblock A matrix extension of the cauchy-schwarz inequality.
\newblock \emph{Economics Letters} \textbf{63}, 1, 1--3.

\bibitem[Wang(2020)]{wang2020logistic}
Wang, H. (2020).
\newblock Logistic regression for massive data with rare events.
\newblock In \emph{International Conference on Machine Learning},  9829--9836. PMLR.

\bibitem[Wang and Ma(2021)]{wang2021optimal}
Wang, H. and Ma, Y. (2021).
\newblock Optimal subsampling for quantile regression in big data.
\newblock \emph{Biometrika} \textbf{108}, 1, 99--112.

\bibitem[Wang \emph{et~al.}(2021)Wang, Zhang, and Wang]{wang2021nonuniform}
Wang, H., Zhang, A., and Wang, C. (2021).
\newblock Nonuniform negative sampling and log odds correction with rare events data.
\newblock \emph{Advances in Neural Information Processing Systems} \textbf{34}.

\bibitem[Wang \emph{et~al.}(2018)Wang, Zhu, and Ma]{WangZhuMa2017}
Wang, H., Zhu, R., and Ma, P. (2018).
\newblock Optimal subsampling for large sample logistic regression.
\newblock \emph{Journal of the American Statistical Association} \textbf{113}, 522, 829--844.

\bibitem[Wang \emph{et~al.}(2022)Wang, Zou, and Wang]{wang2022sampling}
Wang, J., Zou, J., and Wang, H. (2022).
\newblock Sampling with replacement vs poisson sampling: a comparative study in optimal subsampling.
\newblock \emph{IEEE Transactions on Information Theory} \textbf{68}, 10, 6605--6630.

\bibitem[Yao and Wang(2018)]{yao2018optimal}
Yao, Y. and Wang, H. (2018).
\newblock Optimal subsampling for softmax regression.
\newblock \emph{Statistical Papers}  585--599.

\bibitem[Yu \emph{et~al.}(2020)Yu, Wang, Ai, and Zhang]{yu2020optimal}
Yu, J., Wang, H., Ai, M., and Zhang, H. (2020).
\newblock Optimal distributed subsampling for maximum quasi-likelihood estimators with massive data.
\newblock \emph{Journal of the American Statistical Association} \textbf{0}, 0, 1--12.
\newblock DOI:10.1080/01621459.2020.1773832.

\bibitem[Yuan \emph{et~al.}(2012)Yuan, Ho, and Lin]{yuan2012improved}
Yuan, G.-X., Ho, C.-H., and Lin, C.-J. (2012).
\newblock An improved glmnet for l1-regularized logistic regression.
\newblock \emph{J. Mach. Learn. Res.} \textbf{13}, 1999--2030.

\bibitem[Zhang and Lu(2007)]{zhang2007adaptive}
Zhang, H.~H. and Lu, W. (2007).
\newblock {Adaptive Lasso for Cox's proportional hazards model}.
\newblock \emph{Biometrika} \textbf{94}, 3, 691--703.

\bibitem[Zhang \emph{et~al.}(2021)Zhang, Ning, and Ruppert]{Zhang2020optimal}
Zhang, T., Ning, Y., and Ruppert, D. (2021).
\newblock Optimal sampling for generalized linear models under measurement constraints.
\newblock \emph{Journal of Computational and Graphical Statistics} \textbf{30}, 1, 106--114.

\bibitem[Zou(2006)]{zou2006adaptive}
Zou, H. (2006).
\newblock The adaptive lasso and its oracle properties.
\newblock \emph{Journal of the American Statistical Association} \textbf{101}, 476, 1418--1429.

\end{thebibliography}

\end{document}